\documentclass[moor,nonblindrev]{informs4}
\usepackage{eqndefns-left}

\RequirePackage{tgtermes}
\RequirePackage{newtxtext}
\RequirePackage{newtxmath}
\RequirePackage{bm}
\RequirePackage{endnotes}

\SingleSpacedXI

\usepackage[hypertexnames=false]{hyperref}
\usepackage{algorithm}
\usepackage[noend]{algpseudocode}
\usepackage{tikz}
\usepackage{bbm}
\usepackage{outlines}
\usepackage{cleveref}
\usepackage{xcolor}
\usepackage{enumitem}
\usepackage{subcaption}
\usepackage{microtype}
\usepackage{thmtools}
\usepackage{thm-restate}
\usepackage{soul}
\usepackage{dashrule}
\usepackage{caption}
\usepackage[labelfont=sf]{subcaption}
\usepackage{mathtools,nccmath}
\allowdisplaybreaks

\newif\ifcomment
\commenttrue 

\newif\ifrevise
\revisetrue

\newcommand{\norm}[1]{\left\lVert#1\right\rVert}
\newcommand{\normplain}[1]{\lVert#1\rVert}
\newcommand{\normbig}[1]{\big\lVert#1\big\rVert}
\newcommand{\normBig}[1]{\Big\lVert#1\Big\rVert}

\newcommand{\E}[1]{\mathbb{E}\left[#1\right]}
\newcommand{\Eplain}[1]{\mathbb{E}[#1]}
\newcommand{\Ebig}[1]{\mathbb{E}\big[#1\big]}
\newcommand{\EBig}[1]{\mathbb{E}\Big[#1\Big]}
\newcommand{\Ebigg}[1]{\mathbb{E}\bigg[#1\bigg]}
\newcommand{\givenplain}{\,|\,}
\newcommand{\givenbig}{\,\big|\,}
\newcommand{\givenBig}{\,\Big|\,}
\newcommand{\givenbigg}{\,\bigg|\,}

\newcommand{\Varbig}[1]{\text{Var}{\big[#1\big]}}

\newcommand{\R}{\mathbb{R}}

\newcommand{\Prob}[1]{\mathbb{P}\left(#1\right)}

\newcommand{\ProbBig}[1]{\mathbb{P}\Big(#1\Big)}

\newcommand{\indibrac}[1]{\vmathbb{1}\!\left\{#1\right\}}
\newcommand{\abs}[1]{\left\lvert#1\right\rvert}
\newcommand{\absplain}[1]{\lvert#1\rvert}
\newcommand{\absbig}[1]{\big\lvert#1\big\rvert}
\newcommand{\absBig}[1]{\Big\lvert#1\Big\rvert}

\newcommand{\floorBig}[1]{\Big\lfloor#1\Big\rfloor}
\newcommand{\ve}[1]{\bm{#1}}

\newcommand{\vone}{\ve{1}}

\newcommand{\veS}{\ve{S}}

\newcommand{\veA}{\ve{A}}
\newcommand{\sspa}{\mathbb{S}}
\newcommand{\aspa}{\mathbb{A}}
\newcommand{\sumN}{\sum_{i\in[N]}}
\newcommand{\sumsa}{\sum_{s\in\sspa, a\in\aspa}}
\newcommand{\sums}{\sum_{s\in\sspa}}
\newcommand{\sumt}{\sum_{t=0}^{T-1}}

\newcommand{\rel}{\textup{rel}} 
\newcommand{\rmax}{r_{\max}}

\newcommand{\ravg}{R}
\newcommand{\rsysn}{\ravg(\pi, \veS_0)}
\newcommand{\rliminf}{\ravg^{-}(\pi, \veS_0)}
\newcommand{\rlimsup}{\ravg^{+}(\pi, \veS_0)}
\newcommand{\ropt}{\ravg^*(N, \veS_0)}
\newcommand{\rrel}{\ravg^\rel}

\newcommand{\pibar}{{\sysbar{\pi}}}
\newcommand{\pibs}{{\pibar^*}}
\newcommand{\syshat}[1]{\widehat{#1}}
\newcommand{\sysbar}[1]{\bar{#1}}

\newcommand{\constid}{C_{\textup{ID}}}
\newcommand{\constse}{C_{\textup{SE}}}
\newcommand{\constso}{C_{\textup{SO}}}
\newcommand{\constfs}{C_{\textup{FS}}}

\newcommand{\rhoBudget}{\beta}

\newcommand{\rhoContr}{\rho_2}

\newcommand{\gridset}{[0,1]_N}

\newcommand{\setbelow}[2]{[N#2]}
\newcommand{\setnm}[1]{[N#1]} 
\newcommand{\costvec}{c_\pibs}

\newcommand{\Ngood}{N^\pibs_t}
\newcommand{\md}{m_d}
\newcommand{\Md}[1]{N\md(x)}

\newcommand{\statdist}{\mu^*} 
\newcommand{\zmNoise}{\xi} 

\newcommand{\rhobar}{\bar{\rho}}
\newcommand{\wmat}{W} 
\newcommand{\dynset}{D}
\newcommand{\dynsetr}{D}
\newcommand{\ratiocw}{K_{c/h}} 
\newcommand{\constHnoise}{K_{\textup{drift}}}
\newcommand{\constVnorm}{K_{\textup{dist}}}
\newcommand{\liph}{L_h}

\newcommand{\lipw}{L_{\wmat}}
\newcommand{\lipwtilde}{L_{\wmat}}
\newcommand{\constAction}{K_{\textup{conf}}}
\newcommand{\constMono}{K_{\textup{mono}}}
\newcommand{\lipexpand}{L_\textup{cov}}
\newcommand{\constEpnoise}{K_{\textup{cov}}}
\newcommand{\rhoFinal}{\rho_1}

\newcommand{\hw}{h_\wmat}
\newcommand{\hid}{h_{\textnormal{ID}}}
\newcommand{\SE}{\textnormal{SE}}
\newcommand{\lamw}{\lambda_\wmat}
\newcommand{\Vhat}{V}
\newcommand{\slk}{\delta}

\newcommand{\sneu}{\tilde{s}}

\newcommand{\leftsub}{\text{L}}
\newcommand{\rightsub}{\text{R}}

\makeatletter
\newcommand{\algrule}[1][.2pt]{\hdashrule{\linewidth}{1pt}{1pt}}
\makeatother

\usepackage[sort&compress]{natbib}
 \bibpunct[, ]{[}{]}{,}{n}{}{,}%
 \def\bibfont{\small}%

\usepackage{multibib}
\newcites{app}{Appendix References}

\EquationsNumberedThrough  

\TheoremsNumberedThrough    
\ECRepeatTheorems 

\theoremstyle{plain}
\newtheorem{innercustomcond}{Condition}
\newenvironment{customcond}[1]
  {\renewcommand\theinnercustomcond{#1}\innercustomcond}
  {\endinnercustomcond}

\MANUSCRIPTNO{MOR-2024-0678.R1}

\begin{document}

\RUNAUTHOR{Hong, Xie, Chen, and Wang}

\RUNTITLE{Unichain and Aperiodicity are Sufficient for
Asymptotic Optimality of Average-Reward RBs}

\TITLE{Unichain and Aperiodicity are Sufficient for Asymptotic Optimality of Average-Reward Restless Bandits}

\ARTICLEAUTHORS{
\AUTHOR{Yige Hong}
\AFF{
Computer Science Department, 
Carnegie Mellon University, 
Pittsburgh, PA 15213,
\EMAIL{yigeh@andrew.cmu.edu}
}

\AUTHOR{Qiaomin Xie}
\AFF{
Department of Industrial and Systems Engineering, 
University of Wisconsin--Madison,
Madison, WI 53706,
\EMAIL{qiaomin.xie@wisc.edu}
}

\AUTHOR{Yudong Chen}
\AFF{
Department of Computer Sciences, 
University of Wisconsin--Madison, 
Madison, WI 53706,
\EMAIL{yudongchen@cs.wisc.edu}
}

\AUTHOR{Weina Wang}
\AFF{
Computer Science Department, 
Carnegie Mellon University,
Pittsburgh, PA 15213, 
\EMAIL{weinaw@cs.cmu.edu}
}
}

\ABSTRACT{We consider the infinite-horizon, average-reward restless bandit problem in discrete time. We propose a new class of policies that are designed to drive a progressively larger subset of arms toward the optimal distribution. We show that our policies are asymptotically optimal with an $O(1/\sqrt{N})$ optimality gap for an $N$-armed problem, assuming only a unichain and aperiodicity assumption. 
Our approach departs from most existing work that focuses on index or priority policies, which rely on the Global Attractor Property (GAP) to guarantee convergence to the optimum, or a recently developed simulation-based policy, which requires a Synchronization Assumption (SA).
}

\FUNDING{Yige Hong and Weina Wang are supported in part by U.S.\ National Science Foundation (NSF) grants ECCS-2145713, CCF-2403194, CCF-2428569, and ECCS-2432545.
Yudong Chen is supported in part by NSF grant CCF-2233152. Qiaomin Xie is supported in part by NSF grants CNS-1955997, ECCS-2339794, and ECCS-2432546.}

\KEYWORDS{Restless bandits; Long-run average reward; Asymptotic optimality}
\MSCCLASS{Primary: 90C40; Secondary: 90B15, 60J10, 93E20}
\SUBJECTCLASS{Dynamic programming/optimal control: Markov, finite state. Probability: Markov processes. Programming: linear.}

\maketitle

\section{Introduction}
\label{sec:intro}

A restless bandit (RB) problem \cite{Whi_88_rb} is a stochastic sequential decision-making problem that consists of multiple components.
Each component is associated with a Markov decision process (MDP) with two actions: activating/pulling the arm, or idling the arm. 
The MDPs of different arms share the same parameters. 
At each time step, the decision maker, who has knowledge of the MDP parameters, observes the states of all arms and decides which arms to activate.
This decision is subject to a \emph{budget constraint}, which requires that a fixed number of arms is activated at every time step.
The objective is to maximize the reward from all arms, where the reward from each arm is a function of its state and action. We illustrate the problem in \Cref{fig:rb-model}.
The RB problem has a rich history and wide-reaching applications. We refer the reader to the recent survey paper \cite{Nin_23} for a comprehensive overview of the literature.

\begin{figure}[t]
    \FIGURE
    {\includegraphics[width=0.8\linewidth]{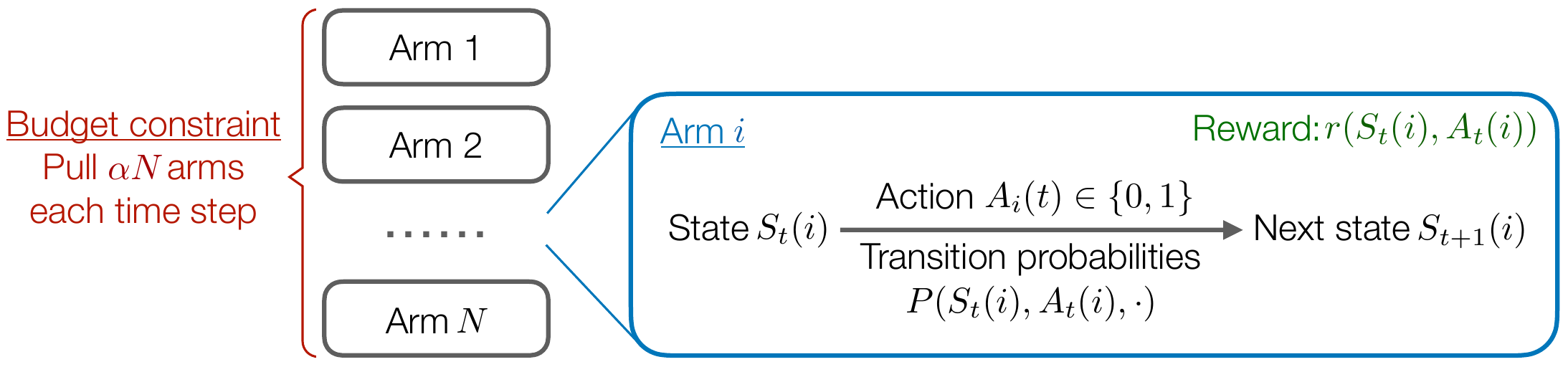}}
    {The restless bandit problem with $N$ arms.   \label{fig:rb-model}}
    {}
\end{figure}

Solving for an optimal policy for the RB problem is known to be PSPACE-hard \cite{PapTsi_99_pspace}.
However, it is possible to find \emph{asymptotically optimal} policies in a computationally efficient manner in the regime where the number of arms, $N$, grows large.
A policy is said to be asymptotically optimal if its \emph{optimality gap} is \emph{diminishing} as $N\to\infty$, where the optimality gap is the difference between the average reward per arm achieved by an optimal policy and that achieved by this policy.
This large $N$ regime, introduced in the seminal papers on the renowned Whittle index policy \cite{Whi_88_rb,WebWei_90}, has recently regained significant attention.
There has been a growing body of work that proposes new policies and provides refined analysis of their optimality gaps, not only in the infinite-horizon average-reward setting but also in the finite-horizon total-reward setting \cite{HuFra_17_rb_asymptotic,ZayJasWan_19_rb,BroSmi_19_rb,ZhaFra_21,BroZha_22,DaeGhoGri_23,GhoNagJaiTam_23_finite_discount,BroZha_23_ftva_and_reopt,GasGauYan_23_exponential,GasGauYan_24_reopt} and the infinite-horizon discounted-reward setting
\cite{ZhaFra_22_discounted_rb,GhoNagJaiTam_23_finite_discount,BroZha_23_ftva_and_reopt}. 
We discuss related work in more detail in \Cref{sec:additional-related-work}.

In this paper, we focus on the $N$-armed RB problem in the \emph{infinite-horizon, average-reward} setting. 
Existing policies for this setting, i.e., the celebrated Whittle index policy \cite{Whi_88_rb} and the more general LP-Priority policies \cite{Ver_16_verloop}, rely on a global attractor property (GAP) \cite{WebWei_90,Ver_16_verloop} or an even stronger property called the Uniform Global Attractor Property (UGAP) \cite{GasGauYan_23_whittles,GasGauYan_23_exponential} to achieve asymptotic optimality, in addition to the standard unichain and aperiodicity type of conditions. 
Roughly speaking, GAP requires the global convergence of the mean-field dynamics over time, where the mean-field dynamics characterizes the limit of the RB system as $N \to \infty$. 
GAP is a technical condition known to be difficult to verify for a given RB instance and policy, due to the non-linearity of the mean-field dynamics. 
Moreover, there are documented RB instances where the Whittle index and LP-Priority policies fail to satisfy GAP and are strictly asymptotically suboptimal \cite{GasGauYan_20_whittles,HonXieCheWan_23}. 
In \Cref{sec:experiments:non-ugap-is-common}, we further examine several natural classes of randomly generated RB instances; we observe that when the transition kernels and the reward function are generated from some sparse distributions, the percentage of non-GAP instances could be as high as $20\%$.

The recent work \cite{HonXieCheWan_23} takes a first step towards relaxing the long-standing GAP assumption. This paper proposes a policy named Follow-The-Virtual-Advice (FTVA) for the discrete-time RB problem and a variant for the continuous-time setting.
We focus on the discrete-time setting here.
FTVA achieves asymptotic optimality without GAP, but rather under an alternative condition named Synchronization Assumption (SA).
As argued in \cite{HonXieCheWan_23},  SA is more intuitive and easier-to-verify than GAP. 
However, the reliance on SA is still unsatisfactory; in particular, there exist RB instances where SA is not satisfied and FTVA performs suboptimally. In \Cref{sec:experiments:compare-non-sa} and Appendix~\ref{app:sa-counterexample}, we provide two counterexamples to SA, and discuss ways to construct more such examples.

The need for additional assumptions like GAP and SA reveals a fundamental gap in our understanding of the restless bandit problem. 
As such, the literature on RBs leaves open the following question: \emph{Is it possible to efficiently find a policy that achieves asymptotic optimality in infinite-horizon, average-reward RBs under only unichain and aperiodicity type of conditions, without imposing any additional conditions?}

\subsection{Our contributions}
\paragraph{Answer to the question.}
In this paper, we focus on the discrete-time RB problem and give a definitive, affirmative answer to this long-standing question. 
We propose a novel class of policies named \emph{focus-set policies}, and construct two concrete instances of focus-set policies that are asymptotically optimal with an $O(1/\sqrt{N})$ optimality gap under a weaker-than-standard aperiodic-unichain assumption (\Cref{assump:aperiodic-unichain}).

\paragraph{Policy design.}
Our proposed policies depart from the prevalent \emph{priority-based} design of existing policies. 
A priority-based policy specifies a fixed priority order over all the \emph{states} of a single arm.
At each time step, the policy pulls arms from states of higher priorities to those of lower priorities, until the budget constraint is met. 
In contrast, each of our proposed policies selects a subset of arms based on the \emph{empirical distribution} of their states and lets the selected arms take their \emph{ideal actions} as much as possible.
These ideal actions are computed using the solution of a single-armed, budget-relaxed problem.
The subset selection is constructed in a way such that most arms in the subset can take their ideal actions and the subset expands over time. 
Note that the FTVA policy in \cite{HonXieCheWan_23} also leverages the same single-armed problem to guide policy construction; however, it requires simulating additional virtual states, whereas our policies do not.

\paragraph{Proof techniques.}
We establish a meta-theorem that provides sufficient conditions for the asymptotic optimality of a focus-set policy. 
The proof of the meta-theorem highlights a class of bivariate Lyapunov functions we term \emph{subset Lyapunov functions}, along with a global Lyapunov function constructed dynamically from one of the subset Lyapunov functions.
Using these Lyapunov functions, we show that, under the stipulated sufficient conditions, the state-action distribution of arms in the selected subset converges to the optimal distribution, and the subset eventually expands to cover most arms. 
This meta-theorem allows us to prove the asymptotic optimality of the two proposed instances of the focus-set policies by verifying the stipulated sufficient conditions.

\subsection{Paper organization}
The remainder of the paper is organized as follows. 
In \Cref{sec:additional-related-work}, we give a more detailed review of the literature. 
In \Cref{sec:problem-statement}, we set up the problem of average-reward restless bandits and introduce the single-armed problem. 
In \Cref{sec:results}, we present our main results, where we propose the focus-set policies, present two concrete instances of focus-set policies, namely, the set-expansion policy and the ID policy, and establish their $O(1/\sqrt{N})$ optimality. 
In \Cref{sec:formalizing}, we present a meta-theorem that provides sufficient conditions for focus-set policies to achieve $O(1/\sqrt{N})$ optimality gaps. 
In \Cref{sec:proof-set-expansion-policy} and \Cref{sec:proof-id-policy}, we use the meta-theorem to prove the asymptotic optimality of the set-expansion policy and the ID policy. 
In \Cref{sec:experiments}, we conduct experiments comparing the performances of our policies with existing policies, and investigate the GAP and SA assumptions. 
We conclude the paper in \Cref{sec:discussion}.

\section{Related work}
\label{sec:additional-related-work}
\paragraph{Related work on conditions for asymptotic optimality.} 
As briefly discussed in the introduction, a line of work on the infinite-horizon average-reward RB problems has been progressively weakening the conditions for achieving asymptotic optimality when the number of arms $N\to\infty$. 
\citet{WebWei_90} establish asymptotic optimality of the Whittle index policy proposed by \citet{Whi_88_rb}, under three assumptions --- indexability, unichain, and the global attractor property (GAP). 
Later, \citet{Ver_16_verloop} proposes a more general class of priority policies derived from an LP relaxation, referred to as ``LP-Priority policies'' in  \cite{GasGauYan_23_exponential}, which removes the reliance on indexability and only requires the unichain and GAP assumptions to achieve asymptotic optimality. Notably, the Whittle index policy is a special case of LP-Priority policies. 
Both \cite{WebWei_90} and \cite{Ver_16_verloop} focus on the continuous-time RB problems. 
Recently, \citet{HonXieCheWan_23} propose a policy named Follow-the-Virtual-Advice (FTVA) for discrete-time RBs, and the continuous-time variant of FTVA named FTVA-CT. 
FTVA and FTVA-CT do not assume GAP for achieving asymptotic optimality. 
In particular, FTVA requires the unichain condition and a new assumption called the Synchronization Assumption (SA) to be asymptotically optimal, whereas FTVA-CT only requires the unichain condition. 
In addition to proving asymptotic optimality, \cite{HonXieCheWan_23} also gives non-asymptotic bounds for the optimality gaps of FTVA and FTVA-CT, which are of the order $O(1/\sqrt{N})$. 

We make two additional comments on the assumptions of the above-mentioned papers. First, the unichain conditions assumed in these papers are different. 
In particular, the unichain condition assumed in \citep{WebWei_90,Ver_16_verloop} requires the $N$-armed system to be unichain under every policy, which is stronger than the single-armed single-policy unichain condition considered in our paper (\Cref{assump:aperiodic-unichain}). 
See Appendix~\ref{app:assump-discuss-unichain} for a detailed discussion on different unichain conditions and their roles in these papers.  
Second, although none of the papers reviewed above explicitly assume any form of aperiodicity, this is either because they focus on the continuous-time setting \citep{WebWei_90, Ver_16_verloop}, where, loosely speaking, all policies are aperiodic, or because there is another assumption that plays a similar role as aperiodicity (i.e., Synchronization Assumption (SA) in \citep{HonXieCheWan_23}). 
Our paper considers the discrete-time setting without the SA assumption, where aperiodicity becomes crucial, as demonstrated by an example in Appendix~\ref{app:aperiodicity-counterexample}.

In addition to the prior work reviewed above, two recent papers \citep{Yan_24_multichain,GolAvr_24_wcmdp_multichain} appeared a few months after the arXiv version of our paper. These papers propose new policies that are asymptotically optimal under weaker conditions than ours. 
In particular, \citep{Yan_24_multichain} considers the discrete-time average-reward RB problem and \citep{GolAvr_24_wcmdp_multichain} considers the so-called weakly-coupled MDP problem, which is a multi-action, multi-constraint generalization of the RB problem. 
The assumptions in these two papers are implied by the single-armed MDP being weakly communicating and aperiodic, whereas ours are not. 
However, \citep{Yan_24_multichain,GolAvr_24_wcmdp_multichain} only provide asymptotic results but do not characterize the order of the optimality gaps.

\paragraph{Related work on better optimality gap orders.}
Apart from weakening the condition for asymptotic optimality, there is also prior work aiming for achieving better optimality gap than $O(1/\sqrt{N})$: 
\citet{GasGauYan_23_whittles,GasGauYan_23_exponential} prove $O(\exp(-CN))$ optimality gap bounds for the Whittle index policy and LP-Priority policies for some $C>0$, in both discrete-time or continuous-time settings. 
The exponential optimality gap mainly relies on an additional assumption named non-singularity or non-degeneracy, inspired by a recent paper \cite{ZhaFra_21} on finite-horizon RB problems to be discussed later in this section. 
Apart from non-singularity or non-degeneracy, other assumptions in \cite{GasGauYan_23_whittles,GasGauYan_23_exponential} are also slightly stronger than those in \cite{WebWei_90,Ver_16_verloop} in several aspects: 
\cite{GasGauYan_23_whittles,GasGauYan_23_exponential} require the RB problem to be irreducible under every policy, instead of just being unichain; for discrete-time RBs, they also require the RB problem to be aperiodic under every policy; in addition, they need a stronger version of global attractor property than GAP, referred to as Uniform Global Attractor Property (UGAP), which is discussed in detail in Section 6.2.1 of \cite{GasGauYan_23_exponential}.

\paragraph{Related work on finite-horizon or discounted-reward settings.}
Apart from the infinite-horizon average-reward setting, there is also a large body of work on other reward settings.

In the finite-horizon total-reward setting, there have been papers \cite{HuFra_17_rb_asymptotic,ZayJasWan_19_rb,BroSmi_19_rb,ZhaFra_21,BroZha_22,DaeGhoGri_23,GhoNagJaiTam_23_finite_discount,BroZha_23_ftva_and_reopt,GasGauYan_23_exponential,GasGauYan_24_reopt} that readily achieve asymptotic optimality in the $N \to \infty$ limit without assumptions, with the main focus being on improving the orders of the optimality gaps. 
Specifically, the optimality gap has been improved from $o(1)$ in \cite{HuFra_17_rb_asymptotic} to $O(\log N/\sqrt{N})$ in \cite{ZayJasWan_19_rb,DaeGhoGri_23,BroZha_22} and $O(1/\sqrt{N})$ in \cite{BroSmi_19_rb,GhoNagJaiTam_23_finite_discount,BroZha_23_ftva_and_reopt}, without assumptions. 
Later, \citet{ZhaFra_21} propose a policy that achieves an $O(1/N)$ optimality gap under a mild assumption called non-degeneracy; \citet{GasGauYan_23_exponential,GasGauYan_24_reopt} further improve the optimality gap to $O(\exp(-CN))$ under the same non-degeneracy assumption. 
Although the above-mentioned papers on the finite-horizon total-reward setting are able to prove strong bounds of optimality gap in terms of the scaling of $N$ under minimal conditions, most of them either do not consider the scaling of the time horizon $T$ \cite{HuFra_17_rb_asymptotic,BroSmi_19_rb}, or have a super-linear dependency on the $T$ (quadratic in \cite{ZayJasWan_19_rb,DaeGhoGri_23,GhoNagJaiTam_23_finite_discount}, and $O(T\log T)$ in \cite{BroZha_22}). 
Consequently, most of the bounds in the finite horizon setting are incomparable to the bounds in the infinite-horizon average-reward setting --- the latter is analogous to having a linear dependency on $T$. 
There are, however, two exceptions: \citet{BroZha_23_ftva_and_reopt} and \citet{GasGauYan_24_reopt} obtain bounds with linear dependencies on $T$, under additional assumptions similar to the Synchronization Assumption in \cite{HonXieCheWan_23}. 
Nevertheless, there is no direct way for applying the policies in \cite{BroZha_23_ftva_and_reopt,GasGauYan_24_reopt} to the infinite-horizon setting, since they both involve solving subproblems whose complexities depend on the time horizon. 

The asymptotic optimality in the infinite-horizon discounted-reward setting has also been considered in the prior work \cite{BroSmi_19_rb,ZhaFra_22_discounted_rb,GhoNagJaiTam_23_finite_discount,BroZha_23_ftva_and_reopt}. 
Similar to the finite-horizon setting, asymptotic optimality can be achieved without assumptions. In particular, an $O(N^{\log_2(\sqrt{\gamma})})$ optimality gap is obtained in \cite{BroSmi_19_rb} for a discount factor $\gamma \in (1/2,1)$, and $O(1/\sqrt{N})$ optimality gaps are achieved in \cite{ZhaFra_22_discounted_rb,GhoNagJaiTam_23_finite_discount,BroZha_23_ftva_and_reopt}. 
These results are incomparable to the asymptotic optimality results in the infinite-horizon average-reward setting due to their dependencies on $\gamma$. 
The $O(1/\sqrt{N})$ bounds in \cite{ZhaFra_22_discounted_rb,GhoNagJaiTam_23_finite_discount,BroZha_23_ftva_and_reopt} scale at least quadratically with the effective horizon $1/(1-\gamma)$ rather than linearly. 
Meanwhile, the $O(N^{\log_2(\sqrt{\gamma})})$ bound in \cite{BroSmi_19_rb} has a coefficient that scales linearly with $1/(1-\gamma)$, but since $N^{\log_2(\sqrt{\gamma})}$ becomes constant in $N$ after taking the limit of $\gamma\to 1$, this bound is not comparable to the asymptotic optimality results in the average-reward setting, where the bounds must be diminishing in $N$.

\paragraph{Generalizations of restless bandits.}
Some generalizations of the restless bandit problem have also been extensively studied in the literature. Those generalizations include having multiple actions, multiple constraints, state-dependent costs, heterogeneous arms, time-inhomogeneous rewards and transitions, etc. 
A lot of the papers mentioned above contain results for one or more such generalizations. While we believe it is possible to also generalize our results to some of these settings, we do not pursue this direction in this paper. 
We refer readers to recent papers such as \cite{BroZha_23_ftva_and_reopt} and \cite{GasGauYan_24_reopt} for a more detailed review of more general settings. 

\paragraph{Follow-up work of this paper.}
Since the initial version of this paper was released on arXiv, two follow-up papers \citep{HonXieCheWan_24_exp,ZhaHonWan_25_het} have directly built upon the techniques in this paper.
In particular, the work \citep{HonXieCheWan_24_exp}, with the same authors as the current paper also considers the long-run average-reward restless bandit problem. 
It shows that a generalization of the set-expansion policy, named ``two-set policy'', achieves $O(\exp(-CN))$ optimality gap. This result holds under three assumptions: aperiodic unichain condition, non-degeneracy, and a local stability condition that relaxes the UGAP assumption required by \citep{GasGauYan_23_whittles,GasGauYan_23_exponential}. 
The second paper \citep{ZhaHonWan_25_het} considers fully heterogeneous weakly-coupled MDPs, where each arm could have multiple actions and distinct transition kernels and reward functions, and there are multiple budget constraints corresponding to different types of cost. 
The authors show that a generalization of the ID policy achieves an $O(1/\sqrt{N})$ optimality gap, under a bounded mixing-time assumption that generalizes the aperiodic unichain assumption in this paper.

\section{Problem setup}\label{sec:problem-statement}

In this section, we set up the average-reward restless bandits problem and its single-armed relaxation, and introduce the assumptions and notations used throughout the paper.

\subsection{Restless bandit problem}

We consider the discrete-time, infinite-horizon restless bandit problem with the average-reward criterion.
The RB problem consists of $N$ homogeneous arms and is henceforth referred to as the $N$-armed problem. Each arm is associated with an MDP called the single-armed MDP, which is defined by the tuple $(\sspa, \aspa, P, r)$. 
Here $\sspa$ is the state space, which is a finite set; 
$\aspa = \{0, 1\}$ is the action space, where the action $1$ is interpreted as activating or pulling the arm; 
$P:\sspa\times\aspa\times\sspa \to [0,1]$ is the transition kernel, where $P(s,a,s')$ is the probability of transitioning to state $s'$ in the next time step conditioned on taking action $a$ at state $s$ in the current step; 
$r: \sspa \times \aspa \to \mathbb{R}$ is the reward function, where $r(s,a)$ is the expected reward for taking action $a$ in state $s$.
Let $\rmax = \max_{s\in\sspa,a\in\aspa} \abs{r(s,a)}$.
The RB problem has a \emph{budget constraint}, which requires that exactly $\alpha N$ arms must be pulled at every time step for some given constant $\alpha\in(0,1)$. Here $\alpha N$ is assumed to be an integer for simplicity. 
We note that although the equality budget constraint is standard in the literature and is the primary focus of this paper, our algorithms and results also extend to a natural alternative formulation with an inequality constraint, which allows \emph{at most} $\alpha N$ arms to be pulled at every time step. 
We will describe the necessary modifications at appropriate points in the paper. 
Finally, note that we focus on the setting in which all the model parameters, $\sspa, \aspa, P, r, \alpha$, are known. 

We index the arms in an $N$-armed bandit by $[N]$, where $[n] \triangleq \{1,2,\dots, n\}$ for integer $n\geq 1$ and $[0]\triangleq\emptyset$. 
We refer to the index~$i$ of Arm~$i$ as its \emph{ID}, to avoid confusion with the Whittle index or other index notion.

A policy $\pi$ for the $N$-armed problem chooses in each time step the action for each of the $N$ arms.
We allow the policy to be randomized and choose actions based on the whole history.

Under a policy $\pi$, we use the \emph{state vector} $\ve{S}_t^\pi \triangleq (S_t^\pi(i))_{i\in[N]} \in \sspa^N$ to represent the states of all arms, where $S_t^\pi(i) \in \sspa$ denotes the state of the $i$-th arm at time $t$. 
Similarly, the \emph{action vector} is defined as $\ve{A}_t^\pi \triangleq (A_t^\pi(i))_{i\in[N]} \in \aspa^N$, where $A_t^\pi(i) \in \aspa$ denotes the action applied to the $i$-th arm at time $t$.

Let the \emph{limsup average reward} be $\rlimsup \triangleq \limsup_{T\to\infty} (1/(NT)) \sum_{t=0}^{T-1} \sumN \E{r(S_t^\pi(i), A_t^\pi(i))}$ and let the \emph{liminf average reward} be $\rliminf \triangleq \liminf_{T\to\infty} (1/(NT)) \sum_{t=0}^{T-1} \sumN \E{r(S_t^\pi(i), A_t^\pi(i))}$. 
When the limsup and liminf average rewards coincide, the infinite-horizon average reward (also known as the long-run average reward) exists and is given by
\[
    \rsysn \triangleq \lim_{T\to\infty} \frac{1}{T} \sum_{t=0}^{T-1} \frac{1}{N} \sumN \E{r(S_t^\pi(i), A_t^\pi(i))}.
\]

Our goal is to solve the following optimization problem.
\begin{subequations}
\begin{align}
    \label{eq:N-arm-formulation} \tag{RB}
    \underset{\text{policy } \pi}{\text{maximize}} & \quad \rliminf  \\
    \text{subject to}  
    &\quad  \sumN A_t^\pi(i) = \alpha N,\quad \forall t=0,1,2,\dots \label{eq:hard-budget-constraint}
\end{align} 
\end{subequations}
Let $\ropt$ be the optimal value of the problem, referred to as the \emph{optimal reward}. 
Note that $\ropt = \sup_{\pi'} R^-(\pi', \veS_0) = \sup_{\pi'} R^+(\pi', \veS_0)$ because \eqref{eq:N-arm-formulation} is an MDP with finite state and action spaces \citep[][Theorem~9.1.6]{Put_05}. 
For any policy $\pi$, we define its optimality gap as $\ropt - \rliminf$; we say the policy is \emph{asymptotically optimal} if its optimality gap vanishes as $N\to\infty$, i.e., $\ropt - \rliminf = o(1)$. 
This notion of asymptotic optimality is consistent with the literature; see, e.g., \cite[][Definition 4.11]{Ver_16_verloop}. 

In later parts of the paper, we will focus on policies under which the long-run average reward $\rsysn$ exists. These policies include any stationary Markovian policies, under which $\veS_t$ is a finite-state Markov chain \citep[][Proposition 8.1.1]{Put_05}. More generally, with a similar argument, it is easy to show that $\rsysn$ is well-defined if $\pi$ makes decisions based on augmented system states with a finite state space. 
Importantly, restricting to such policies is sufficient, because there always exists a stationary Markovian policy whose long-run average reward achieves the optimal reward, by standard results for the MDPs with finite state and action spaces \citep[][Theorem~9.1.8]{Put_05}. 
For simplicity, we will refer to $\rsysn$ as the objective function of \eqref{eq:N-arm-formulation} and write the optimality gap as $\ropt - \rsysn$.

\subsection{Scaled state-count vector}

We introduce an alternative way, used extensively in the paper, for representing the information contained in the state vector $\ve{S}_t^\pi$. 
For each subset $D\subseteq[N]$, we define the \emph{scaled state-count vector on $D$} as $X_t^\pi(D) = (X_t^\pi(D, s))_{s\in\sspa}$, where 
\[
    X_t^\pi(D, s) = \frac{1}{N} \sum_{i\in D} \indibrac{S_t^\pi(i) = s}.
\]
Note that each entry of the vector $X_t^\pi(D)$ is the number of arms in $D$ in a certain state, scaled by $1/N$. 
When $D = [N]$ is the set of all arms, we simply call $X_t^\pi([N])$ the \emph{scaled state-count vector}.

Sometimes we view $X_t^\pi(D)$ as a vector-valued function of $D \subseteq [N]$. We refer to this function $X_t$ as the \emph{system state} at time $t$. The system state $X_t^\pi$ contains the same information as the state vector $\veS_t^\pi$ does; in particular, from $X_t^\pi$ one can deduce the state of each arm.

\subsection{LP relaxation}
In this section, we discuss a linear programming (LP) relaxation of the $N$-armed problem \eqref{eq:N-arm-formulation} which is crucial for the design and analysis of RB policies. 
This LP is defined as follows.
\begin{subequations}
\begin{align}
    \label{eq:lp-single} \tag{LP} \underset{\{y(s, a)\}_{s\in\sspa,a\in\aspa}}{\text{maximize}} \mspace{12mu}&\sum_{s\in\sspa,a\in\aspa} r(s, a) y(s, a) \\
    \text{subject to}\mspace{21mu}
    &\mspace{15mu}\sum_{s\in\sspa} \quad y(s, 1) = \alpha, \label{eq:expect-budget-constraint}\\
    & \sum_{s'\in\sspa, a\in\aspa} y(s', a) P(s', a, s) = \sum_{a\in\aspa} y(s,a), \quad \forall s\in\sspa, \label{eq:flow-balance-equation}\\
    &\mspace{3mu}\sum_{s'\in\sspa, a'\in\aspa} y(s',a') = 1,  
    \quad
     y(s,a) \geq 0, \;\; \forall s\in\sspa, a\in\aspa.  \label{eq:non-negative-constraint}
\end{align}
\end{subequations}
To see why \eqref{eq:lp-single} is a relaxation of \eqref{eq:N-arm-formulation}, for any stationary Markovian policy $\pi$, consider 
\[
    y^\pi(s,a) = \lim_{T\to\infty} \frac{1}{T} \sum_{t=0}^{T-1} \EBig{\frac{1}{N} \sumN \indibrac{S_t^\pi(i) = s, A_t^\pi(i) = a}} \quad \forall s\in\sspa, a\in\aspa.
\]
It is not hard to see that $\rsysn = \sumsa r(s,a)y^\pi(s,a)$, and that $(y^\pi(s,a))_{s\in\sspa, a\in\aspa}$ satisfies the constraints \eqref{eq:expect-budget-constraint}--\eqref{eq:non-negative-constraint}. 
Therefore, letting $\rrel$ be the optimal value of \eqref{eq:lp-single}, it can be shown that $\rrel \geq \ropt$ (See Appendix~\ref{app:upper-bound} for the detailed proof). 
This relation allows us to bound the optimality gap of any policy $\pi$ using the inequality $\ropt - \rliminf \leq \rrel - \rliminf$. Bounding the optimality gap via this LP relaxation is the standard approach in the literature \cite{WebWei_90,Ver_16_verloop,GasGauYan_23_whittles,GasGauYan_23_exponential,HonXieCheWan_23}.

\subsection{Optimal single-armed policy}
To understand how to achieve the average-reward upper bound $\rrel$ given by the LP relaxation \eqref{eq:lp-single}, it is helpful to view \eqref{eq:lp-single} as solving for a certain stationary state-action probability, $y(s,a)$, in the single-armed MDP, $(\sspa, \aspa, P, r)$.  
Specifically, the objective of \eqref{eq:lp-single} equals the expected reward under the stationary probability. 
The constraint in \eqref{eq:expect-budget-constraint} can be interpreted as a budget constraint, which requires the arm to be activated with $\alpha$ probability in the steady state. 
The constraint \eqref{eq:flow-balance-equation} is the stationary equation. 
The constraint \eqref{eq:non-negative-constraint} ensures that $(y(s,a))_{s\in\sspa, a\in\aspa}$ is a valid probability distribution.

From each stationary state-action probability $(y(s,a))_{s\in\sspa, a\in\aspa}$, one can construct a policy for the single-armed MDP, which we call a \emph{single-armed policy}, that achieves this state-action probability. 
In particular, let $\{y^*(s, a)\}_{s\in\sspa,a\in\aspa}$ be an optimal solution to \eqref{eq:lp-single}. We consider the following single-armed policy $\pibs$: 
\begin{equation}\label{eq:single-arm-opt-def}
    \pibs(a | s) =
    \begin{cases}
        y^*(s, a) / (y^*(s, 0) + y^*(s,1)), & \text{if } y^*(s, 0) + y^*(s,1) > 0,  \\
        1/2, &  \text{if } y^*(s, 0) + y^*(s,1) = 0.
    \end{cases}
    \quad \text{for $s\in\sspa$, $a\in\aspa$.}
\end{equation}
We call $\pibs$ the \emph{optimal single-armed policy}. Let $P_\pibs$ be the transition matrix induced by $\pibs$ in the single-armed MDP. 
We make the following assumption throughout the paper:

\begin{assumption}[Unichain and aperiodicity]
\label{assump:aperiodic-unichain}
    There exists an optimal solution $\{y^*(s, a)\}_{s\in\sspa,a\in\aspa}$ to \eqref{eq:lp-single}, such that the optimal single-armed policy $\pibs$ defined in \eqref{eq:single-arm-opt-def} induces an aperiodic unichain (i.e., a Markov chain with a single recurrent class and a possibly empty set of transient states) with state space $\sspa$ and transition matrix $P_\pibs$.
\end{assumption}

With \Cref{assump:aperiodic-unichain}, the Markov chain induced by $\pibs$ converges to a unique stationary distribution, which we denote as $\statdist = (\statdist(s))_{s\in\sspa}$. 
From the definition of $\pibs$ in \eqref{eq:single-arm-opt-def}, it is easy to verify that $\statdist(s) = y^*(s,0) + y^*(s,1)$; thus the steady-state state-action probability under $\pibs$ is $(y^*(s,a))_{s\in\sspa, a\in\aspa}$. 
Consequently, the long-run average reward of $\pibs$ equals the optimal value of \eqref{eq:lp-single}, $\rrel$; the long-run average budget usage of $\pibs$ equals $\alpha$. 

In Appendix~\ref{app:assump-discuss}, we discuss the generality of \Cref{assump:aperiodic-unichain}. In particular, we compare \Cref{assump:aperiodic-unichain} with the assumptions in the literature; we also give an example to show that $\rrel-\ropt$ can be non-diminishing as $N\to\infty$ when the single-armed MDP is periodic.

\subsection{Additional notation}

For a subset $D\subseteq[N]$, we let $m(D) = |D| / N$ denote the fraction of arms contained in $D$. 
We introduce a convenient shorthand $[0,1]_N = \{0,1/N, 2/N,\dots, 1\}$. Then $m(D) \in [0,1]_N$ for any $D\subseteq [N]$. 
Let $\Delta(\sspa)$ denote the set of probability distributions on the state space $\sspa$.
We treat each distribution $v\in\Delta(\sspa)$ as a row vector.
Recall that $\pi$ denotes a policy for the $N$-armed problem. In later sections, when the context is clear, we drop the superscript $\pi$ from the vectors $\veS_t^\pi$, $\veA_t^\pi$, and $X_t^\pi$.
We use $a^+\triangleq \max\{a,0\}$ to denote the positive part of $a\in\mathbb{R}$.

\section{Main results: policies and optimality guarantees}\label{sec:results}
In this section, we propose policies for the average-reward RB problems and bound their optimality gaps. 
In \Cref{sec:focus-set-policy}, we present an algorithmic idea based on the convergence of state distribution to $\statdist$ under the optimal single-armed policy $\pibs$, and propose a novel class of policies named \emph{focus-set policies}. 
Then in \Cref{sec:results_set-expansion-policy} and \Cref{sec:results_ID-policy}, we present two instances of focus-set policies, named the \emph{set-expansion policy} and the \emph{ID policy}, and state their optimality gap bounds.  
In \Cref{sec:connection-previous-policies}, we discuss the relationships between our policies and the policies in the literature to explain why they rely on different assumptions. 
Finally, in \Cref{sec:inequality}, we discuss how our results can be adapted for the alternative setting with the inequality budget constraint.

\subsection{Algorithmic idea and focus-set policies}
\label{sec:focus-set-policy}

Consider the single-armed system and the optimal single-armed policy $\pibs$. Because $P_\pibs$ is an aperiodic unichain by \Cref{assump:aperiodic-unichain}, it follows that starting from any initial distribution in $\Delta(\sspa)$, the state distribution of the Markov chain $P_\pibs$ converges to the steady-state distribution $\statdist$.

We observe the following simple fact based on the single-armed convergence under $\pibs$:
an RB system would achieve the reward upper bound if all arms could follow the optimal single-armed policy $\pibs$.
However, exactly achieving this goal is not possible due to the budget constraint. A natural way is to approach it gradually: let a subset of arms persistently follow $\pibs$ and wait for them to approach $\mu^*$; at this point, more arms can be included into the subset, as we explain later.

We capture this idea of letting a subset of arms follow $\pibs$ and then gradually expanding the subset in a class of policies named \emph{focus-set policies}, the template of which is given in \Cref{alg:focus-set}.
In particular,
a focus-set policy first samples an \emph{ideal action} $\syshat{A}_t(i)$ using $\pibs$ for each arm $i\in [N]$ based on its state $S_t(i)$ at time $t$ (Line \ref{alg:focus-set-action-sampling}). 
The policy then selects a subset of arms $D_t$, referred to as the \emph{focus-set}, and gives them precedence to set the actual action equal to the ideal action, i.e., $A_t(i) = \syshat{A}_t(i)$ (Line \ref{alg:focus-set-action-rect-1}).

\begin{algorithm}[t]
\caption{Focus-set policy}
\begin{flushleft}
\hspace*{\algorithmicindent} \textbf{Input}: number of arms $N$, budget $\alpha N$, the optimal single-armed policy $\pibs$, \\
\hspace*{\algorithmicindent} \hspace{0.45in} initial system state $X_0$, initial state vector $\veS_0$, initial focus set $D_{-1}$
\end{flushleft}
\label{alg:focus-set}
\begin{algorithmic}[1]
    \For{$t=0,1,2,\dots$}
        \State Choose a \emph{focus set} $\dynsetr_t\subseteq [N]$ based on $X_t$ and $\dynset_{t-1}$ \Comment{\ul{\emph{Set update}}}
        \State Independently sample $\syshat{A}_t(i)\sim \pibs(\cdot \givenplain S_t(i))$ for $i\in[N]$  \Comment{\ul{\emph{Action sampling}}}  \label{alg:focus-set-action-sampling}
        \State Pick $A_t(i)$ for $i\in D_t$ based on $X_t$ and $D_t$ to achieve
        \Comment{\ul{\emph{Action rectification}}}
        \begin{equation*}
            \underset{\{A_t(i)\colon i\in D_t\}}{\text{maximize}} \mspace{18mu} \left|\left\{i\in D_t\colon A_t(i) = \syshat{A}_t(i)\right\}\right|\mspace{81mu}
        \end{equation*}
        \begin{equation}
            \mspace{3mu}\text{subject to}\mspace{24mu}\alpha N - (N - |D_t|) \leq \sum_{i\in \dynsetr_t} A_t(i) \leq \alpha N \label{eq:policy-1-actions-rounding-constraints} 
        \end{equation} 
        \label{alg:focus-set-action-rect-1}
        \State Pick $A_t(i)$ for $i\in \dynsetr_t^c$ based on $X_t$ and $D_t$  such that 
        \begin{equation}\label{eq:budget-constraint-restate}
            \sum_{i\in[N]} A_t(i) = \alpha N
        \end{equation}
        \label{alg:focus-set-action-rect-2}
        \State  Apply $A_t(i)$ and observe $S_{t+1}(i)$ for each arm $i\in[N]$   \label{alg:focus-set-apply-action}
    \EndFor
\end{algorithmic}
\end{algorithm}

The performance benefit of a focus-set policy will not be realized until one specifies a proper rule to update the focus set $D_t$.
In subsequent subsections, we propose two instances of focus-set policies with updating rules that lead to asymptotically optimal performance.
In \Cref{sec:formalizing}, we provide more general sufficient conditions for a focus-set policy to achieve asymptotic optimality.

\subsection{Set-expansion policy}
\label{sec:results_set-expansion-policy}

In this section, we introduce an instance of the focus-set policies called the \emph{set-expansion policy}. 
The set-expansion policy updates $D_t$ based on a quantity referred to as the \emph{slack},
which is defined for a subset $D \subseteq [N]$ based on the system state $x$ as follows: 
\begin{equation}\label{eq:slack-def}
    \slk(x, \dynset) = \rhoBudget (1-m(\dynset)) -  \frac{1}{2}\norm{x(D) - m(D) \statdist}_1,
\end{equation}
where $\beta = \min(\alpha, 1-\alpha)$ and recall that $m(D) = |D| / N$. 
The policy chooses $D_t$ with the maximal cardinality such that $\slk(X_t, \dynsetr_t) \geq 0$, under the constraint that it either expands upon $D_{t-1}$ (i.e., $D_t\supseteq D_{t-1}$) if $\slk(X_t, D_{t-1}) > 0$ or shrinks (i.e., $D_t\subseteq D_{t-1}$) otherwise. See Algorithm~\ref{alg:set-expansion} for the full definition of the set-expansion policy.

\begin{algorithm}[t] 
\caption{Set-expansion policy}
\label{alg:set-expansion}
\begin{flushleft}
\hspace*{\algorithmicindent} \textbf{Input}: number of arms $N$, budget $\alpha N$, the optimal single-armed policy $\pibs$, \\
\hspace*{\algorithmicindent} \hspace{0.45in} initial system state $X_0$, initial state vector $\veS_0$, initial focus set $D_{-1}=\emptyset$ 
\end{flushleft}
\begin{algorithmic}[1]
    \For{$t=0,1,2,\dots$} 
        \If{$\slk(X_t, \dynsetr_{t-1}) > 0$}
        \Comment{\ul{\emph{Set update}}}
        \label{alg:set-expansion:set-update-step-begins}
            \State Let $D_t$ be any set with the largest $m(D_t)$ such that $D_t \supseteq D_{t-1}$ and  $\slk(X_t, D_t) \geq 0$
        \Else
            \State Let $D_t$ be any set with the largest $m(D_t)$ such that $D_t \subseteq D_{t-1}$ and $\slk(X_t, D_t) \geq 0$
        \EndIf
        \label{alg:set-expansion:set-update-step-ends}
        
        \noindent\algrule
        \Comment{\emph{Lines below implement Lines~\ref{alg:focus-set-action-sampling}--\ref{alg:focus-set-apply-action} of \Cref{alg:focus-set}, with random tie-breaking for Lines~\ref{alg:focus-set-action-rect-1}--\ref{alg:focus-set-action-rect-2}}}
        \State Independently sample $\syshat{A}_t(i) \sim \pibs(\cdot|S_t(i))$ for $i\in[N]$ 
        \Comment{\ul{\emph{Action sampling}}}\label{line:action-sampling}
        \label{alg:set-expansion-sampling}
        \If{$\sum_{i\in \dynsetr_t} \syshat{A}_t(i) \geq \alpha N$} \Comment{\ul{\emph{Action rectification}}}
        \label{alg:set-expansion-action-rect-starts}
            \State Select $\alpha N$ arms in $\dynsetr_t$ with $\syshat{A}_t(i)=1$ uniformly at random, and set $A_t(i) = 1$
            \State For the rest of $i\in [N]$, set $A_t(i)=0$
        \ElsIf{$\sum_{i\in \dynsetr_t} \syshat{A}_t(i) \leq \alpha N - (N - |D_t|)$} \label{alg:set-expansion-action-rect-second-case-starts}
            \State Select $(1-\alpha)N$ arms in $\dynsetr_t$ with $\syshat{A}_t(i)=0$ uniformly at random, and set $A_t(i)=0$
            \State For the rest of $i\in [N]$, set $A_t(i) = 1$
        \Else
            \label{alg:set-expansion-action-rect-third-case-starts}
            \State Set $A_t(i) = \syshat{A}_t(i)$ for $i\in D_t$
            \State Set $A_t(i) = \syshat{A}_t(i)$ for as many $i\notin D_t$ as possible; break ties uniformly at random
            \label{alg:set-expansion-action-rect-ends}
        \EndIf
        \State Apply $A_t(i)$ and observe $S_{t+1}(i)$ for each arm $i\in[N]$ \label{alg:set-expansion-apply-action}
    \EndFor
\end{algorithmic}
\end{algorithm}

\begin{figure}[t]
    \FIGURE{\includegraphics[height=2.3cm]{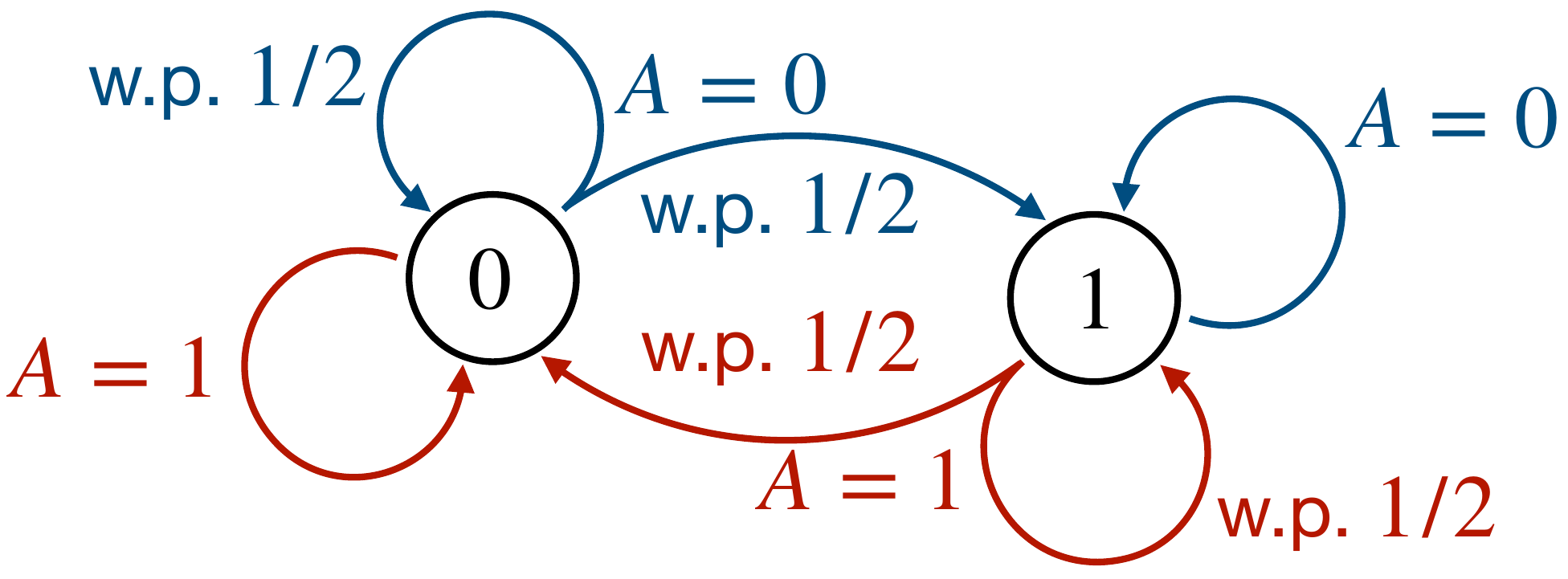}}
    {Single-armed MDP of a simple RB example for illustrating the policies. \label{fig:union-state-space:two-state-mdp}}
    {The single-armed MDP of the RB problem has two states, $\{0, 1\}$, whose transition structure is illustrated above. One unit of reward is generated if and only if the state changes. The budget parameter $\alpha$ is $0.5$. One can easily see that the optimal single-armed policy $\pibs$ activates the arm if and only if the arm is in state $1$. The optimal stationary distribution is $\statdist=(0.5,0.5)$.}
\end{figure}

Here we briefly explain the design of the set-expansion policy and why it works. 
The slack $\slk(x, D)$ is carefully constructed such that $\slk(X_t, D_t)\geq 0$ ensures that most arms in $D_t$ can follow $\pibs$. Moreover, $P_\pibs$ is non-expansive under the $L_1$ norm, so $\slk(X_{t}, D_t) \geq 0$ often implies that $\slk(X_{t+1}, D_t) \geq 0$, preventing $D_{t+1}$ from shrinking significantly. 
Consequently, each arm in the focus set is likely to remain in the focus set for a long time, where they persistently follow $\pibs$ and converge to the optimal stationary distribution $\statdist$. 
As soon as the scaled state-count vector on $D_t$, $X_t(D_t)$, is sufficiently close to $m(D_t)\statdist$, the focus set expands. 
Therefore, in the long run, the arms in the focus set converges to $\statdist$, allowing the focus set to cover most of the arms.

\begin{figure}
    \FIGURE{
    \subcaptionbox{Intuitions of the set-expansion policy. \label{fig:union-state-space:subset-dynamics}}{\includegraphics[height=7.2cm]{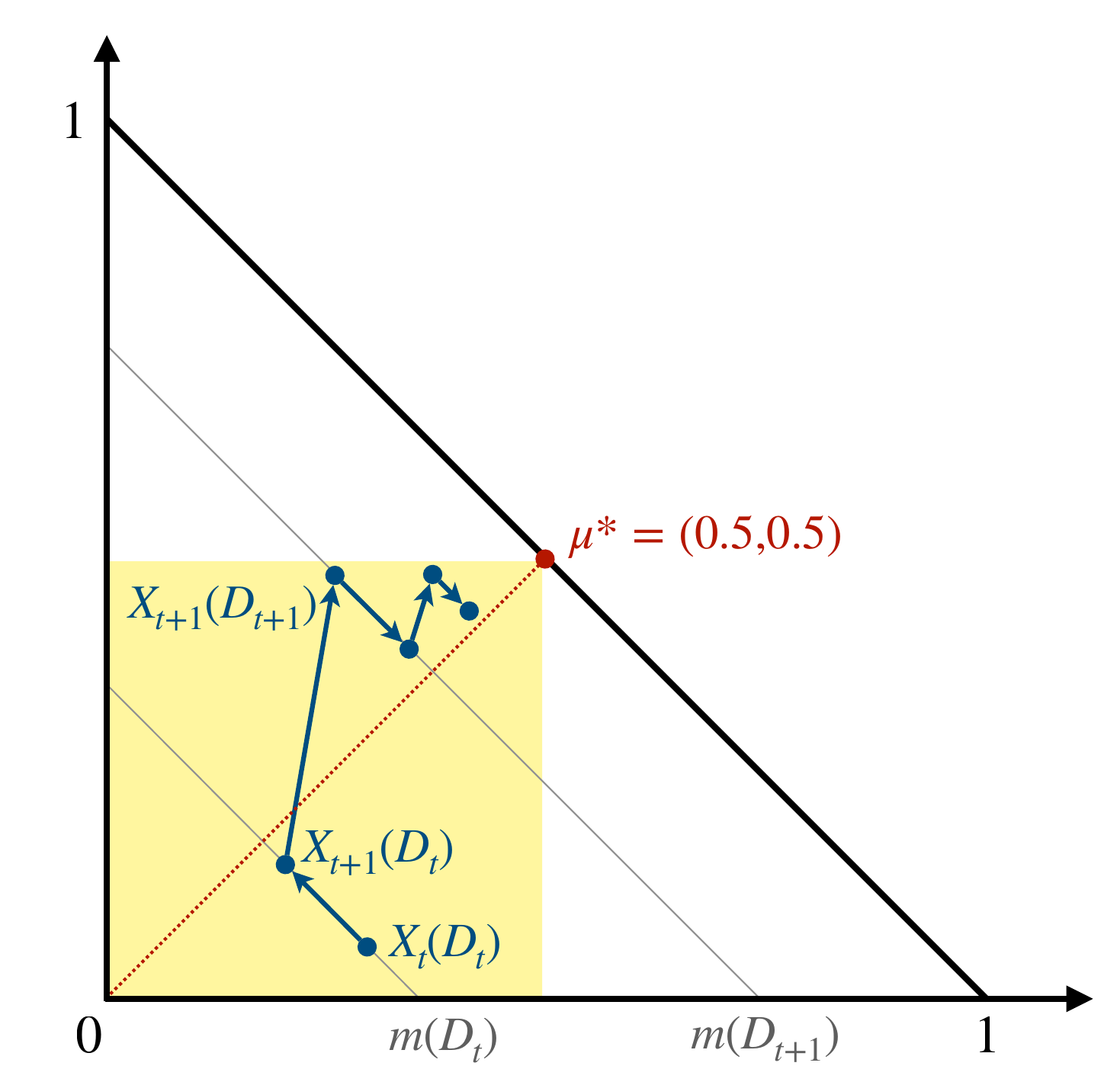}}
    \hfill
    \subcaptionbox{Intuitions of the ID policy. \label{fig:id-union-state-space}}
    {\includegraphics[height=7.2cm]{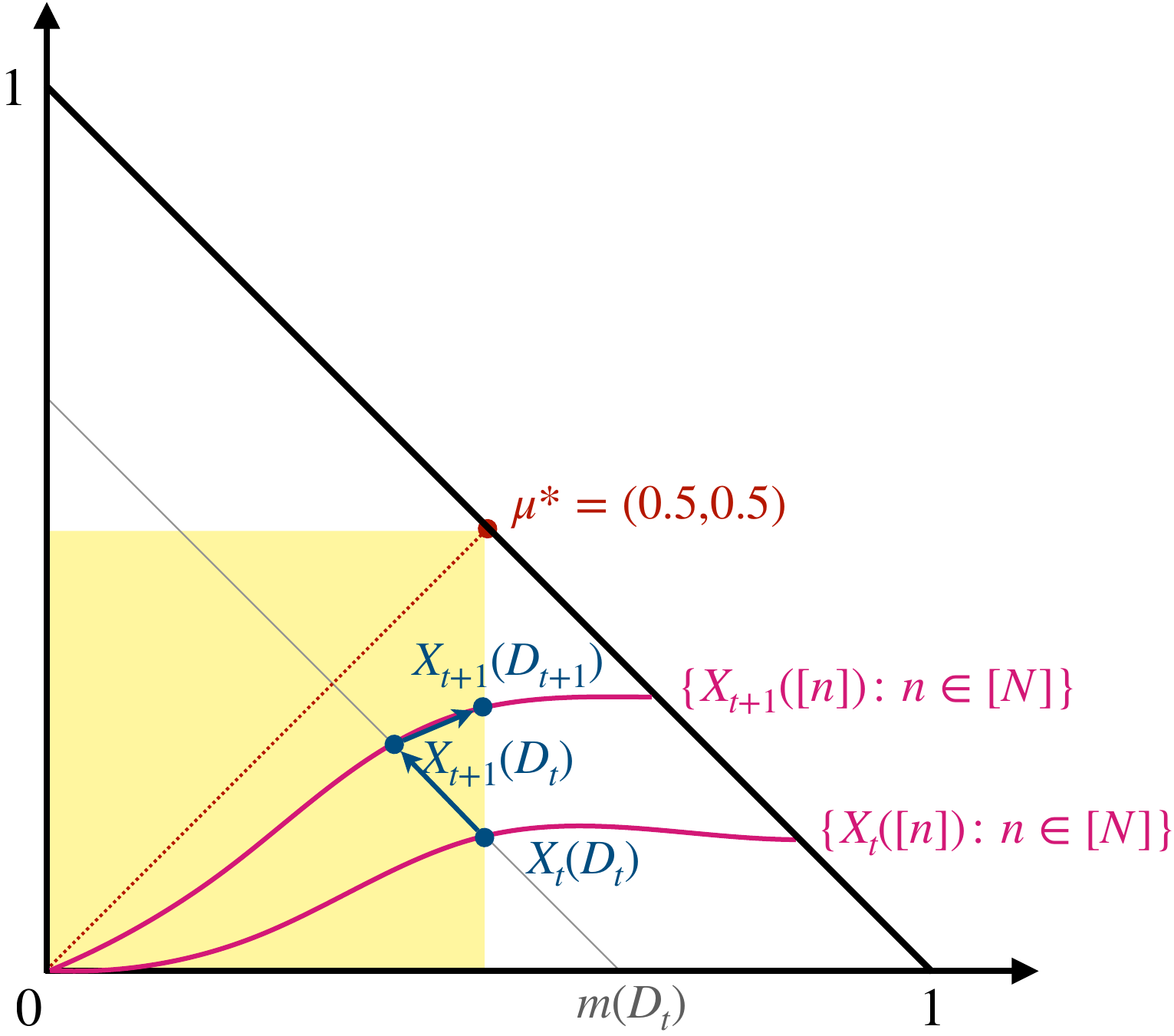}}
    }
    {Intuitions of the set-expansion policy and the ID policy through the example in \Cref{fig:union-state-space:two-state-mdp}. \label{fig:union-state-space}}
    {\textbf{(a)} Under the set-expansion policy, the solid dots and arrows illustrate a typical trajectory of the alternating updates of the system state $X_t$ and the focus set $D_t$. Specifically, when $X_t$ transitions to $X_{t+1}$, the scaled state-count vector $X_t(D_t)$ updates to $X_{t+1}(D_t)$, moving towards the central dotted line $\{m\statdist\colon m\in [0,1]_N\}$ since most arms in $D_t$ follow the optimal single-armed policy $\pibs$. 
    Subsequently, when the set $D_t$ expands to $D_{t+1}$, $X_{t+1}(D_t)$ updates to $X_{t+1}(D_{t+1})$,
    moving up and to the right as much as possible while remaining in the shaded region. 
    \textbf{(b)} Under the ID policy, the system state at time $t$ is visualized by the solid curve  $\{X_t([n]) \colon n\in [N]\}$. The segment of this curve in the shaded region corresponds to the arms following $\pibs$, so this segment drifts towards the central dotted line $\{m\statdist\colon m\in [0,1]_N\}$ in the next time step. 
    To enable a similar analysis to the set-expansion policy, 
    we intuitively define the focus set $D_t$ such that $X_t(D_t)$ is the point where the solid curve intersects the boundary of the shaded region.}
\end{figure}

Next, we use an example to provide some more concrete intuition. For illustration purposes, let us temporarily suppose that the $L_1$ norm between $\statdist$ and any distribution on $\sspa$ strictly decreases after right-multiplying $P_\pibs$. 

Consider the RB problem defined by the two-state single-armed MDP in \Cref{fig:union-state-space:two-state-mdp}. For any system state $x$ and any subset of arms $D$, the scaled state-count vector on $D$, $x(D) = (X(D,0), X(D,1))$, can be represented by a point in the triangle $\{(a,b)\colon a\geq 0, b\geq 0, a+b\leq 1\}$, depicted in \Cref{fig:union-state-space:subset-dynamics}. 
The shaded region where all arms in $D$ can follow $\pibs$ is $[0,0.5]\times [0,0.5]$, marked in yellow. 
One can verify that the constraint $\slk(X_t, D_t) \geq 0$ exactly keeps $X_t(D_t)$ in the shaded region. In fact, $\slk(X_t, D_t)$ is proportional to the $L_1$ distance between $X_t(D_t)$ and the boundary of the shaded region, and is non-negative if $X_t(D_t)$ stays within this region. 
Consequently, under the set-expansion policy, in each time step, most of the arms in $D_t$ follow $\pibs$, causing $X_{t+1}(D_t)$ to move closer to $m(D_t)\statdist$ (on the central dotted line) than $X_t(D_t)$; 
when the set-expansion policy updates $D_t$ to $D_{t+1}$, it maximizes $m(D_{t+1})$ under constraints $\slk(X_{t+1}, D_{t+1})\geq 0$ and $D_{t+1}\supseteq D_t$, so $X_{t+1}(D_{t+1})$ moves to the upper right of $X_{t+1}(D_t)$ but still stays in the shaded region. 
As these steps repeat, the sequence $X_1(D_1), X_{2}(D_1), X_{2}(D_{2}), X_{3}(D_{2}), X_3(D_3) \dots$ converges in a zigzag fashion to a small neighborhood of $\statdist$, as illustrated by the solid arrows in \Cref{fig:union-state-space:subset-dynamics}. 

Based on the above intuition, we can formally prove that the set-expansion policy is asymptotically optimal, as stated in \Cref{thm:set-expansion-policy} below. The proof of \Cref{thm:set-expansion-policy} is provided in \Cref{sec:proof-set-expansion-policy}. 

\begin{theorem}[Optimality gap of set-expansion policy]\label{thm:set-expansion-policy}
    Consider an $N$-armed restless bandit problem with the single-armed MDP $(\sspa, \aspa, P, r)$ and budget $\alpha N$ for $0< \alpha < 1$. 
    Assume that the optimal single-armed policy induces an aperiodic unichain (\Cref{assump:aperiodic-unichain}). Let $\pi$ be the set-expansion policy (\Cref{alg:set-expansion}). Then for all $N$ and initial state vector $\veS_0$, the optimality gap of $\pi$ is bounded as
    \begin{equation}
    \label{eq:set-expansion-policy-bound}
        \ropt - \rsysn \leq \frac{\constse}{\sqrt{N}},
    \end{equation}
    where $\constse$ is a constant depending on $\rmax$, $|\sspa|$, $\rhoBudget \triangleq \min\{\alpha, 1-\alpha\}$, and $P_{\pibs}$; the explicit expression of $\constse$ is given in the proof. 
\end{theorem}

Theorem~\ref{thm:set-expansion-policy} shows that under \Cref{assump:aperiodic-unichain}, the set-expansion policy is asymptotically optimal with an $O(1/\sqrt{N})$ optimality gap. 
The bound \eqref{eq:set-expansion-policy-bound} holds for any finite $N$, and only depends on intuitive quantities, including the problem primitives and a quantity reflecting the mixing time of the transition matrix $P_\pibs$.

Finally, note that Lines~\ref{alg:set-expansion-action-rect-starts}--\ref{alg:set-expansion-action-rect-ends} of \Cref{alg:set-expansion} is just one way to implement the action rectification step of the focus-set policy by breaking ties uniformly at random. 
The analysis of the set-expansion policy still goes through as long as the constraints in Lines~\ref{alg:focus-set-action-rect-1}--\ref{alg:focus-set-action-rect-2} of \Cref{alg:focus-set} are satisfied. 
In \Cref{sec:experiments}, we will also consider an alternative way for action rectification for the set-expansion policy, where we let arms outside the focus-sets follow the LP index policy \cite{GasGauYan_23_exponential}, which sometimes leads to a better empirical performance.

\subsection{ID policy}
\label{sec:results_ID-policy}

Next, we introduce another instance of the focus-set policies, named the ID policy, whose pseudocode is given in \Cref{alg:id}. 
The ID policy first samples an ideal action $\syshat{A}_t(i)$ for each arm $i\in[N]$ using $\pibs$. It then proceeds sequentially through the arms by their IDs ($i=1,2,\dots, N$), assigning $A_t(i) = \syshat{A}_t(i)$ for as many arms as the budget constraint allows. 
All subsequent arms with higher IDs are then forced to take the same action (either all $0$ or all $1$).

\begin{algorithm}[t] 
\caption{ID policy}
\label{alg:id}
\begin{flushleft}
\hspace*{\algorithmicindent} \textbf{Input}: number of arms $N$, budget $\alpha N$, the optimal single-armed policy $\pibs$,\\
\hspace*{\algorithmicindent} \hspace{0.37in}
initial system state $X_0$, initial state vector $\veS_0$
\end{flushleft}
\begin{algorithmic}[1]
    \For{$t=0,1,2,\dots$}
        \State Independently sample $\syshat{A}_t(i)\sim \pibs(\cdot | S_t(i))$ for $i\in[N]$
        \Comment{\ul{\emph{Action sampling}}}
        \If{$\sumN \syshat{A}_t(i) \geq \alpha N$}
        \Comment{\ul{\emph{Action rectification}}}
        \label{alg:id:act-rect-start}
            \State $\Ngood \gets \max \big\{n \leq N \colon \sum_{i=1}^{n} \syshat{A}_t(i) \leq \alpha N \big\}$
            \State $A_t(i) \gets \syshat{A}_t(i)$ for $i \leq \Ngood$, $A_t(i) \gets 0$ for $i > \Ngood$
        \Else \label{alg:id:act-rect-second-case-starts}
            \State $\Ngood \gets \max \big\{n\leq N \colon \sum_{i=1}^n (1-\syshat{A}_t(i)) \leq (1-\alpha) N \big\}$
            \State $A_t(i) \gets \syshat{A}_t(i)$ for $i \leq \Ngood$, $A_t(i) \gets 1$ for $i > \Ngood$
        \EndIf \label{alg:id:act-rect-ends}
        \State  Apply $A_t(i)$ and observe $S_{t+1}(i)$ for each arm $i\in[N]$
    \EndFor
\end{algorithmic}
\end{algorithm}

Although the focus set of the ID policy could in principle be any subset of the form $[n]$, for analysis purposes, we will construct a suitable one to use the same framework as the set-expansion policy. The formal definition of this focus set is deferred to \Cref{sec:proof-id-policy:md}. 
We will show that this constructed focus set has two properties: (1) most arms within it follow $\pibs$, and (2) the set expands almost monotonically over time to eventually contain most of the arms in the system. 
In \Cref{fig:id-union-state-space}, we illustrate the dynamics of $X_t$ under the ID policy and the choice of the focus set.

The ID policy again has an $O(1/\sqrt{N})$ optimality gap, as stated in \Cref{thm:id-policy} below; the proof of \Cref{thm:id-policy} is provided in Section~\ref{sec:proof-id-policy}.

\begin{theorem}[Optimality gap of ID policy]
\label{thm:id-policy}
    Consider an $N$-armed restless bandit problem with the single-armed MDP $(\sspa, \aspa, P, r)$ and budget $\alpha N$ for $0< \alpha < 1$. 
    Assume that the optimal single-armed policy induces an aperiodic unichain (\Cref{assump:aperiodic-unichain}). Let $\pi$ be the ID policy (\Cref{alg:id}). Then for all $N$ and initial state vector $\veS_0$, the optimality gap of $\pi$ is bounded as
    \begin{equation}\label{eq:id-policy-bound}
        \ropt - \rsysn \leq \frac{\constid}{\sqrt{N}}, 
    \end{equation}
    where $\constid$ is a constant depending on $\rmax$, $|\sspa|$, $\rhoBudget \triangleq \min\{\alpha, 1-\alpha\}$, and $P_{\pibs}$; the explicit expression of $\constid$ is given in the proof. 
\end{theorem}

\begin{remark}[Comparison between the ID policy and the set-expansion policy]
    The ID policy is simpler to implement and does not require an explicit calculation of the focus set. 
    Moreover, in simulations, the ID policy often performs slightly better than the set-expansion policy. 
    We also notice in simulations that a larger set of arms are able to persistently follow $\pibs$ under the ID policy than under the set-expansion policy (see Appendix~\ref{app:experiments:se-vs-id} for a closer investigation of this phenomenon). 
    On the other hand, under the ID policy, the arms with higher IDs have less chance to follow $\pibs$, whereas the set-expansion policy is ID-oblivious. 
\end{remark}

\begin{remark}
    Due to the homogeneity of the arms, it is actually sufficient to focus on the scaled state-count vector $X_t([N])$ as the state representation of the RB problem and design policies based on $X_t([N])$. In fact, most prior work on homogeneous RBs \cite{WebWei_90,Ver_16_verloop,GasGauYan_23_whittles,GasGauYan_23_exponential} focuses on $X_t([N])$ rather than $X_t$ as the state representation. 
    However, in our case, neither the set-expansion policy nor the ID policy is Markovian with respect to $X_t([N])$, which naturally leads to the question: Does there exist a focus-set policy that makes $X_t([N])$ a Markov chain and achieves the $O(1/\sqrt{N})$ optimality gap? 
    In Appendix~\ref{app:set-optimization-policy}, we construct another focus-set policy termed \emph{set-optimization policy} that indeed satisfies these two requirements. 
    The basic idea of the set-optimization policy is to update the focus set $D_t$ by minimizing a Lyapunov function of $X_t(D_t)$. 
\end{remark}

\subsection{Connection with existing policies}
\label{sec:connection-previous-policies}

As we intuitively explained in previous subsections, a focus-set policy achieves asymptotic optimality if most arms independently take actions according to the optimal single-armed policy $\pibs$ in the steady state. 
In this subsection, we consider the policies in the prior work, the LP-Priority policies \cite{Ver_16_verloop, GasGauYan_23_exponential} and the FTVA policy \cite{HonXieCheWan_23}, and investigate \emph{the number of arms that follow $\pibs$} under these policies.
This unified view through the number of arms that follow $\pibs$ can help us develop a better understanding of the connection among these policies and the roles of these assumptions in proving asymptotic optimality.

We first consider LP-Priority policies, which achieve asymptotic optimality under the GAP assumption \cite{Ver_16_verloop, GasGauYan_23_exponential}. 
Recall that an LP-Priority policy assigns a priority order to each state, and prioritizes activating the arms in the high-priority states. 
The priority orders must be compatible with the optimal single-armed policy $\pibs$ in the following way. Assuming that $y^*(s,1)+y^*(s,0)>0$ for all $s\in\sspa$, the state space can be partitioned into three subsets: 
$S^+ \triangleq \{s\in\sspa \colon \pibs(1|s) =1 \}$, $S^0 \triangleq \{s\in\sspa \colon 0 < \pibs(1|s) < 1\}$, and $S^- \triangleq \{s\in\sspa \colon \pibs(1|s) = 0\}$,
where one can always find a $\pibs$ such that $|S^0| \leq 1$ and we let $S^0=\{\sneu\}$ if $|S^0| = 1$. 
The definition of an LP-Priority policy requires the priorities of the states in $S^+$ to be higher than those in $S^0$, and further higher than those in $S^-$ (Definition 4.4 of \cite{Ver_16_verloop}).

Here is a heuristic way to see why GAP implies the asymptotic optimality of an LP-Priority policy: 
When GAP holds, the state-count vector of the system concentrates around $\statdist$ in the steady state (see, e.g., Lemma~12 in \cite{GasGauYan_23_whittles}). 
The relation $\sum_{s \in S^+} \statdist(s) + \sum_{s \in S^0} \statdist(s) \pibs(1|s) = \sum_{s \in \sspa} y^*(s,1) = \alpha$, along with the concentration of the state-count vector, suggests that, in an approximate sense, all arms in $S^+$ are activated, a fraction $\pibs(1|\sneu)$ of arms in $S^0$ are activated if $|S^0|=1$, and all arms in $S^-$ remain passive. 
Thus, in the steady state, for each state $s\in\sspa$, the fraction of active arms under an LP-Priority policy approximately coincides with the fraction of active arms if the actions were sampled by $\pibs$. Consequently, nearly all arms can be considered as following $\pibs$. 
On the other hand, when GAP fails, the scaled state-count vector may significantly deviate from $\statdist$, making it infeasible to activate a $\pibs(1|s)$ fraction of arms for each state $s\in\sspa$. In this case, only a limited subset of arms can be considered as following $\pibs$.

Next, we discuss the FTVA policy from \cite{HonXieCheWan_23}, an asymptotically optimal policy that relies on the SA condition. FTVA is simulation-based: it first generates ideal actions by simulating a virtual, unconstrained system where each arm independently follows $\pibs$. It then lets as many arms follow these virtual actions as the budget constraint allows.
The asymptotic optimality of FTVA hinges on the SA condition, which guarantees that the real states of most arms remain aligned with their virtual counterparts, thereby ensuring that their real actions are effectively drawn from $\pibs$ conditioned on the real states. 

FTVA is similar to our policies in the sense that it also maintains a set of arms that follow $\pibs$ persistently for a long time (referred to as the ``good arms'' in \citep{HonXieCheWan_23}). However, for the rest of the arms whose virtual and real states are misaligned, FTVA \emph{passively} waits for them to re-align on their own, which is guaranteed to happen only when SA holds. 
In contrast, our policies \emph{actively} expand the focus set whenever the empirical state distribution of the focus set is sufficiently close to $\statdist$, which is guaranteed to happen if the optimal single-armed policy $\pibs$ induces an aperiodic unichain.

\subsection{Alternative formulation with inequality budget constraint}
\label{sec:inequality}
The set-expansion policy and the ID policy can be extended to the alternative formulation of the RB problem with the inequality budget constraint: 
\begin{subequations}
\label{eq:N-arm-formulation-inequality}
\begin{align}
    \label{eq:objective-inequality}
    \underset{\text{policy } \pi}{\text{maximize}} & \quad \rliminf  \\
    \label{eq:hard-budget-constraint-inequality}
    \text{subject to}  
    &\quad  \sumN A_t^\pi(i) \leq \alpha N,\quad \forall t= 0,1,2,\dots 
\end{align} 
\end{subequations}
The modified pseudo-code of the two policies is given in Algorithms~\ref{alg:set-expansion-inequality} and \ref{alg:id-inequality} of Appendix~\ref{app:inequality-constraint}. The changes in the policies are summarized below. 
\begin{itemize}
    \item The equality budget constraint of the LP relaxation \eqref{eq:expect-budget-constraint} changes to an inequality. 
    \item In the action rectification step of the set-expansion policy (Lines~\ref{alg:set-expansion-action-rect-second-case-starts}--\ref{alg:set-expansion-action-rect-ends} of \Cref{alg:set-expansion}), we should instead take $A_t(i) = \syshat{A}_t(i)$ for all $i\in D_t$ whenever $\sum_{i\in D_t} \syshat{A}_t(i) \leq \alpha N$. 
    \item In the action rectification step of the ID policy (Lines~\ref{alg:id:act-rect-second-case-starts}--\ref{alg:id:act-rect-ends} of \Cref{alg:id}), we should instead take $\Ngood = N$ and $A_t(i) = \syshat{A}_t(i)$ for all $i\in[N]$ whenever $\sumN \syshat{A}_t(i) \leq \alpha N$. 
\end{itemize}

We show that the modified set-expansion policy and the modified ID policy also achieve $O(1/\sqrt{N})$ optimality gaps for the inequality-constraint setting, as stated in the following theorems. 

\begin{restatable}{theorem}{seineq}
    \label{thm:set-expansion-inequality}
    Consider the $N$-armed restless bandit problem with the inequality constraint \eqref{eq:N-arm-formulation-inequality} and assume that the optimal single-armed policy $\pibs$ induces an aperiodic unichain. Let $\pi$ be the set-expansion policy defined in \Cref{alg:set-expansion-inequality}. For any $N$ and initial state vector $\veS_0$, the optimality gap of $\pi$ is bounded as
    \begin{equation}
        \ropt - \rsysn \leq \frac{\constse'}{\sqrt{N}}, 
    \end{equation}
    for some constant $\constse'$ depending on $\rmax$, $|\sspa|$, $\rhoBudget \triangleq \min\{\alpha, 1-\alpha\}$, and $P_{\pibs}$. 
\end{restatable}

\begin{restatable}{theorem}{idineq}
    \label{thm:id-inequality}
    Consider the $N$-armed restless bandit problem with the inequality constraint \eqref{eq:N-arm-formulation-inequality} and assume that the optimal single-armed policy $\pibs$ induces an aperiodic unichain. Let $\pi$ be the ID policy defined in \Cref{alg:id-inequality}. For any $N$ and initial state vector $\veS_0$, the optimality gap of $\pi$ is bounded as
    \begin{equation}
        \ropt - \rsysn \leq \frac{\constid'}{\sqrt{N}}, 
    \end{equation}
    for some constant $\constid'$ depending on $\rmax$, $|\sspa|$, $\rhoBudget \triangleq \min\{\alpha, 1-\alpha\}$, and $P_{\pibs}$. 
\end{restatable}

The proofs of Theorems~\ref{thm:set-expansion-inequality} and \ref{thm:id-inequality} are almost identical to the proofs of Theorems~\ref{thm:set-expansion-policy} and \ref{thm:id-policy} with the only differences lying in the proofs of Lemmas~\ref{lem:set-expansion-pibs-consistency} and \ref{lem:id-pibs-consistency}. We refer the reader to Appendix~\ref{app:inequality-constraint} for the details.

\section{A meta-theorem for focus-set policies and the proof}
\label{sec:formalizing}

In this section, we establish a meta-theorem, \Cref{thm:focus-set-policy}, which provides sufficient conditions for a focus-set policy to have an $O(1/\sqrt{N})$ optimality gap. The meta-theorem and its conditions are stated in \Cref{sec:meta-theorem}, followed by its proof in \Cref{sec:proof-meta-theorem}. 

The meta-theorem contains the main technical novelty of our analysis. In the subsequent sections, we will simply verify that the set-expansion policy and the ID policy satisfy these conditions under the aperiodic unichain assumption, thereby proving the optimality gap bounds in \Cref{thm:set-expansion-policy} and \Cref{thm:id-policy}.

\subsection{Meta-theorem on $O(1/\sqrt{N})$ optimality gaps of focus-set policies}
\label{sec:meta-theorem}

We now state a set of conditions which, once satisfied by a focus-set policy, guarantees an $O(1/\sqrt{N})$ optimality gap.

To begin with, we define a class of functions called the \emph{subset Lyapunov functions}, which are indexed by a collection of subsets $D\subseteq  [N]$. 
The subset Lyapunov function indexed by $D$ upper bounds the distance between $x(D)$ and $m(D) \statdist$, and decreases geometrically if the arms in $D$ follow the optimal single-armed policy $\pibs$ indefinitely. 
The formal definition is given below.

\begin{definition}[Subset Lyapunov functions]\label{def:feature-lyapunov}
    Let $\mathcal{D}$ be a collection of subsets of $[N]$. 
    Consider a class of functions $\{h(\cdot,\dynset) \colon \dynset \in \mathcal{D} \}$, where each $h(\cdot, \dynset)$ maps a system state $x$ to a real value that depends only on the states of the arms in $D$. 
    This class of functions is referred to as the \emph{subset Lyapunov functions} for the policy $\pibs$ if they satisfy the following conditions: 
    \begin{enumerate}
        \item (Drift condition for a fixed $D$). There exist constants $\rhoContr \in (0,1)$ and  $\constHnoise > 0$ such that for any $D\in\mathcal{D}$ and any system state $x$,
        \begin{equation}\label{eq:feature-lyapunov:drift}
            \mathbb{E}\big[h(X_1, D) \givenbig X_0 = x, A_0(i)\sim \pibs(\cdot |S_0(i)) \, \forall i\in D \big] \leq \rhoContr h(x, D) + \frac{\constHnoise}{\sqrt{N}}.
        \end{equation}
        \item (Distance domination). There exists a constant $\constVnorm > 0$ such that for any $D\in\mathcal{D}$ and any system state $x$,
        \begin{equation}\label{eq:feature-lyapunov:strength}
            h(x, D) \geq \constVnorm \norm{x(D) - m(D) \statdist}_1.
        \end{equation} 
        \item (Lipschitz continuity in $D$). There exists a constant $\liph > 0$ such that for any $D,D'\in\mathcal{D}$ with $D\subseteq D'$ and any system state $x$,
        \begin{equation}\label{eq:feature-lyapunov:lipschitz}
            \abs{h(x, D') - h(x, D)} \leq \liph \big(m(D') - m(D)\big). 
        \end{equation}
    \end{enumerate}
\end{definition}

As an example, in \Cref{sec:proof-set-expansion-policy}, we will define a weighted $L_2$ norm, $\norm{v}_{\wmat} \triangleq \sqrt{v\wmat v^\top}$ for some weight matrix $\wmat$. We will show that the class of functions $\{\hw(\cdot, D)\}_{D\subseteq[N]}$ with $\hw(x, D) = \norm{x(D) - m(D)\statdist}_\wmat$ 
satisfies the definition of subset Lyapunov functions.

While the subset Lyapunov function $h(\cdot,D)$ is constructed to witness the convergence of $X_t(D)$ to $m(D) \statdist$ for a \emph{fixed} set $D$, in a focus-set policy, the set $D_t$ is not fixed but rather is chosen dynamically.
Below we introduce three conditions on $D_t$, which would allow us to use the subset Lyapunov functions to establish the asymptotic optimality of a focus set policy. 

Condition~\ref{def:focus-set:pibs-consistency} requires that most arms in the focus set $D_t$ conform to the actions sampled from $\pibs$.

\begin{customcond}{1}[Majority conformity]\label{def:focus-set:pibs-consistency}
    Let $\constAction>0$ be a constant.
    For any time step $t\geq0$, there exists $D_t'\subseteq D_t$ such that for any $i\in D_t'$, the policy chooses $A_t(i) = \syshat{A}_t(i)$, and 
    \begin{equation}
        \Ebig{m(D_t \backslash D_t') \givenbig X_t, D_t} \leq \frac{\constAction}{\sqrt{N}} \quad a.s.
    \end{equation}
\end{customcond}

Condition~\ref{def:focus-set:monotonic} requires that $D_t$ changes in a set-inclusive manner and does not shrink much in expectation. 

\begin{customcond}{2}[Almost non-shrinking]\label{def:focus-set:monotonic} 
    For any time step $t\geq0$, either $D_{t+1} \supseteq D_t$ or $D_{t+1} \subseteq D_t$. 
    Moreover, there exists a constant $\constMono > 0$ such that for any $t\geq0$, 
    \begin{equation}\label{eq:focus-set:monotonic}
        \Ebig{\big(m(D_t) - m(D_{t+1})\big)^+ \givenbig X_t, D_t} \leq  \frac{\constMono}{\sqrt{N}} \quad a.s.
    \end{equation}
\end{customcond}

Condition~\ref{def:focus-set:large-enough} requires that $m(D_t)$, the fraction of arms covered by $D_t$, is sufficiently large with respect to a subset Lyapunov function on $D_t$.

\begin{customcond}{3}[Sufficient coverage]\label{def:focus-set:large-enough} 
There exist a class of subset Lyapunov functions  $\{h(\cdot,\dynset) \colon \dynset \in \mathcal{D} \}$ 
and constants $\lipexpand > 0, \constEpnoise>0$ such that for any time step $t\geq0$, 
        \begin{equation}
            1 - m(D_t) \leq \lipexpand h(X_t, D_t) + \frac{\constEpnoise}{\sqrt{N}} \quad a.s.
        \end{equation}
\end{customcond}

We remark that Conditions~\ref{def:focus-set:pibs-consistency} and \ref{def:focus-set:monotonic} are generally easier to satisfy when the focus set $D_t$ is small, while Condition~\ref{def:focus-set:large-enough} requires $D_t$ to be large.

We are now ready to state the meta-theorem, which establishes an $O(1/\sqrt{N})$ bound on the optimality gap of a focus-set policy that satisfies the above conditions. 

\begin{theorem}[Meta-theorem on optimality gap of focus-set policies]\label{thm:focus-set-policy}
    Consider an $N$-armed restless bandit problem with the single-armed MDP $(\sspa, \aspa, P, r)$ and budget $\alpha N$ for $0< \alpha < 1$. 
    Assume that the optimal single-armed policy induces an aperiodic unichain (\Cref{assump:aperiodic-unichain}). 
    Let $\pi$ be a focus-set policy given in \Cref{alg:focus-set} satisfying Conditions~\ref{def:focus-set:pibs-consistency}, \ref{def:focus-set:monotonic}, and \ref{def:focus-set:large-enough} for a class of subset Lyapunov functions $\{h(\cdot, D)\}_{D\in\mathcal{D}}$, then 
    \begin{equation}
        \ropt - \rsysn \leq \rmax \left(\Big(\frac{1}{\constVnorm} + \frac{2}{\liph}\Big) \frac{K_1}{1-\rhoFinal} + 2\constAction\right) \frac{1}{\sqrt{N}}  \quad \forall N, \veS_0,
    \end{equation}     
    where $\rhoFinal = 1 - (1-\rhoContr)/(1+ \liph \lipexpand)$ and $K_1 = \constHnoise + 2\liph\constAction + 2\liph\constMono + (1-\rhoContr)L_h\constEpnoise/(1+ \liph \lipexpand)$. 
\end{theorem}

\subsection{Proof of \Cref{thm:focus-set-policy} assuming convergence to stationary distribution.}
\label{sec:proof-meta-theorem}

In this section, we prove \Cref{thm:focus-set-policy} under the assumption that the focus set policy induces a Markov chain that converges to a unique stationary distribution. 
This simplified setting allows us to present the key proof ideas more clearly. 
The proof for the general case follows essentially the same line of argument and is given in Appendix~\ref{app:proof-meta-thm-general}. The general proof also bounds the finite-time expected reward as a by-product.

Under the above assumption, we use $\veS_\infty$, $\syshat{\veA}_{\infty},$ $\veA_\infty$, $X_\infty$, $D_\infty$ to denote the random variables following the stationary distributions of $\veS_t$, $\syshat{\veA}_{t},$ $\veA_t$, $X_t$, $D_t$, respectively. With this notation, the long-run average reward of the policy $\pi$ is equal to $\rsysn =  \frac{1}{N} \sumN \Ebig{r(S_\infty(i), A_\infty(i))}$.

\begin{proof}{\textit{Proof of \Cref{thm:focus-set-policy} assuming convergence to stationary distribution}.}
Our proof is structured into two steps: understanding the optimality gap, and bounding the Lyapunov function.

\paragraph{\textbf{Understanding the optimality gap.}}
    Recall that the optimality gap can be upper bounded as $\ropt - \rsysn \leq \rrel - \rsysn$, where $\rrel$ is the expected reward associated with the optimal steady-state state-action distribution $y^* = (y^*(s,a))_{s\in\sspa,a\in\aspa}$.  Then
    \begin{align}
        & \ropt - \rsysn \nonumber\\
        &\qquad \leq \rrel - \rsysn \nonumber\\
        &\qquad = \sumsa r(s,a) y^*(s,a)  - \frac{1}{N} \sumN \EBig{r(S_\infty(i), A_\infty(i))} \nonumber\\
        &\qquad \leq \sumsa r(s,a) y^*(s,a)  - \frac{1}{N} \sumN \EBig{r(S_\infty(i), \syshat{A}_\infty(i))} + \frac{2\rmax}{N} \sumN \ProbBig{\syshat{A}_\infty(i)\neq A_\infty(i)} \nonumber\\
        &\qquad \leq \sumsa r(s,a)\Big( y^*(s,a) - \pibs(a|s) \Ebig{X_\infty([N], s)}\Big) + 2\rmax  \Ebig{1 - m(D_\infty')} \nonumber\\
        &\qquad = \sumsa r(s,a)\pibs(a|s) \Big(\statdist(s)-\Ebig{X_\infty([N], s)}\Big)  + 2\rmax  \Ebig{1 - m(D_\infty')} \nonumber \\
        &\qquad \leq \rmax \Ebig{\normbig{\statdist - X_\infty([N])}_1} + 2\rmax  \Ebig{1 - m(D_\infty)} +  \frac{2\rmax\constAction}{\sqrt{N}}, \label{eq:opt-gap-bound-1}
    \end{align}
    where $D_\infty'$ is the subset of $D_\infty$ assumed in \Cref{def:focus-set:pibs-consistency}, which satisfies $m(D_\infty') \geq m(D_\infty) - \constAction /\sqrt{N}$. 
    Therefore, to bound the optimality gap, it suffices to bound $\Ebig{\normbig{\statdist - X_\infty([N])}_1}$, which is the distributional distance, and $\Ebig{1 - m(D_\infty)}$, which is the size of the complement of the focus set.

    In this proof, we construct a Lyapunov function that can be viewed as an upper bound on a weighted sum of the two terms in \eqref{eq:opt-gap-bound-1}.  In particular, consider the following Lyapunov function
    \begin{equation}
        V(x, D) = h(x, D) + \liph (1-m(D)).
    \end{equation}
    
    Let us first see how the terms in \eqref{eq:opt-gap-bound-1} are upper bounded by $\E{V(X_\infty, D_\infty)}$.
    For the first term, it is easy to see that $\constVnorm \norm{\statdist - X_\infty([N])}_1 \leq h(X_\infty, [N])$ by the distance domination property of $h$ (cf.\ Equation \eqref{eq:feature-lyapunov:strength}).  Then by the Lipschitz continuity of $h$, we have $h(X_\infty, [N]) \le h(X_\infty, D_{\infty})+\liph (1-m(D_\infty))=V(X_\infty, D_{\infty})$.  Thus, $\Ebig{\normbig{\statdist - X_\infty([N])}_1}\le \Ebig{V(X_\infty, D_{\infty})}/\constVnorm$.
    For the second term, clearly $\Ebig{1 - m(D_\infty)}\le \Ebig{V(X_\infty, D_{\infty})}/\liph$.
    Therefore, the upper bound in \eqref{eq:opt-gap-bound-1} can be further bounded as
    \begin{equation}\label{eq:opt-gap-bdd-by-v}
        \ropt - \rsysn \le \rmax \left(\frac{1}{\constVnorm} + \frac{2}{\liph}\right)\Ebig{V(X_\infty, D_{\infty})} + \frac{2\rmax\constAction}{\sqrt{N}}, 
    \end{equation}
    which makes it sufficient to bound $\Ebig{V(X_\infty, D_{\infty})}$.

    \paragraph{\textbf{Bounding the Lyapunov function.}}
    We establish an upper bound on $\Ebig{V(X_\infty, D_{\infty})}$ by proving the following drift condition: 
    for any $t\ge 0$, 
    \begin{equation}\label{eq:thm4-proof:drift-condition}
        \Ebig{V(X_{t+1}, \dynset_{t+1}) \givenbig X_t, D_t} \leq  \rhoFinal V(X_t, \dynset_t) + \frac{K_1}{\sqrt{N}},
    \end{equation}
    for some constants $\rhoFinal\in(0,1)$ and $K_1>0$. To prove \eqref{eq:thm4-proof:drift-condition}, observe that for any time step $t\geq0$, 
    \begin{align}
        &V(X_{t+1}, \dynset_{t+1}) 
        = h(X_{t+1}, \dynset_{t+1})  + \liph(1-m(\dynset_{t+1})) \nonumber \\
        &\qquad\leq \Bigl(h(X_{t+1}, \dynset_t) + \liph\big\lvert m(\dynset_{t+1}) - m(\dynset_t)\big\rvert \Bigr) + \Bigl(\liph(1-m(\dynset_t)) + \liph(m(\dynset_t) - m(\dynset_{t+1})) \Bigr)  \nonumber \\
        &\qquad= h(X_{t+1}, \dynset_t)  + \liph(1-m(\dynset_t)) + 2\liph \big(m(\dynset_t) -  m(\dynset_{t+1})\big)^+, \label{eq:thm4-proof:drift-intermediate-0}
    \end{align}
    where we have used the facts that $D_{t+1} \supseteq D_t$ or $D_{t+1}\subseteq D_t$ (\Cref{def:focus-set:monotonic}) and the Lipschitz continuity of $h(x, D)$ in $D$. 
    Subtracting $V(X_t, D_t)$ and taking the expectation, we obtain the  \emph{key decomposition} below: 
    \begin{align}
        \Ebig{V(X_{t+1}, \dynset_{t+1})\givenbig X_t, D_t} - V(X_t, \dynset_t) 
        &\leq \Ebig{h(X_{t+1}, \dynset_t) \givenbig X_t, D_t} - h(X_t, \dynset_t) \label{eq:thm4-proof:drift-intermediate-1}\\
        &\mspace{14mu} + 2 \liph \Ebig{\big(m(\dynset_t) - m(\dynset_{t+1})\big)^+\givenbig X_t, D_t}.\label{eq:thm4-proof:drift-intermediate-2}
    \end{align}
    where the term in  \eqref{eq:thm4-proof:drift-intermediate-1} represents the contribution of state transitions to the drift of $V(X_t, D_t)$, and the term in \eqref{eq:thm4-proof:drift-intermediate-2} represents the contribution of set updates. 
    
    We first upper bound the term $\E{h(X_{t+1}, \dynset_t) \givenplain X_t, D_t} - h(X_t, \dynset_t)$ in \eqref{eq:thm4-proof:drift-intermediate-1}.  Note that this bound would immediately follow from the drift condition of subset Lyapunov functions if all the arms in $D_t$ were to follow the ideal actions. 
    By the majority conformity property of the focus set $D_t$ (\Cref{def:focus-set:pibs-consistency}), there exists $D_t'\subseteq D_t$ such that for any $i\in D_t'$, the policy chooses $A_t(i) = \syshat{A}_t(i)$, and $\E{m(D_t\backslash D_t') \givenplain X_t, D_t} = O(1/\sqrt{N})$. Let $X_{t+1}'$ be a random element denoting the system state at time $t+1$ if $A_t(i) = \syshat{A}_t(i)$ for all $i\in [N]$. We couple $X_{t+1}$ with $X_{t+1}'$ such that they have the same states on the set $D_t'$, and thus $h(X_{t+1}, D_t') = h(X_{t+1}', D_t')$. Then
    \begin{align*}
        & \Ebig{h(X_{t+1}, D_t) \givenbig X_t, D_t} \nonumber\\
        &\qquad =  \Ebig{h(X_{t+1}', D_t) + \big(h(X_{t+1}, D_t) - h(X_{t+1}', D_t)\big) \givenbig X_t, D_t}\nonumber\\
        &\qquad= \Ebig{h(X_{t+1}', D_t) + \big(h(X_{t+1}, D_t) - h(X_{t+1}, D_t')\big) + \big(h(X_{t+1}', D_t') - h(X_{t+1}', D_t)\big)  \givenbig X_t, D_t}\\
        &\qquad\leq \rhoContr h(X_t, D_t) + \frac{\constHnoise}{\sqrt{N}}  + 2\liph \Ebig{ m(D_t\backslash D_t') \givenbig X_t, D_t}\\
        &\qquad\leq \rhoContr h(X_t, D_t) + \frac{\constHnoise + 2\liph \constAction}{\sqrt{N}},
    \end{align*}
    where we have used the drift condition and the Lipschitz continuity of $h$. 
    It follows that
    \begin{equation}
        \Ebig{h(X_{t+1}, D_t) \givenbig X_t, D_t} - h(X_t, D_t) \leq -(1-\rhoContr) h(X_t, D_t) + \frac{\constHnoise + 2\liph \constAction}{\sqrt{N}}.
    \end{equation}
    Next, to bound the term in \eqref{eq:thm4-proof:drift-intermediate-2}, we simply apply \Cref{def:focus-set:monotonic}:
    \begin{equation}
        2 \liph \Ebig{\big(m(\dynset_t) - m(\dynset_{t+1})\big)^+\givenbig X_t, D_t} \le \frac{2 \liph\constMono}{\sqrt{N}}.
    \end{equation}
    Combining the above bounds for \eqref{eq:thm4-proof:drift-intermediate-1} and \eqref{eq:thm4-proof:drift-intermediate-2}, we get
    \begin{align}\label{eq:thm4-proof:drift-intermediate-3}
       \Ebig{V(X_{t+1}, \dynset_{t+1})\givenbig X_t, D_t} - V(X_t, \dynset_t) 
        \leq  -(1-\rhoContr) h(X_t, D_t)  + \frac{\constHnoise + 2\liph\constAction + 2\liph\constMono}{\sqrt{N}}. 
    \end{align}

    To get \eqref{eq:thm4-proof:drift-condition}, it remains to upper bound the $-(1-\rhoContr) h(X_t, D_t)$ term. 
    By the sufficient coverage condition (\Cref{def:focus-set:large-enough}), $1 - m(D_t) \leq \lipexpand h(X_t, D_t) + \constEpnoise / \sqrt{N}$,  so 
    \[
        V(X_t, D_t) = h(X_t, D_t) + \liph (1-m(D_t))
        \le (1+ \liph \lipexpand) h(X_t, D_t)  + \frac{\liph \constEpnoise}{\sqrt{N}}.
    \]
    Upper bounding the $-(1-\rhoContr) h(X_t, D_t)$ term in \eqref{eq:thm4-proof:drift-intermediate-3} using the above inequality, we get 
    \begin{equation*}
        \Ebig{V(X_{t+1}, D_{t+1}) \givenbig X_t, D_t}  \leq \rhoFinal V(X_t, D_t) + \frac{K_1}{\sqrt{N}},
    \end{equation*}
    where $\rhoFinal = 1 - (1-\rhoContr)/(1+ \liph \lipexpand)$ and $K_1 = \constHnoise + 2\liph\constAction + 2\liph\constMono + (1-\rhoContr)\liph\constEpnoise/(1+ \liph \lipexpand)$. 
    This is the bound in \eqref{eq:thm4-proof:drift-condition} that we set out to prove.

    Taking the expectation on both sides of \eqref{eq:thm4-proof:drift-condition} letting $t\to\infty$, we have
    \[
        \E{V(X_\infty, D_\infty)} \leq \rhoFinal \E{V(X_\infty, D_\infty)} + \frac{K_1}{\sqrt{N}},
    \]
    which implies that 
    \begin{equation}
    \label{eq:proof_meta_1}
        \E{V(X_\infty, D_\infty)} \leq \frac{K_1}{(1-\rhoFinal)\sqrt{N}}.
    \end{equation}
    This completes the proof of \Cref{thm:focus-set-policy}. \Halmos
\end{proof}

\begin{remark}
    We conclude this section by a remark on our use of the \emph{bivariate Lyapunov functions} $h(x,D)$ and  $V(x, D) = h(x, D) + \liph(1-m(D))$. By definition, the subset Lyapunov function $h(x,D)$ depends on the system state $x$ only through $x(D).$ 
    This means that for fixed $D$, the drifts of $h(x,D)$ and $V(x, D)$ only depend on the state transitions of the arms in $D$. When $D$ is chosen appropriately, most arms in $D$ can follow $\pibs$ under the budget constraint, thus inheriting the convergence and concentration properties of the aperiodic unichain induced by $\pibs$.  
    Therefore, the auxiliary variable $D$ provides the flexibility of focusing on a subset of arms so that the drift is easy to bound and expanding the subset gradually to the entire system. 
    
    For the ID policy, $D_t$ is determined by the system state $X_t$, and hence $h(X_t,D_t)$ can be written as a function of $X_t$ alone. Even in this case, using a bivariate $h$ is beneficial, as it allows us to decouple the two variables---in particular,  quantities like $h(X_{t+1},D_t)$ play a prominent role in our proof of \Cref{thm:focus-set-policy}.
    
    Our use of bivariate Lyapunov functions departs from most prior work on the RB problem \cite{Whi_88_rb,WebWei_90,Ver_16_verloop,GasGauYan_23_whittles,GasGauYan_23_exponential}, whose analysis is in terms of the full system state $X_t([N])$, under which the dynamics of arms in a subset is less visible. 
    We expect that our approach is useful for a broader class of problems where the system state consists of multiple components, a subset of which have a more tractable dynamic at a given time. In this case, one may construct a Lyapunov function that can zoom into this more tractable subset and seek to gradually expand the subset based on the system state.  
\end{remark}

\section{Proof of Theorem~\ref{thm:set-expansion-policy} (Optimality gap of set-expansion policy)}
\label{sec:proof-set-expansion-policy}
In this section, we prove Theorem~\ref{thm:set-expansion-policy} using the framework established in Section~\ref{sec:formalizing}.
This section is organized as follows. 
In \Cref{sec:pf-set-exp:subset-lyapunov}, we define the subset Lyapunov functions for the set-expansion policy. 
In \Cref{sec:pf-set-exp:lemmas}, we present three lemmas verifying that the set-expansion policy satisfies Conditions~\ref{def:focus-set:pibs-consistency}, \ref{def:focus-set:monotonic} and \ref{def:focus-set:large-enough}, and prove Theorem~\ref{thm:set-expansion-policy} by citing the meta-theorem \Cref{thm:focus-set-policy}. 
These three lemmas are subsequently proved in Sections~\ref{sec:pf-set-exp:pf-majority-conformity}, \ref{sec:pf-set-exp:pf-non-shrinking} and \ref{sec:pf-set-exp:pf-sufficient-cov}.

\subsection{Subset Lyapunov functions}\label{sec:pf-set-exp:subset-lyapunov}
To construct the subset Lyapunov functions, we consider the $L_2$ norm weighted by a carefully constructed matrix $\wmat$ defined below. 

\begin{restatable}{definition}{wdef}\label{def:w-and-w-norm}
     Let $\wmat$ be an $|\sspa|$-by-$|\sspa|$ matrix given by  
    \begin{equation}\label{eq:w-def}
        \wmat = \sum_{k=0}^\infty (P_\pibs - \Xi)^k (P_\pibs^{\top} - \Xi^{\top})^k,
    \end{equation}
    where $\Xi$ is an $|\sspa|$-by-$|\sspa|$ matrix with each row being $\statdist$. 
    Let $\lamw$ denote maximal eigenvalue of $\wmat$. 
\end{restatable}

Intuitively, the infinite series in \eqref{eq:w-def} is convergent because $P_\pibs$ is aperiodic unichain, $\Xi$ is the limit of $P_\pibs^k$ as $k\to\infty$, and $(P_\pibs - \Xi)^k = P_\pibs^k - \Xi$. 
In Appendix~\ref{app:proof-norm-lemmas}, we formally show that the matrix $\wmat$ is well-defined and positive definite, with eigenvalues in the range $[1, \lamw]$. 
Our next lemma shows that $P_\pibs$ is a pseudo-contraction under the $\wmat$-weighted $L_2$ norm.

\begin{restatable}[Pseudo-contraction under $\wmat$-weighted $L_2$ norm]{lemma}{wnorm}\label{lem:pibar-one-step-contraction-W}
    Suppose $P_\pibs$ is an aperiodic unichain on $\sspa$. Then
    for any distribution $v\in \Delta(\sspa)$, 
    \begin{equation}\label{eq:pibar-contraction}
        \norm{v P_\pibs - \statdist}_\wmat \leq \Big(1 - \frac{1}{2\lamw}\Big) \norm{v - \statdist}_\wmat,
    \end{equation}
    where $\lamw$ is the maximal eigenvalue of the matrix $\wmat$ given in \Cref{def:w-and-w-norm}, and $\norm{\,\cdot\,}_\wmat$ is the $\wmat$-weighted $L_2$ norm, that is, $\norm{u}_\wmat = \sqrt{u \wmat u^\top}$ for any row vector $u \in \R^{|\sspa|}$. 
\end{restatable}

Now we are ready to define the subset Lyapunov functions. 
For any system state $x$ and $D\subseteq [N]$, let
\begin{equation}
    \label{eq:hw-def}
    \hw(x, D) = \norm{x(D) - m(D)\statdist}_W,
\end{equation}
which measures the distance between $x(D)$, the scaled state-count vector on $D$, and $m(D)\statdist$, the correspondingly scaled optimal stationary distribution. 
Note that $\hw(x, D)$ depends only on the states of the arms in $D$, as required by the definition of subset Lyapunov functions. 
The next lemma, \Cref{lem:hw-feature-lyaupnov}, shows that the class of functions $\{\hw(x, D)\}_{D\subseteq[N]}$ satisfies the definition of subset Lyapunov functions (\Cref{def:feature-lyapunov}). 
The proof of \Cref{lem:hw-feature-lyaupnov} is provided in Appendix~\ref{app:proof-feature-lyapunov-lemmas}. 

\begin{restatable}{lemma}{hwfeature}\label{lem:hw-feature-lyaupnov}
    The class of functions $\{\hw(\cdot, D)\}_{D\subseteq[N]}$ defined in \eqref{eq:hw-def} satisfies that for any system state $x$ and any pair of subsets $D,D'\subseteq [N]$ with $D\subseteq D'$, 
    \begin{align}
        \label{eq:hw-feature-lyapunov:drift}
        \Eplain{\hw(X_1, D) \givenplain X_0 = x, A_0(i)\sim \pibs(\cdot| S_0(i)) \, \forall i\in D } &\leq  \big(1-\frac{1}{2\lamw}\big) \hw(x, D) + \frac{2\lamw^{1/2}}{\sqrt{N}} \\
        \label{eq:hw-feature-lyapunov:strength}
        \hw(x, D) &\geq \frac{1}{|\sspa|^{1/2}} \norm{x(D) - m(D)\statdist}_1 \\
        \label{eq:hw-feature-lyapunov:lipschitz}
        \abs{\hw(x, D) - \hw(x, D')} &\leq \lipw (m(D') - m(D)), 
    \end{align}
    where the Lipschitz constant $\lipw = 2\lamw^{1/2}$. 
    These inequalities imply the drift condition, distance dominance property, and Lipschitz continuity in \Cref{def:feature-lyapunov}, respectively. Consequently, $\{\hw(x, D)\}_{D\subseteq[N]}$ is a class of subset Lyapunov functions for the optimal single-armed policy $\pibs$.     
\end{restatable}

\subsection{Lemmas verifying Conditions~\ref{def:focus-set:pibs-consistency}, \ref{def:focus-set:monotonic} and \ref{def:focus-set:large-enough}; Proof of Theorem~\ref{thm:set-expansion-policy}}\label{sec:pf-set-exp:lemmas}
Next, we establish Lemmas~\ref{lem:set-expansion-pibs-consistency}, \ref{lem:set-expansion-monotonic} and \ref{lem:set-expansion-large-enough}, which verify that the set-expansion policy in \Cref{alg:set-expansion} satisfies Conditions~\ref{def:focus-set:pibs-consistency}, \ref{def:focus-set:monotonic} and \ref{def:focus-set:large-enough}, respectively. Then we apply \Cref{thm:focus-set-policy} to prove \Cref{thm:set-expansion-policy}. 

\begin{lemma}[Set-expansion policy satisfies \Cref{def:focus-set:pibs-consistency}]\label{lem:set-expansion-pibs-consistency}
    Consider the set-expansion policy (\Cref{alg:set-expansion}). For any  $t\geq0$, there exists a subset $D_t' \subseteq D_t$ such that for any $i\in D_t'$, 
    $A_t(i) = \syshat{A}_t(i)$, and 
    \begin{equation}
        \Ebig{m(D_t \backslash D_t') \givenbig X_t, D_t} \leq \frac{1}{\sqrt{N}} + \frac{1}{N} \quad a.s.
    \end{equation}
\end{lemma}

\begin{lemma}[Set-expansion policy satisfies \Cref{def:focus-set:monotonic}]\label{lem:set-expansion-monotonic}
    Consider the set-expansion policy in \Cref{alg:set-expansion}. For any $t\geq0$, 
    \begin{equation}\label{eq:set-expansion-monotonic}
        \Ebig{(m(D_t) - m(D_{t+1}))^+ \givenbig X_t, D_t} \leq \frac{|\sspa|^{1/2} + 1}{\rhoBudget\sqrt{N}} + \frac{1 + (\rhoBudget + 1)|\sspa|}{\rhoBudget N} \quad a.s.
    \end{equation}
\end{lemma}

\begin{lemma}[Set-expansion policy satisfies \Cref{def:focus-set:large-enough}]\label{lem:set-expansion-large-enough}
    Consider the set-expansion policy in \Cref{alg:set-expansion}. For any $t\geq0$, 
    \begin{equation}
        1 - m(D_t) \leq \frac{|\sspa|^{1/2}}{\rhoBudget} \hw(X_t, D_t) + \frac{2}{\rhoBudget N} \quad a.s.
    \end{equation}
\end{lemma}

\begin{proof}{\textit{Proof of \Cref{thm:set-expansion-policy}.}}
    By Lemmas~\ref{lem:set-expansion-pibs-consistency}, \ref{lem:set-expansion-monotonic} and \ref{lem:set-expansion-large-enough}, the set-expansion policy satisfies Conditions~\ref{def:focus-set:pibs-consistency}, \ref{def:focus-set:monotonic} and \ref{def:focus-set:large-enough} with the subset Lyapunov functions $\{\hw(x, D)\}_{D\subseteq[N]}$. Applying \Cref{thm:focus-set-policy} and substituting the constants, we get
    \[
        \rrel - \rsysn \leq \frac{256\rmax \lamw^2 |\sspa|^2}{\rhoBudget^2\sqrt{N}}, 
    \]
    which implies the optimality gap bound in \Cref{thm:set-expansion-policy}. 
    Note that we relax all $1/N$ factors to $1/\sqrt{N}$ when deriving the bound. \Halmos
\end{proof}

\subsection{Proof of Lemma~\ref{lem:set-expansion-pibs-consistency}}\label{sec:pf-set-exp:pf-majority-conformity}

\begin{proof}{\textit{Proof of \Cref{lem:set-expansion-pibs-consistency}.}}
    Recall that under the set-expansion policy, the actions in the focus set $(A_t(i))_{i\in D_t}$ are chosen according to the ideal actions $(\syshat{A}_t(i))_{i\in D_t}$ based on the following rules: 
    \begin{itemize}
        \item When $\sum_{i\in D_t} \syshat{A}_t(i) > \alpha N$, the set-expansion policy chooses $\alpha N$ arms in $D_t$ with $\syshat{A}_t(i) = 1$ and sets $A_t(i) = \syshat{A}_t(i)$. 
        \item When $\sum_{i\in D_t} (1 - \syshat{A}_t(i)) > (1-\alpha) N$, the set-expansion policy chooses $(1-\alpha)N$ arms in $D_t$ with $\syshat{A}_t(i) = 0$ and sets $A_t(i) = \syshat{A}_t(i)$. 
        \item Otherwise, the set-expansion policy sets $A_t(i) = \syshat{A}_t(i)$ for all $i\in D_t$. 
    \end{itemize}
    Let $D_t' = \{i \in D_t \colon A_t(i) = \syshat{A}_t(i)\}$. Then we have
    \begin{equation}
        |D_t\setminus D_t'| = \Big(\sum_{i\in D_t} \syshat{A}_t(i) - \alpha N\Big)^+ + \Big(\sum_{i\in D_t} (1-\syshat{A}_t(i)) - (1-\alpha)N \Big)^+.
    \end{equation}
    Therefore, if we can show that for any $t\geq0$, 
    \begin{equation}\label{eq:setdiff}
        \EBig{\Big(\sum_{i\in D_t} \syshat{A}_t(i) - \alpha N\Big)^+ + \Big(\sum_{i\in D_t} (1-\syshat{A}_t(i)) - (1-\alpha)N \Big)^+ \givenBig X_t, D_t} \leq \sqrt{N} \quad a.s., 
    \end{equation}
    we will have $\E{m(D_t\setminus D_t') \givenplain X_t, D_t} \leq 1/\sqrt{N}$, which will complete the proof. 
    The remainder of the proof is dedicated to proving \eqref{eq:setdiff}.

    Consider the \emph{scaled expected budget requirement} for a given system state $x$ and subset $D\subseteq[N]$, defined as
    \begin{equation}
        C_\pibs(x, D) \triangleq \frac{1}{N}\EBig{\sum_{i\in D}\syshat{A}_t(i) \givenBig X_t=x} = \sum_{s\in\sspa} x(D, s) \pibs(1|s).
    \end{equation}
    Then $\Ebig{\sum_{i\in D_t} \syshat{A}_t(i) \givenbig X_t, D_t} = N C_\pibs(X_t, D_t)$. By the Cauchy-Schwarz inequality and the fact that given $X_t$ and $D_t$, $\syshat{A}_t(i)$'s are independent across $i\in D_t$, we have
    \begin{align}
        \EBig{\absBig{\sum_{i\in D_t} \syshat{A}_t(i) - N C_\pibs(X_t, D_t) } \givenBig X_t, D_t} 
        &\leq \EBig{\Big(\sum_{i\in D_t} \syshat{A}_t(i) - N C_\pibs(X_t, D_t)\Big)^2 \givenBig X_t, D_t}^{\frac{1}{2}} \nonumber\\
        &= \Big(\sum_{i\in D_t} \Varbig{\syshat{A}_t(i) \givenbig X_t, D_t}\Big)^{\frac{1}{2}} \nonumber\\
        &\leq \sqrt{N}. \label{eq:budget-dev}
    \end{align}

    We next prove \eqref{eq:setdiff} utilizing the bound \eqref{eq:budget-dev}.
    Recall that $D_t$ is chosen with $\slk(X_t, D_t) \geq 0$, i.e., $\norm{X_t(D_t) - m(D_t) \statdist}_1 / 2 \leq \rhoBudget(1-m(D_t))$.
    Then $\abs{C_\pibs(X_t, D_t) - \alpha m(D_t)}$ can be bounded as 
    \begin{align}
         \abs{C_\pibs(X_t, D_t) - \alpha m(D_t)}
         \label{eq:pf-set-exp:pibs-consist:c-deviation-1}
         &=  \absBig{\sum_{s\in\sspa} \big(X_t(D_t, s) - m(D_t) \statdist(s)\big) \pibs(1|s) } \\ 
         \label{eq:pf-set-exp:pibs-consist:c-deviation-1-5}
         &\leq  \absBig{\sum_{s\in\sspa}\Big( X_t(D_t, s) - m(D_t)\statdist(s)\Big) \Big(\pibs(1|s) - \frac{1}{2}\Big)}\\
         \label{eq:pf-set-exp:pibs-consist:c-deviation-1-6}
         &\leq \frac{1}{2} \norm{X_t(D_t) - m(D_t) \statdist}_1\\
         \label{eq:pf-set-exp:pibs-consist:c-deviation-2}
         &\leq \rhoBudget(1-m(D_t)),
    \end{align}
    where \eqref{eq:pf-set-exp:pibs-consist:c-deviation-1} is because $\sums \statdist(s) \pibs(1|s) = \alpha$,
    \eqref{eq:pf-set-exp:pibs-consist:c-deviation-1-5} is because $\sum_{s\in\sspa} \big(X_t(D_t, s) - m(D_t)\statdist(s)\big) = 0$, and \eqref{eq:pf-set-exp:pibs-consist:c-deviation-1-6} is because $\abs{\pibs(1|s) - 1/2} \leq 1/2$ for all $s\in\sspa$. 
    Thus, 
    \begin{align*}
        NC_\pibs(X_t, D_t) &\le \alpha |D_t| + \beta\big(N-|D_t|\big)\le \alpha |D_t| + \alpha\big(N-|D_t|\big) = \alpha N, \\
        NC_\pibs(X_t, D_t) &\ge \alpha |D_t| - \beta\big(N-|D_t|\big)\ge \alpha |D_t| - (1-\alpha)\big(N-|D_t|\big)=\alpha N - \big(N-|D_t|\big).
    \end{align*}
    Therefore, we have
    \begin{align}
        \nonumber
        &\EBig{\Big(\sum_{i\in D_t} \syshat{A}_t(i) - \alpha N\Big)^+ + \Big(\sum_{i\in D_t} (1-\syshat{A}_t(i)) - (1-\alpha)N \Big)^+ \givenBig X_t, D_t}\\
        \nonumber
        &\qquad= \EBig{\Big(\sum_{i\in D_t} \syshat{A}_t(i) - \alpha N\Big)^+ + \Big(\alpha N - \big(N - |D_t|\big) -  \sum_{i\in D_t} \syshat{A}_t(i) \Big)^+ \givenBig X_t, D_t}\\
        \label{eq:pf-set-exp:pibs-consist:c-deviation-3}
        &\qquad\leq  \EBig{\absBig{\sum_{i\in D_t} \syshat{A}_t(i) - N C_\pibs(X_t, D_t) } \,\givenBig\, X_t, D_t}  \\
        \nonumber
        &\qquad\leq \sqrt{N}, 
    \end{align}
    where the inequality \eqref{eq:pf-set-exp:pibs-consist:c-deviation-3} follows from the fact that $(x-a)^+ + (b-x)^+ \leq |x-c|$ for any real numbers $x,a,b,c$ such that $b\leq c \leq a$. 
    This proves \eqref{eq:setdiff}, concluding the proof of Lemma~\ref{lem:set-expansion-pibs-consistency}. \Halmos
\end{proof}

\subsection{Proof of Lemma~\ref{lem:set-expansion-monotonic}}\label{sec:pf-set-exp:pf-non-shrinking}

\begin{figure}
    \FIGURE{
    \subcaptionbox{When $D_t$ expands. \label{fig:set-exp-proof-expand}}{\includegraphics[height=4.9cm]{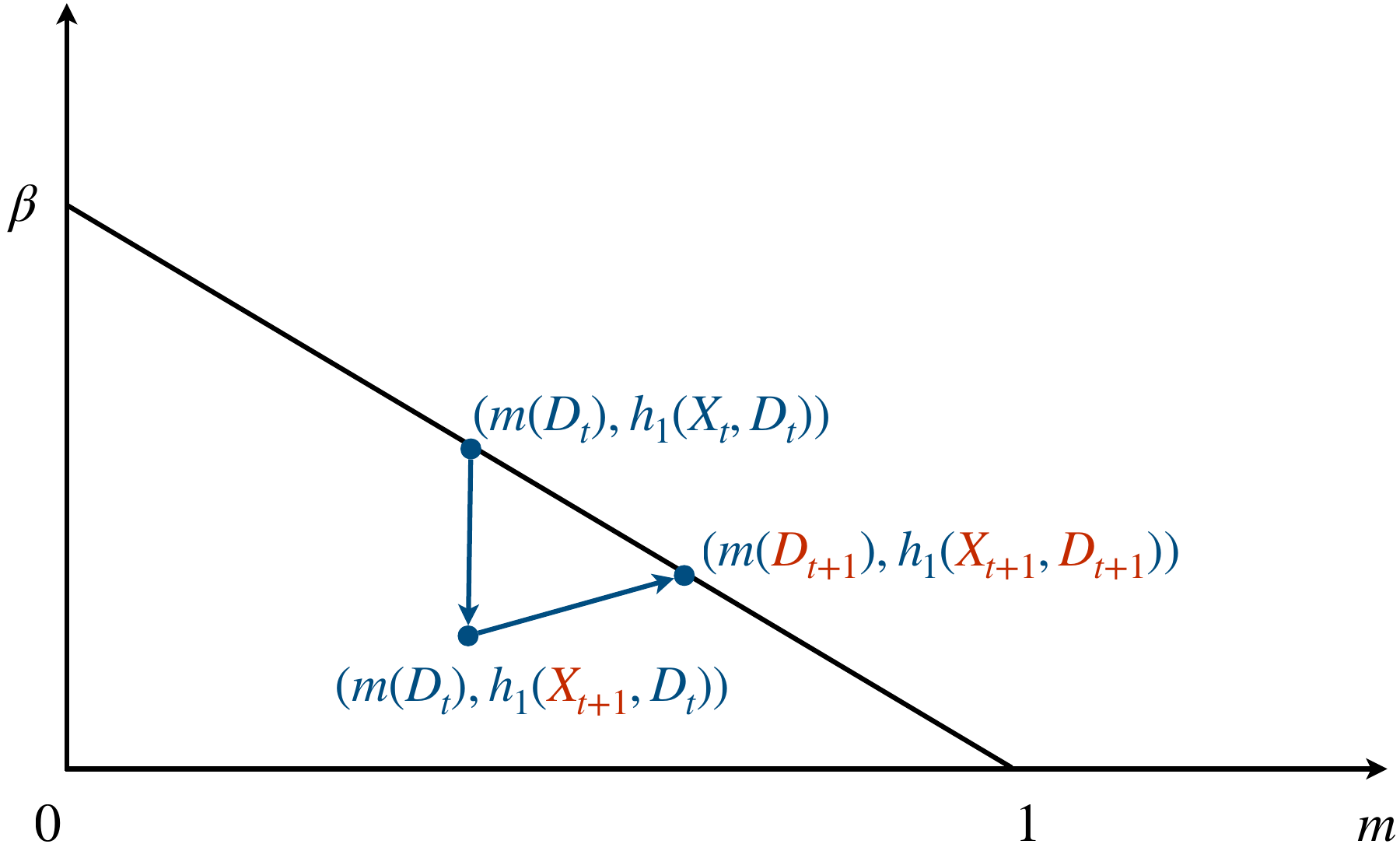}}
    \hfill
    \subcaptionbox{When $D_t$ shrinks. \label{fig:set-exp-proof-shrink}}{\includegraphics[height=4.9cm]{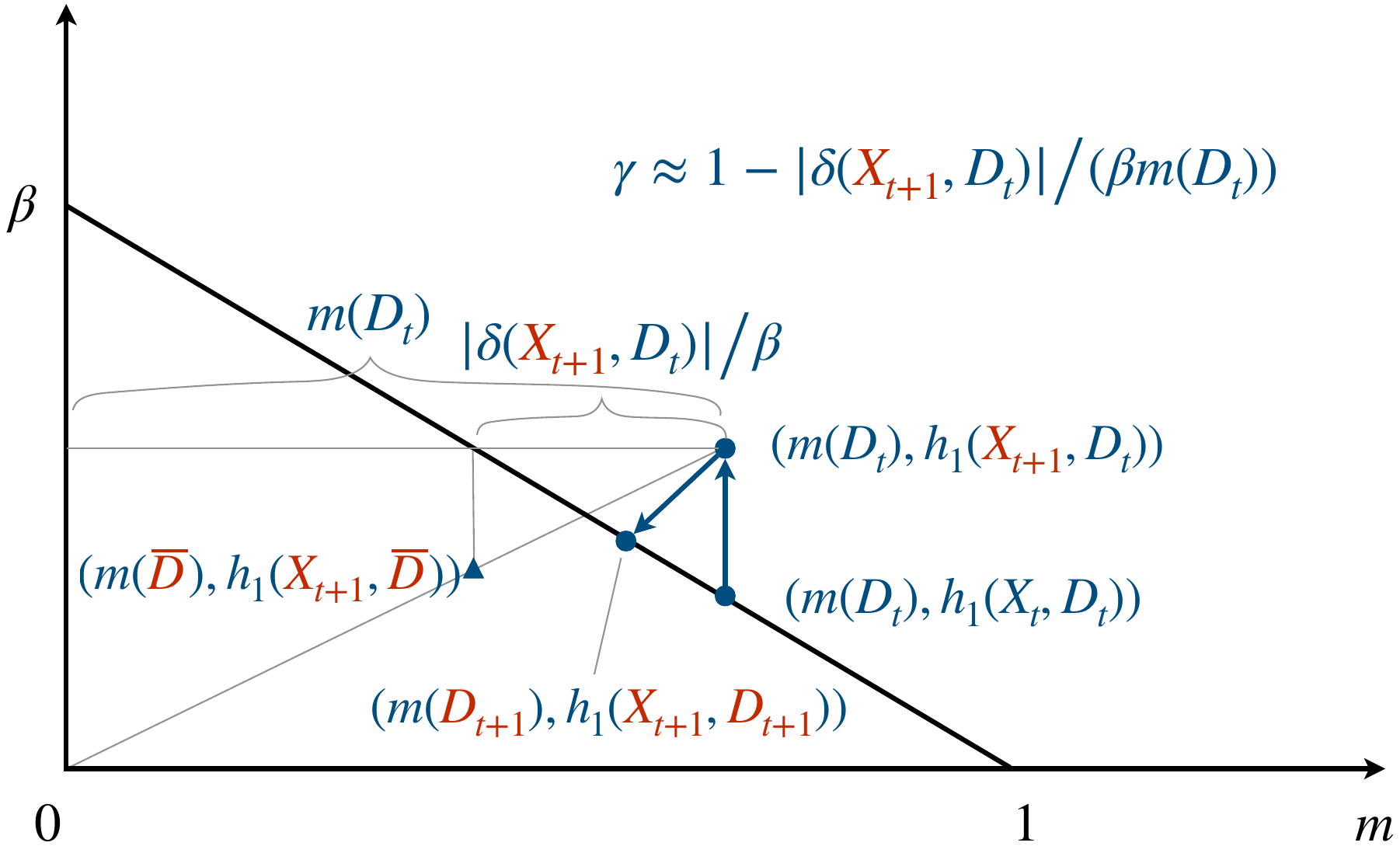}}
    } 
    {Intuition of why the focus set is almost non-shrinking under the set-expansion policy (proved in \Cref{lem:set-expansion-monotonic}).     \label{fig:non-shrink-proof}}
    {Each point denotes $(m(D_t), h_1(X_t, D_t))$ or $(m(D_t), h_1(X_{t+1}, D_t))$ under the set-expansion policy, where $h_1(x, D)$ is a shorthand for $0.5\normplain{x(D) - m(D)\statdist}_1$. 
    Recall that the set-expansion policy chooses the focus set $D_t$ as a maximal set whose corresponding point $(m(D_t), h_1(X_t, D_t))$ is below the line $m\mapsto \beta(1-m)$, ensuring that $\slk(X_t, D_t) \triangleq \beta(1-m(D_t))-h_1(X_t, D_t)\geq 0$. 
    The two subfigures illustrates two possible ways that the focus set can change based on the outcomes of the state transitions from $X_t$ to $X_{t+1}$. 
    In \Cref{fig:set-exp-proof-expand}, the point $(m(D_t), h_1(X_{t+1}, D_t))$ remains below the line, causing the focus set to expand ($D_{t+1}\supseteq D_t$). In \Cref{fig:set-exp-proof-shrink}, the point $(m(D_t), h_1(X_{t+1}, D_t))$ goes above the line, forcing the focus set to shrink ($D_{t+1}\subseteq D_t$). \\
    To show that the shrinkage is small in the latter case, by the maximality of $D_{t+1}$, we only need to find a sufficiently large subset $\overline{D}\subseteq D_t$ whose corresponding point on the figure, $(m(\overline{D}), h_1(X_t, \overline{D}))$, is below the line $m\mapsto \beta(1-m)$. 
    This subset $\overline{D}$ can be constructed by picking $\gamma$ fraction of arms in each state from $D_t$, with $\gamma \approx 1 - \abs{\slk(X_{t+1}, D_t)} / (\beta m(D_t))$; this construction of $\overline{D}$ corresponds to the triangular dot in \Cref{fig:set-exp-proof-shrink}. By bounding $\abs{\slk(X_{t+1}, D_t)}$, we can show that $\gamma$ is close to $1$, which implies that the cardinality of $D_{t+1}$ is not significantly smaller than that of $D_t$. 
    }
\end{figure}

In this subsection, we prove \Cref{lem:set-expansion-monotonic}, which shows that the focus set $D_t$ does not shrink significantly every time step. The intuition of the proof is illustrated in \Cref{fig:non-shrink-proof}.

\begin{proof}{\textit{Proof of \Cref{lem:set-expansion-monotonic}.}}
    Our proof consists of two steps. In Step 1, we focus on proving the following inequality for each time step $t$: 
    \begin{equation}\label{eq:proof-set-expansion-monotonic-intermediate-1}
        (m(D_t) - m(D_{t+1}))^+ \leq \frac{1}{\rhoBudget} (-\slk(X_{t+1}, D_t))^+ + \frac{K}{N},
    \end{equation}
    where $K = (1 + 1/\rhoBudget)|\sspa|$. 
    In Step 2, we utilize \eqref{eq:proof-set-expansion-monotonic-intermediate-1} to bound $\Ebig{(m(D_t) - m(D_{t+1}))^+ \givenbig X_t, D_t}$. 
    For the ease of reference, we first recall the definition of the slack $\slk(x, \dynset)$ for some system state $x$ and subset $D\subseteq [N]$, which is heavily used in this proof:
    \begin{equation}
        \slk(x, \dynset) = \rhoBudget (1-m(\dynset)) -  \frac{1}{2}\norm{x(D) - m(D) \statdist}_1.
    \end{equation}
    
    \paragraph{Step 1: Proving \eqref{eq:proof-set-expansion-monotonic-intermediate-1}.}
    When $\slk(X_{t+1}, D_t) > 0$, the focus set expands and we have $D_{t+1}\supseteq D_t$, so \eqref{eq:proof-set-expansion-monotonic-intermediate-1} trivially holds. 
    We therefore focus on the non-trivial case of $\slk(X_{t+1}, D_t) < 0$, where the focus set shrinks. 
    By definition, $D_{t+1}$ has the largest cardinality among all subsets $\overline{D} \subseteq D_t$ such that $\slk(X_{t+1}, \overline{D}) \geq 0$. We can thus look for a subset $\overline{D}\subseteq D_t$ with the following two properties:
    \begin{align}
        \label{eq:proof-set-expansion-monotonic-intermediate-11}
        \slk(X_{t+1}, \overline{D}) &\geq 0 \\
        \label{eq:proof-set-expansion-monotonic-intermediate-12}
        \big(m(D_t) - m(\overline{D})\big)^+ &\leq \frac{1}{\rhoBudget} \big(-\slk(X_{t+1}, D_t) \big)^+ + \frac{K}{N}.
    \end{align}
    The existence of such a subset $\overline{D}$ implies \eqref{eq:proof-set-expansion-monotonic-intermediate-1} because $(m(D_t) - m(D_{t+1}))^+ \leq (m(D_t) - m(\overline{D}))^+ \leq (-\slk(X_{t+1}, D_t))^+ / \rhoBudget + K/N$. 

    We construct $\overline{D}$ based on two cases: 
    \begin{itemize}
        \item If $\abs{\slk(X_{t+1}, D_t)} /\rhoBudget + K / N \geq m(D_t)$, we let $\overline{D} = \emptyset$;
        \item Otherwise, we let    
        \begin{equation}\label{eq:proof-set-expansion-monotonic-alpha-def}
         \gamma = 1 - \frac{1}{m(D_t)}\left(\frac{\abs{\slk(X_{t+1}, D_t)}}{\rhoBudget} + \frac{K - |\sspa|}{N}\right).
        \end{equation}
        One can verify that $0 < \gamma < 1$. We define $\overline{D}$ as a subset of $D_t$ such that 
        \begin{equation}\label{eq:proof-set-expansion-monotonic-Dbar-def}
            X_{t+1}(\overline{D}, s) = \frac{\lfloor \gamma X_{t+1}(D_t, s) N\rfloor}{N} \quad \forall s\in\sspa.
        \end{equation}
        Intuitively, $\overline{D}$ is obtained by sampling $\gamma$ fraction of arms in each state from $D_t$. A pictorial illustration of $\gamma$ is given in \Cref{fig:set-exp-proof-shrink}.
    \end{itemize}

    We now verify that this construction of $\overline{D}$ satisfies the properties \eqref{eq:proof-set-expansion-monotonic-intermediate-11} and \eqref{eq:proof-set-expansion-monotonic-intermediate-12}. When $\abs{\slk(X_{t+1}, D_t)} /\rhoBudget + K / N \geq m(D_t)$, we have $\overline{D}=\emptyset$, so these two conditions are trivially true.  
    We can thus focus on the case where $\abs{\slk(X_{t+1}, D_t)} /\rhoBudget + K / N < m(D_t)$. 
    
    We first show \eqref{eq:proof-set-expansion-monotonic-intermediate-11}, i.e., $\slk(X_{t+1}, \overline{D}) \geq 0$. By the definition of $\overline{D}$, we have $m(\overline{D}) \leq \gamma m(D_t)$. 
    Substituting the definitions of $\gamma$, $K$, and $\slk(X_{t+1}, D_t)$, we upper bound $m(\overline{D})$ as 
    \begin{equation*}
        m(\overline{D}) 
        \leq m(D_t) + \frac{1}{\rhoBudget} \slk(X_{t+1}, D_t) - \frac{|\sspa|}{\rhoBudget N} 
        = 1 - \frac{1}{2\rhoBudget} \norm{X_{t+1}(D_t) - m(D_t)\statdist}_1 - \frac{|\sspa|}{\rhoBudget N}. 
    \end{equation*}
    Then we can lower bound $\slk(X_{t+1}, \overline{D})$ using the above upper bound of $m(\overline{D})$:
    \begin{align}
        \slk(X_{t+1}, \overline{D}) 
        &= \rhoBudget(1-m(\overline{D})) - \frac{1}{2} \normplain{X_{t+1}(\overline{D}) - m(\overline{D})\statdist}_1 \nonumber \\
        \label{eq:pf-set-exp:mono:temp-1}
        &\geq \frac{1}{2}\normplain{X_{t+1}(D_t) - m(D_t)\statdist}_1 - \frac{1}{2}\normplain{X_{t+1}(\overline{D}) - m(\overline{D})\statdist}_1 + \frac{|\sspa|}{N}.
    \end{align}
    We further lower bound $\normplain{X_{t+1}(D_t) - m(D_t)\statdist}_1 - \normplain{X_{t+1}(\overline{D}) - m(\overline{D})\statdist}_1$ in  \eqref{eq:pf-set-exp:mono:temp-1} as
    \begin{align*}
         & \normplain{X_{t+1}(D_t) - m(D_t)\statdist}_1 - \normplain{X_{t+1}(\overline{D}) - m(\overline{D})\statdist}_1 \\
         &\qquad\geq \gamma \normplain{X_{t+1}(D_t) - m(D_t)\statdist}_1 - \normplain{X_{t+1}(\overline{D}) - m(\overline{D})\statdist}_1 \\
         &\qquad\geq - \normplain{\gamma X_{t+1}(D_t) - \gamma m(D_t)\statdist - X_{t+1}(\overline{D})  +  m(\overline{D})\statdist}_1 \\
         &\qquad\geq - \normplain{\gamma X_{t+1}(D_t) - X_{t+1}(\overline{D})}_1 - \absplain{\gamma m(D_t) - m(\overline{D})}\cdot \normplain{\statdist}_1 \\
         &\qquad\geq -\frac{2|\sspa|}{N},
    \end{align*}
    where the first inequality is due to $\gamma < 1$; the second and third inequalities are due to the triangular inequality; 
    the last inequality is due to the definition of $\overline{D}$ in \eqref{eq:proof-set-expansion-monotonic-Dbar-def}. Therefore, $\slk(X_{t+1}, \overline{D}) \geq 0$. 

    Next, we show \eqref{eq:proof-set-expansion-monotonic-intermediate-12}. 
    By the definition of $\overline{D}$, we have $m(\overline{D}) \geq \gamma m(D) - |\sspa| / N$. Substituting the definition of $\gamma$ in \eqref{eq:proof-set-expansion-monotonic-alpha-def}, we get 
    \begin{equation*}
        m(\overline{D}) 
        \geq m(D_t) + \frac{1}{\rhoBudget} \slk(X_{t+1}, D_t) - \frac{K - |\sspa|}{N} - \frac{|\sspa|}{N} 
        = m(D_t) + \frac{1}{\rhoBudget} \slk(X_{t+1}, D_t) - \frac{K}{N}, 
    \end{equation*}
    which implies \eqref{eq:proof-set-expansion-monotonic-intermediate-12}. 
    Therefore, we have proved the inequality \eqref{eq:proof-set-expansion-monotonic-intermediate-1} claimed at the beginning of this proof. 

    \paragraph{Step 2: Utilizing \eqref{eq:proof-set-expansion-monotonic-intermediate-1} to bound $\Ebig{(m(D_t) - m(D_{t+1}))^+ \givenbig X_t, D_t}$.}
    Taking the expectation on both sides of \eqref{eq:proof-set-expansion-monotonic-intermediate-1}, we get
    \begin{equation}\label{eq:proof-set-expansion-monotonic-intermediate-2}
        \Ebig{(m(D_t) - m(D_{t+1}))^+ \givenbig X_t, D_t} \leq \frac{1}{\rhoBudget} \Ebig{(-\slk(X_{t+1}, D_t))^+ \givenbig X_t, D_t} + \frac{K}{N}.
    \end{equation}
    Now, it remains to upper bound $\Ebig{(-\slk(X_{t+1}, D_t))^+ \givenbig X_t, D_t}$. 
    Let $X_{t+1}'$ be a random element denoting the system state at time $t+1$ if we were able to set $A_t(i) = \syshat{A}_t(i)$ for all $i\in [N]$. 
    Let $D_t'\triangleq \{i\in D_t\colon \syshat{A}_t(i) = A_t(i)\}$ (i.e., the subset given in \Cref{lem:set-expansion-pibs-consistency}), 
    we couple $X_{t+1}'$ and $X_{t+1}$ such that $X_{t+1}'(D_t') = X_{t+1}(D_t')$. We make the following decomposition with the help of $X_{t+1}'$: 
    \begin{align}
        &\Ebig{(-\slk(X_{t+1}, D_t))^+ \givenbig X_t, D_t}\nonumber\\
        \label{eq:proof-set-expansion-monotonic-intermediate-2-5}
        &\qquad\leq  \Ebig{\big(- \slk(X_{t+1}, D_t) + \slk(X_{t+1}', D_t)\big)^+\givenbig X_t, D_t} + \Ebig{\big(- \slk(X_{t+1}', D_t)\big)^+ \givenbig X_t, D_t}.
    \end{align}
    Below, we bound the two terms in \eqref{eq:proof-set-expansion-monotonic-intermediate-2-5} separately.
    
    To bound $\Ebig{\big(- \slk(X_{t+1}, D_t) + \slk(X_{t+1}', D_t)\big)^+\givenbig X_t, D_t}$, note that the coupling implies that
    \begin{align}
        \big(- \slk(X_{t+1}, D_t) + \slk(X_{t+1}', D_t)\big)^+ 
        &= \frac{1}{2} \big(\normbig{X_{t+1}(D_t) - m(D_t)\statdist}_1  - \normbig{X_{t+1}'(D_t) - m(D_t)\statdist}_1\big)^+ \nonumber\\
        &\leq \frac{1}{2}\norm{X_{t+1}(D_t) - X_{t+1}'(D_t)}_1 \nonumber\\
        &\leq \frac{1}{2}\norm{X_{t+1}(D_t\backslash D_t')}_1 + \frac{1}{2}\norm{X_{t+1}'(D_t\backslash D_t')}_1 \nonumber\\
        \label{eq:pf-set-exp:temp-3}
        &\leq m(D_t\backslash D_t'), 
    \end{align}
    where the last inequality uses the fact that $\normplain{X_{t+1}(D_t\backslash D_t')}_1 = \normplain{X_{t+1}'(D_t\backslash D_t')}_1  = m(D_t\backslash D_t')$. Taking the expectation and applying \Cref{lem:set-expansion-pibs-consistency}, we get
    \begin{equation}
        \Ebig{\big(- \slk(X_{t+1}, D_t) + \slk(X_{t+1}', D_t)\big)^+ \givenbig X_t, D_t}
        \le  \Eplain{m(D_t\backslash D_t')\givenbig X_t, D_t} 
        \label{eq:proof-set-expansion-monotonic-intermediate-4}
        \leq \frac{1}{\sqrt{N}} + \frac{1}{N}.
    \end{equation}

    Next, we bound $\Ebig{\big(- \slk(X_{t+1}', D_t)\big)^+ \givenbig X_t, D_t}$. By the definition of $\slk(X_{t+1}', D_t)$, we have
    \begin{align}
        \Ebig{(-\slk(X_{t+1}', D_t))^+  \givenbig  X_t, D_t} 
        &= \EBig{\Big(\frac{1}{2}\normplain{X_{t+1}'(D_t) - m(D_t)\statdist}_1 - \rhoBudget(1-m(D_t)) \Big)^+  \givenBig  X_t, D_t}\\
        &\leq \frac{1}{2}\Ebig{\big(\normplain{X_{t+1}'(D_t) - m(D_t)\statdist}_1 - \normplain{X_t(D_t) - m(D_t)\statdist}_1 \big)^+ \givenbig  X_t, D_t}, \label{eq:proof-set-expansion-monotonic-intermediate-3-1}
    \end{align}
    where \eqref{eq:proof-set-expansion-monotonic-intermediate-3-1} follows from the facts that $\slk(X_{t}, D_t) \ge 0$ and $\slk(X_{t}, D_t) \triangleq \rhoBudget(1-m(D_t)) - \normplain{X_t(D_t) - m(D_t)\statdist}_1$. 
    Intuitively, the expression on the right-hand side of \eqref{eq:proof-set-expansion-monotonic-intermediate-3-1} represents the expected one-step increase in the $L_1$ distance from the stationary distribution if all arms follow $\pibs$. Because the transition matrix $P_\pibs$, is non-expansive, this increase should be small. Indeed, as we formally prove in \Cref{lem:ell-1-drift} (Appendix~\ref{app:proof-ell-1}), this expression is bounded as: 
    \begin{equation}\label{eq:pibar-contraction-L1}
        \Ebig{\big(\normplain{X_{t+1}'(D_t) - m(D_t)\statdist}_1 - \normplain{X_t(D_t) - m(D_t)\statdist}_1 \big)^+ \givenbig X_t, D_t} \leq \frac{2|\sspa|^{1/2}}{\sqrt{N}}.
    \end{equation}
    Therefore,
    \begin{equation}\label{eq:proof-set-expansion-monotonic-intermediate-3}
        \Ebig{(-\slk(X_{t+1}', D_t))^+  \givenbig  X_t, D_t} \le \frac{|\sspa|^{1/2}}{\sqrt{N}}.
    \end{equation}

    Plugging \eqref{eq:proof-set-expansion-monotonic-intermediate-4} and \eqref{eq:proof-set-expansion-monotonic-intermediate-3} into \eqref{eq:proof-set-expansion-monotonic-intermediate-1}, we get
    \[
        \Ebig{(m(D_t) - m(D_{t+1}))^+  \givenbig X_t, D_t} \leq \frac{|\sspa|^{1/2} + 1}{\rhoBudget\sqrt{N}} + \frac{1 + (\rhoBudget + 1)|\sspa|}{\rhoBudget N}. \Halmos
    \]

\end{proof}

\subsection{Proof of Lemma~\ref{lem:set-expansion-large-enough}}\label{sec:pf-set-exp:pf-sufficient-cov}
\begin{proof}{\textit{Proof of \Cref{lem:set-expansion-large-enough}.}}
    Recall that $D_t$ is taken to be a maximal set such that $\slk(X_t, D_t) \geq 0$, where $\slk(x, \dynset) = \rhoBudget (1-m(\dynset)) - \norm{x(D) - m(D) \statdist}_1 / 2$. We first prove that 
    \[
        \slk(X_t, D_t) \leq \frac{2}{N}.
    \]
    Assume, for the sake of contradiction, that $\slk(X_t, D_t) > 2/N$. 
    Then by the definition of $\slk(X_t, D_t)$, we can see that $m(D_t) < 1$ and thus $D_t^c\neq \emptyset$.
    Picking an arbitrary $i\in D_t^c$, we have
    \begin{align*}
        \slk(X_t, D_t\cup\{i\}) - \slk(X_t, D_t)  
        &= -\frac{\rhoBudget}{N} - \frac{1}{2}\norm{X_t(D_t\cup\{i\}) - m(D_t\cup\{i\}) \statdist}_1 + \frac{1}{2}\norm{X_t(D_t) - m(D_t) \statdist}_1 \\
        &\geq -\frac{\rhoBudget}{N} - \frac{1}{2}\norm{X_t(\{i\}) - m(\{i\}) \statdist}_1 \\
        &\geq - \frac{2}{N}.
    \end{align*}
    Therefore, $\slk(X_t, D_t\cup\{i\}) > 0$, which contradicts the maximality of $D_t$. 

    Since $\slk(X_t, D_t) \leq 2/N$, we have
    \begin{equation*}
        1 - m(D_t) 
        \leq  \frac{1}{\rhoBudget}\norm{X_t(D_t) - m(D_t) \statdist}_1 + \frac{2}{\rhoBudget N}
        \leq \frac{|\sspa|^{1/2}}{\rhoBudget} \hw(X_t, D_t) + \frac{2}{\rhoBudget N}, 
    \end{equation*}
    where the second inequality is due to the distance dominance property of $\hw(x, D)$ in \eqref{eq:hw-feature-lyapunov:strength}. \Halmos
\end{proof}

\section{Proof of Theorem~\ref{thm:id-policy} (Optimality gap of ID policy)}
\label{sec:proof-id-policy}

In this section, we prove Theorem~\ref{thm:id-policy} using the framework established in Section~\ref{sec:formalizing}.
This section is organized as follows.
We first define the subset Lyapunov functions for the ID policy in \Cref{sec:proof-id-policy:feature-lyapunov}.
We then define the focus set for the ID policy in \Cref{sec:proof-id-policy:md}.
In \Cref{sec:proof-id-policy:lemmas-and-thm-pf}, we present three lemmas verifying that the ID policy satisfies Conditions~\ref{def:focus-set:pibs-consistency}, \ref{def:focus-set:monotonic} and \ref{def:focus-set:large-enough}, respectively, and prove Theorem~\ref{thm:id-policy} by combining these three lemmas and applying the meta-theorem \Cref{thm:focus-set-policy}. 
In Sections~\ref{sec:proof-id-policy:lemma-pfs}, \ref{sec:proof-id-policy:monotonic}, and \ref{sec:proof-id-policy:large-enough}, we prove the three lemmas.

\subsection{Subset Lyapunov functions}\label{sec:proof-id-policy:feature-lyapunov}

To construct the subset Lyapunov functions for the ID policy, consider the following class of functions, $\{\hid(x, [Nm])\}_{m\in [0,1]_N}$:  For any system state $x$ and $m\in[0,1]_N$, we let $\hid(x, [Nm])$ be a non-decreasing ``envelop'' of $\hw(x, [Nm])$, given by
\begin{equation}
    \label{eq:id-policy-h-def}
    \hid(x, [Nm]) =  \max_{\substack{m'\in [0,1]_N\\m'\leq m}} \hw(x, [Nm']).
\end{equation}
In the rest of the paper, we write $\hw(x, m)$ and $\hid(x, m)$ as shorthands for $\hw(x, [Nm])$ and $\hid(x, [Nm])$. 
Note that $\hid(x, m)$ depends only on the states of the arms in $[Nm]$, as required by the definition of subset Lyapunov functions. 
The lemma below verifies that $\{\hid(x, m)\}_{m\in [0,1]_N}$ satisfies \Cref{def:feature-lyapunov} and is proved in Appendix~\ref{app:proof-feature-lyapunov-lemmas}. 

\begin{restatable}{lemma}{hidfeature}\label{lem:hid-feature-lyaupnov}
    The class of functions $\{\hid(\cdot, m)\}_{m\in[0,1]_N}$ defined in \eqref{eq:id-policy-h-def} satisfies that for any system state $x$ and any $m, m' \in [0,1]_N$, 
    \begin{align}
        \label{eq:hid-feature-lyapunov:drift}
        &\EBig{\Big(\hid(X_1, m) - \big(1-\frac{1}{2\lamw}\big) \hid(x, m)\Big)^+ \givenBig X_0 = x, A_0(i)\sim \pibs(\cdot| S_0(i)) \, \forall i\in [Nm] } \leq \frac{4\lamw^{1/2}}{\sqrt{N}}, \\
        \label{eq:hid-feature-lyapunov:strength}
        &\hid(x, m) \geq \frac{1}{|\sspa|^{1/2}} \norm{x([Nm]) - m\statdist}_1, \\
        \label{eq:hid-feature-lyapunov:lipschitz}
        &\abs{\hid(x, m) - \hid(x, m')} \leq 2\lamw^{1/2} \abs{m'-m}.
    \end{align}
    These inequalities imply the drift condition, distance dominance property, and Lipschitz continuity in \Cref{def:feature-lyapunov}, respectively.
    Consequently, $\{\hid(x, m)\}_{m\in[0,1]_N}$ are subset Lyapunov functions for the optimal single-armed policy $\pibs$. 
\end{restatable}

We note that the inequality \eqref{eq:hid-feature-lyapunov:drift} is stronger than the drift condition required by the definition of subset Lyapunov functions.  This stronger version is needed for later analysis.

\subsection{Focus set}\label{sec:proof-id-policy:md}
The ID policy, as previously noted, does not explicitly specify focus sets within its algorithm.
Nonetheless, for analysis purposes, we can introduce a set $D_t$ at each time step $t$, effectively serving as the focus set for the ID policy.
Specifically, let $D_t = [N \md(X_t)]$, where $\md(\cdot)$ is a function that maps a system state to a number in $[0,1]_N = \{0,1/n,\dots, 1\}$, formally defined as follows:
\begin{equation}\label{eq:md-def}
    \md(x) = \max \{m\in [0,1]_N \colon \ratiocw \hid(x, m) \leq \rhoBudget(1-m)\},
\end{equation}
where $\rhoBudget \triangleq \min\{\alpha, 1-\alpha\}$ and $\ratiocw$ is a constant.
More concretely, the constant $\ratiocw = \norm{\costvec}_{\wmat^{-1}} = \sqrt{\costvec \wmat^{-1} \costvec^\top}$, where $\costvec$ denotes the row vector $(\pibs(1|s))_{s\in\sspa}$ and $\wmat$ is the matrix given in \Cref{def:w-and-w-norm}. 

The definition of $\md(x)$ has a nice geometric representation, as shown in \Cref{fig:md-def}.
For a system state $x$, note that $\hid(x, [0])=0$ and recall that $\hid(x, [Nm])$ is non-decreasing in $m$.
Then $\md(x)$ is the value of $m$ at which the curve $m \mapsto \ratiocw\hid(x, [Nm])$ intersects with the line $m \mapsto \rhoBudget (1-m)$, ignoring the integer effect.

\begin{figure}
    \FIGURE{
    \subcaptionbox{Subset Lyapunov functions and the focus set. \label{fig:md-def}}{\includegraphics[height=4.9cm]{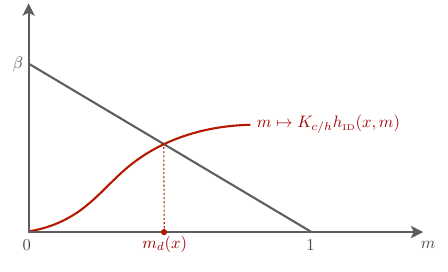}}
    \hfill
    \subcaptionbox{Illustration of the proof of \Cref{lem:id-pibs-consistency}. \label{fig:id-pibs-consistency}}{\includegraphics[height=4.9cm]{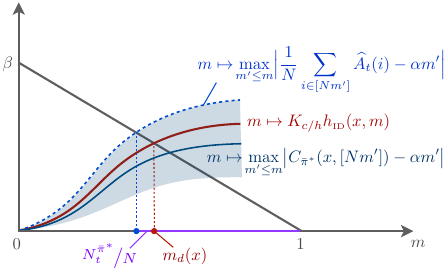}}
    }
    {Illustrations of the focus set and the proof of \Cref{lem:id-pibs-consistency} for the ID policy.}
    {\textbf{(a)} Suppose the current system state is $X_t=x$.  The function $\hid(x,m)$, a shorthand for $\hid(x,[Nm])$, is a subset Lyapunov function on the subset $[Nm]$.  The set $[N \md(x)]$ is the focus set. 
    \textbf{(b)} The three curves illustrated are central to the proof of \Cref{lem:id-pibs-consistency}, e.g., see the inequality~\eqref{eq:curve-decomp}. Take the bottom curve $m\mapsto\max_{m'\le m}\absplain{C_\pibs(x, [Nm']) - \alpha m'}$ as the baseline. We show that the red curve based on the subset Lyapunov function $m \mapsto \ratiocw\hid(x, [Nm])$ is always above the bottom curve, and that the curve $m\mapsto \max_{m'\le m} \absplain{\frac{1}{N}\sum_{i\in[Nm']} \syshat{A}_t(i) - \alpha m'}$ deviates from the bottom curve by $O(1/\sqrt{N})$ in expectation.
    Since $\Ngood/N$ is always to the right of the solid blue dot where $m\mapsto \beta(1-m)$ and $m\mapsto \max_{m'\le m} \absplain{\frac{1}{N}\sum_{i\in[Nm']} \syshat{A}_t(i) - \alpha m'}$ intersect, we have $(N\md(X_t) - \Ngood)^+= O(1/\sqrt{N})$ in expectation.}
\end{figure}

\begin{remark}
    Here we comment on the different choices of subset Lyapunov functions in the analysis of the set-expansion policy and the ID policy. 
    In the analysis of the set-expansion policy, $\{\hw(x, D)\}_{D\subseteq[N]}$ is constructed to satisfy \Cref{def:feature-lyapunov}, allowing the application of the meta-theorem \Cref{thm:focus-set-policy}. 
    In the analysis of the ID policy, $\{\hid(x, m)\}_{m\in[0,1]_N}$ is constructed to satisfy \Cref{def:feature-lyapunov} and to be non-decreasing in $m$. 
    This monotonicity of $\hid(x,m)$ ensures that the focus set $D_t=[N \md(X_t)]$, defined via $\hid(x,m)$, can be proved to satisfy the almost non-shrinking condition (\Cref{lem:id-monotonic}). 
\end{remark}

\subsection{Lemmas for verifying Conditions~\ref{def:focus-set:pibs-consistency}, \ref{def:focus-set:monotonic} and \ref{def:focus-set:large-enough}; Proof of Theorem~\ref{thm:id-policy}}\label{sec:proof-id-policy:lemmas-and-thm-pf}

Having defined the subset Lyapunov functions $\{\hid(x, [Nm])\}_{m\in [0,1]_N}$ and the focus set $D_t = [N \md(X_t)]$, we proceed to establish Lemmas~\ref{lem:id-pibs-consistency}, \ref{lem:id-monotonic} and \ref{lem:id-large-enough}, which verify that the ID policy satisfies Conditions~\ref{def:focus-set:pibs-consistency}, \ref{def:focus-set:monotonic} and \ref{def:focus-set:large-enough}, respectively. Then we apply \Cref{thm:focus-set-policy} to prove \Cref{thm:id-policy}.

\begin{lemma}[ID policy satisfies Condition~\ref{def:focus-set:pibs-consistency}]\label{lem:id-pibs-consistency}
    Consider the ID policy in \Cref{alg:id}. 
    For any $t\ge 0$, let $D_t' = [\min(\Ngood, N\md(X_t))]$, where recall that $\Ngood\in[N]$ is defined in \Cref{alg:id} as the largest number such that for any  $i\in[\Ngood]$, $A_t(i) = \syshat{A}_t(i)$.
    Then
    \begin{equation}
        \Ebig{m(D_t \backslash D_t') \givenbig X_t} = \frac{1}{N} \Ebig{(N\md(X_t) - \Ngood)^+ \givenbig X_t} \leq  \frac{2}{\rhoBudget\sqrt{N}} + \frac{1}{N} \quad a.s.
    \end{equation} 
\end{lemma}

\begin{restatable}[ID policy satisfies \Cref{def:focus-set:monotonic}]{lemma}{idmonotonic}
\label{lem:id-monotonic}
    Consider the ID policy in \Cref{alg:id}. For any  $t\geq0$,
    \vspace*{-\baselineskip}
    \begin{align}
        \Ebig{\big(m(D_t) - m(D_{t+1})\big)^+ \big| X_t} &= \Ebig{(\md(X_t) - \md(X_{t+1}))^+\big|X_t}\nonumber\\
        &\leq \frac{4\ratiocw\lamw^{1/2}(1+\rhoBudget)}{\rhoBudget^2 \sqrt{N}} + \frac{2\ratiocw\lamw^{1/2}+\rhoBudget}{\rhoBudget N} \quad a.s.
    \end{align}
\end{restatable}
\vspace*{-\baselineskip}

\begin{restatable}[ID policy satisfies \Cref{def:focus-set:large-enough}]{lemma}{idcoverage}\label{lem:id-large-enough}
    Consider the ID policy in \Cref{alg:id}. For any $t\geq0$, 
    \begin{equation}\label{eq:id-large-enough}
        1 - m(D_t) \leq \frac{\ratiocw}{\rhoBudget} \hid(X_t, D_t) + \frac{2\ratiocw\lamw^{1/2}+\rhoBudget}{\rhoBudget N} \quad a.s.
    \end{equation}
\end{restatable}

\begin{proof}{\textit{Proof of \Cref{thm:id-policy}}.}
    By Lemma~\ref{lem:id-pibs-consistency}, \ref{lem:id-monotonic} and \ref{lem:id-large-enough}, the ID policy satisfies Conditions~\ref{def:focus-set:pibs-consistency}, \ref{def:focus-set:monotonic} and \ref{def:focus-set:large-enough} with the subset Lyapunov functions $\{\hid(x, [Nm])\}_{m\in [0,1]_N}$. Applying \Cref{thm:focus-set-policy} and substituting the constants, we get 
    \[
        \rrel - \rsysn \leq \frac{684\rmax \lamw^{5/2} |\sspa|^{3/2}}{\rhoBudget^3 \sqrt{N}},
    \]
    which implies the optimality gap bound in the theorem statement. 
    Note that we bound $\ratiocw$ by $|\sspa|^{1/2}$ and relax all $1/N$ factors to $1/\sqrt{N}$ when deriving this bound. 
    \Halmos
\end{proof}

\subsection{Proof of \Cref{lem:id-pibs-consistency}}\label{sec:proof-id-policy:lemma-pfs}

Before delving into the proof, we first offer a high-level understanding of \Cref{lem:id-pibs-consistency}.
Recall that $[\Ngood]$ is defined to be the largest set of arms that follow their ideal actions under the ID policy. 
Then \Cref{lem:id-pibs-consistency} states that the focus set we define, $D_t=[N\md(X_t)]$, is close to $[\Ngood]$, differing by only $O(\sqrt{N})$ elements. 
Note that whether a set of arms $[Nm]$ can follow their ideal actions or not is determined by the amount of budget required by them, i.e., the number of action $1$'s in their ideal actions. 
Our proof of \Cref{lem:id-pibs-consistency} utilizes the relationship between the budget requirement by arms in $[Nm]$ and the distributional distance $\norm{x([Nm]) - m\statdist}_W$.

\begin{proof}{\textit{Proof of \Cref{lem:id-pibs-consistency}}.}
    In this proof, the variables $n$ and $n'$ are by default non-negative integers. 
    We fix a time step $t\ge 0$, and condition on $X_t = x$ for a fixed system state realization $x$. 
    We first prove a lower bound for $\Ngood$ in terms of the \emph{budget requirements} of the arms in $[n]$, $\sum_{i\in[n]} \syshat{A}_t(i)$, for each $n \leq N$. 
    Observe that the arms in $[n]$ can follow their ideal actions if and only if 
    \begin{equation}
        \label{eq:action-0-1-conformity}
        \sum_{i\in[n]} \syshat{A}_t(i) \le \alpha N \quad \text{ and } \quad
        \sum_{i\in[n]} (1-\syshat{A}_t(i)) \le (1-\alpha) N.
    \end{equation}
    Here the first condition in \eqref{eq:action-0-1-conformity} requires that the number of action $1$'s is within budget.  For the second condition in \eqref{eq:action-0-1-conformity}, the easiest way to understand it is that it requires the number of action $0$'s to not exceed $(1-\alpha)N$, where $(1-\alpha)N$ can be interpreted as the ``budget for passive actions''. 
    As a result, a sufficient condition for the arms in $[n]$ to follow their ideal actions is
    \begin{equation}
    \absBig{\sum_{i\in[n]} \syshat{A}_t(i) - \alpha n} \leq \rhoBudget\big(N - n\big),\label{eq:suff-1}
    \end{equation}
    where recall that $\rhoBudget=\min\{\alpha, 1-\alpha\}$.
    In this proof, we use a further sufficient condition for the inequality \eqref{eq:suff-1} above, which is
    \begin{equation}
        \max_{n'\le n}\absBig{\sum_{i\in[n']} \syshat{A}_t(i) - \alpha n'} \leq \rhoBudget\big(N - n\big).\label{eq:suff-2}
    \end{equation}
    Therefore, by the definition of $\Ngood$, we have
    \begin{align}
        \Ngood &\geq \max \bigg\{n \le N \colon \max_{n'\le n}\absBig{\sum_{i\in[n']} \syshat{A}_t(i) - \alpha n'} \leq \rhoBudget\big(N - n\big) \bigg\}\nonumber\\
        &= N \max \bigg\{m\in[0,1]_N \colon \max_{\substack{m'\in[0,1]_N\\m'\le m}} \absBig{\frac{1}{N}\sum_{i\in[Nm']} \syshat{A}_t(i) - \alpha m'} \leq \rhoBudget\big(1 - m\big) \bigg\}.\label{eq:prove-ngood-bound-intermediate-0}
    \end{align}

    Our next step is to bound the quantity $\max_{m'\in[0,1]_N,m'\le m} \absplain{\frac{1}{N}\sum_{i\in[Nm']} \syshat{A}_t(i) - \alpha m'}$. 
    To do this, we relate it to the subset Lyapunov function $\hid(x,m)$ and the scaled expected budget requirement, $C_\pibs(x, D)$, which is defined as: 
    \begin{equation}
        C_\pibs(x, D) \triangleq \frac{1}{N}\EBig{\sum_{i\in D}\syshat{A}_t(i)\givenBig X_t=x} = \sum_{s\in\sspa} x(D, s) \pibs(1|s) = x(D) \costvec^\top,
    \end{equation}
    where $\costvec$ is the row vector $(\pibs(1|s))_{s\in\sspa}$.
    We first decompose the target quantity using the triangle inequality: for any $m\in[0,1]_N$, we have
    \begin{align}
        &\max_{\substack{m'\in[0,1]_N\\m'\le m}} \absBig{\frac{1}{N}\sum_{i\in[Nm']} \syshat{A}_t(i) - \alpha m'}\nonumber\\
        &\qquad\leq \max_{\substack{m'\in[0,1]_N\\m'\le m}} \biggl(\absBig{C_\pibs(x, [Nm']) - \alpha m'} + \absBig{\frac{1}{N}\sum_{i\in[Nm']} \syshat{A}_t(i) - C_\pibs(x, [Nm'])}\biggr)\nonumber\\
        &\qquad\leq \max_{\substack{m'\in[0,1]_N\\m'\le m}}\absBig{C_\pibs(x, [Nm']) - \alpha m'} + \max_{m'\in[0,1]_N} \absBig{\frac{1}{N}\sum_{i\in[Nm']} \syshat{A}_t(i) - C_\pibs(x, [Nm'])},\label{eq:curve-decomp}
    \end{align}
    The second term on the right-hand side is a noise term that we will bound later. We now focus on bounding the first term, which captures the deviation of the instantaneous expected budget requirement from its steady-state expectation. Using the fact that $\mu^* \costvec^\top = \sum_{s\in\sspa} \mu^*(s) \pibs(1|s) = \alpha$, we have:
    \begin{align}
        \abs{C_\pibs(x, \setbelow{\dynset}{m'}) - \alpha m' }
        &= \abs{(x(\setbelow{\dynset}{m'}) - m' \statdist) \costvec^\top}  \nonumber \\
        &= \abs{(x(\setbelow{\dynset}{m'}) - m'\statdist) \wmat^{1/2} {\wmat}^{-1/2} \costvec^\top } \nonumber \\
        &\leq \norm{x(\setbelow{\dynset}{m'}) - m'\statdist}_\wmat \norm{\costvec}_{\wmat^{-1}}\nonumber\\
        &=\ratiocw\hw(x,m') \\
        \label{eq:c-deviate-bdd-by-hid}
        \max_{\substack{m'\in[0,1]_N\\m'\le m}}\absBig{C_\pibs(x, [Nm']) - \alpha m'}
        &\le 
        \ratiocw\hid(x,m).
    \end{align}
    Using the monotonicity of $\hid(x,m)$ in $m$ and the definition of $\md(x)$, we have that for any $m \le \md(x)$, 
    \begin{equation}
        \ratiocw\hid(x,m) \leq \ratiocw\hid(x,\md(x)) \leq \rhoBudget(1-\md(x)).
    \end{equation}
    Substituting these bounds back into the decomposition in \eqref{eq:curve-decomp} implies that, for any $m\in\gridset$ such that $m\leq \md(x)$, we have the following bound: 
    \begin{equation}
        \label{eq:budget-diff-upper}
        \max_{\substack{m'\in[0,1]_N\\m'\le m}} \absBig{\frac{1}{N}\sum_{i\in[Nm']} \syshat{A}_t(i) - \alpha m'}
        \le \rhoBudget(1-\md(x))+\max_{m'\in[0,1]_N} \absBig{\frac{1}{N}\sum_{i\in[Nm']} \syshat{A}_t(i) - C_\pibs(x, [Nm'])}. 
    \end{equation}

    Next, we derive a lower bound for $\Ngood$ using \eqref{eq:prove-ngood-bound-intermediate-0} and \eqref{eq:budget-diff-upper}. Consider any $\overline{m}\in\gridset$ satisfying two conditions: (1) $\overline{m} \leq \md(x)$, and (2) the right-hand side of \eqref{eq:budget-diff-upper} is no greater than $\rhoBudget(1-\overline{m})$. For any such $\overline{m}$, \eqref{eq:budget-diff-upper} implies
    \[
        \overline{m} \in \bigg\{m\in[0,1]_N \colon \max_{\substack{m'\in[0,1]_N\\m'\le m}} \absBig{\frac{1}{N}\sum_{i\in[Nm']} \syshat{A}_t(i) - \alpha m'} \leq \rhoBudget\big(1 - m\big) \bigg\},
    \]
    and then \eqref{eq:prove-ngood-bound-intermediate-0} implies $\Ngood \geq N \, \overline{m}$. Taking the largest such $\overline{m}$ yields
    \begin{align}
        \Ngood &\geq \min\bigg\{N\md(x), \floorBig{N - \frac{N}{\rhoBudget}\Big(\rhoBudget(1-\md(x)) +\max_{m'\in[0,1]_N} \absBig{\frac{1}{N}\sum_{i\in[Nm']} \syshat{A}_t(i) - C_\pibs(x, [Nm'])}\Big)}\bigg\} \nonumber \\
        &\geq  \min\bigg\{N\md(x), N - \frac{N}{\rhoBudget}\Big(\rhoBudget(1-\md(x)) +\max_{m'\in[0,1]_N} \absBig{\frac{1}{N}\sum_{i\in[Nm']} \syshat{A}_t(i) - C_\pibs(x, [Nm'])}\Big)-1\bigg\}  \nonumber \\ 
        &= \min\bigg\{N\md(x), N\md(x) - 1 - \frac{1}{\rhoBudget} \max_{m'\in[0,1]_N} \absBig{\sum_{i\in[Nm']} \syshat{A}_t(i) - N C_\pibs(x, [Nm'])}\bigg\} \nonumber \\ 
        \label{eq:Ngood-lower-bound}
        &= N\md(x) -1 - \frac{1}{\rhoBudget} \max_{n'\leq N} \absBig{\sum_{i\in[n']} \syshat{A}_t(i) - N C_\pibs(x, [n'])}. 
    \end{align}
    
    Rearranging the terms in \eqref{eq:Ngood-lower-bound} and taking the expectation, we get
    \begin{equation}\label{eq:prove-ngood-bound-intermediate-3}
        \EBig{\big(N\md(x) - \Ngood\big)^+ \givenBig X_t = x} \leq 1 + \frac{1}{\rhoBudget} \EBig{\max_{n'\leq N} \absBig{\sum_{i\in[n']} \syshat{A}_t(i) - N C_\pibs(x, [n'])} \givenBig X_t=x}. 
    \end{equation}
    
    With \eqref{eq:prove-ngood-bound-intermediate-3}, it remains to prove
    \begin{equation}\label{eq:budget-req-maximal-deviation}
        \EBig{\max_{n \leq N} \absBig{\sum_{i\in[n]} \syshat{A}_t(i) - N C_\pibs(x, [n])} \givenBig X_t = x} \leq 2\sqrt{N}.
    \end{equation}
    Our proof uses Doob's maximum inequality. First, we simplify the expression by defining the centered random variables $\zmNoise(i) = \syshat{A}_t(i) - \Ebig{\syshat{A}_t(i) \givenbig X_t = x}$. This allows us to rewrite the left-hand side of \eqref{eq:budget-req-maximal-deviation} as:
    \begin{equation}\label{eq:budget-req-deviation:discretize}
        \EBig{\max_{n \leq N} \absBig{\sum_{i\in[n]} \syshat{A}_t(i) - N C_\pibs(x, [n])} \givenBig X_t = x} = \EBig{\max_{n \leq N} \absBig{\sum_{i\in[n]} \zmNoise(i)} \givenBig X_t = x}. 
    \end{equation}
    The sequence of partial sums $(\sum_{i\in[n]} \zmNoise(i))_{n\in[N]}$ forms a martingale when conditioned on $X_t=x$. This holds because the ideal actions $\syshat{A}_t(i)$ are sampled independently, making the $\zmNoise(i)$ terms independent random variables, each with a conditional expectation of zero.
    Applying Cauchy-Schwarz and Doob's $L_2$ maximum inequality \citep[see, e.g.,][Theorem 4.4.4]{Dur_19_prob_book}, we bound this term as:
    \begin{align*}
        \EBig{\max_{n \leq N} \absBig{\sum_{i\in[n]} \zmNoise(i)} \givenBig X_t = x}
        &\leq \EBig{\max_{n \leq N} \Big|\sum_{i\in[n]} \zmNoise(i)\Big|^2 \givenBig X_t = x}^{1/2}\\
        &\le \biggl(4\EBig{\Big|\sum_{i\in[N]} \zmNoise(i)\Big|^2  \givenBig X_t = x}\biggr)^{1/2}\\
        &=\biggl(4\sum_{i\in[N]} \EBig{\zmNoise(i)^2 \givenBig X_t = x}\biggr)^{1/2}\\
        &\le 2\sqrt{N}.
    \end{align*}
    Here, the first inequality follows from Cauchy-Schwarz, the second applies Doob's maximal $L_2$ inequality, and the equality follows from the independence of the $\zmNoise(i)$ terms. The final inequality holds because $\absbig{\zmNoise(i)} = \absbig{ \syshat{A}_t(i) - \Ebig{\syshat{A}_t(i) \givenbig X_t = x}} \leq 1$. This completes the proof. 
    \Halmos
\end{proof}

\subsection{Proof of \Cref{lem:id-monotonic}}\label{sec:proof-id-policy:monotonic}

Next, we prove that the focus set that we choose for the ID policy, $D_t = [N\md(X_t)]$, satisfies the almost non-shrinking condition (\Cref{def:focus-set:monotonic}). 
The key intuition lies in the geometric interpretation shown in \Cref{fig:md-def}, where $\md(X_t)$ is the intersection point of two functions: the non-decreasing curve $m\mapsto \hid(X_t, m)$ and the strictly decreasing line $m \mapsto \rhoBudget (1-m)$. 
Our proof shows that the vertical shift of the curve $m\mapsto \hid(X_t, m)$ from one time step to the next is small (in particular, $\hid(X_{t+1}, \md(X_t)) - \hid(X_t, \md(X_t))$ is small). 
Because the non-decreasing curve's vertical movement is small as it intersects the strictly decreasing line, the intersection point $\md(X_t)$ cannot shift significantly to the left. This implies that the focus set $D_t = [N\md(X_t)]$ is almost non-shrinking.

\begin{proof}{\textit{Proof of \Cref{lem:id-monotonic}}.}
    Fixing a time step $t\ge 0$, 
    we first prove the following inequality, which will be used to establish an upper bound on $\Ebig{(\md(X_t) - \md(X_{t+1}))^+ \givenbig X_t}$:
    \begin{equation}\label{eq:proof-id-monotonic-intermediate-1}
        \md(X_{t+1}) \geq  \md(X_t) - \frac{\ratiocw}{\rhoBudget} \big(\hid(X_{t+1}, \md(X_t)) - \hid(X_t, \md(X_t))\big)^+ - \frac{1}{N}.
    \end{equation}
    By the maximality of $\md(X_{t+1})$, it suffices to show $\ratiocw \hid(X_{t+1}, \overline{m}) \leq \rhoBudget(1-\overline{m})$ for each $\overline{m} \in [0,1]_N$ with $\overline{m} \leq \md(X_t) - \frac{\ratiocw}{\rhoBudget} \big(\hid(X_{t+1}, \md(X_t)) - \hid(X_t, \md(X_t))\big)^+$. For any such $\overline{m}$, 
    \begin{align*}
        \rhoBudget(1-\overline{m}) 
        &\geq \rhoBudget(1-\md(X_t)) + \ratiocw \big(\hid(X_{t+1}, \md(X_t)) - \hid(X_t, \md(X_t))\big)^+ \\
        &\geq \ratiocw \hid(X_t, \md(X_t)) + \ratiocw \big(\hid(X_{t+1}, \md(X_t)) - \hid(X_t, \md(X_t))\big)^+ \\
        &\geq \ratiocw\hid(X_{t+1}, \md(X_t)) \\
        &\geq \ratiocw\hid(X_{t+1}, \overline{m}),
    \end{align*} 
    where the second inequality is because $\ratiocw \hid(X_t, \md(X_t)) \leq \rhoBudget(1-\md(X_t))$, and the last inequality is because $\hid(X_t, m)$ is \emph{non-decreasing in $m$} and $\overline{m} \leq \md(X_t)$. This proves \eqref{eq:proof-id-monotonic-intermediate-1}. 

    The inequality \eqref{eq:proof-id-monotonic-intermediate-1} implies that 
    \begin{equation}
    \label{eq:proof-id-monotonic-intermediate-2}
        \Ebig{(\md(X_t) - \md(X_{t+1}))^+ \givenbig X_t } 
        \leq  \frac{\ratiocw}{\rhoBudget}  \Ebig{\big(\hid(X_{t+1}, \md(X_t)) - \hid(X_t, \md(X_t))\big)^+ \givenbig X_t} + \frac{1}{N}.
    \end{equation}
    We now upper bound $\Ebig{\big(\hid(X_{t+1}, \md(X_t)) - \hid(x, \md(x))\big)^+ \givenbig X_t}$ by coupling $X_{t+1}$ with a random element $X_{t+1}'$ constructed below.
    Let $X_{t+1}'$ be the random element denoting the system state at time step $t+1$ if we were able to set $A_t(i) = \syshat{A}_t(i)$ for all  $i\in [N]$. By the drift condition of the subset Lyapunov function $\hid(\cdot,D)$ established as \eqref{eq:hid-feature-lyapunov:drift} in \Cref{lem:hid-feature-lyaupnov}, 
    \begin{align}
        & \EBig{\big(\hid(X_{t+1}', \md(X_t)) - \hid(X_t, \md(X_t))\big)^+ \givenBig X_t} \nonumber \\
        &\qquad\leq \EBig{\Big(\hid(X_{t+1}', \md(X_t)) - \Big(1-\frac{1}{2\lamw}\Big) \hid(X_t, \md(X_t))\big)^+ \givenBig X_t}  
        \label{eq:proof-id-monotonic-intermediate-3}
        \leq \frac{4\lamw^{1/2}}{\sqrt{N}}.
    \end{align}
    We couple $X_{t+1}'$ and $X_{t+1}$ such that $X_{t+1}'(\{i\}) = X_{t+1}(\{i\})$ for all $i\leq \min(N\md(X_t), \Ngood)$. Then
    \begin{align}
        & \EBig{\big(\hid(X_{t+1}, \md(X_t)) - \hid(X_t, \md(X_t))\big)^+ - \big(\hid(X_{t+1}', \md(X_t)) - \hid(X_t, \md(X_t))\big)^+ \givenBig X_t}  \nonumber \\
        &\qquad\leq \EBig{\big(\hid(X_{t+1}, \md(X_t)) - \hid(X_{t+1}', \md(X_t))\big)^+ \givenBig X_t}  \nonumber \\
        &\qquad=  \EBig{\Big(\max_{m'\in[0,1]_N,\,m'\leq\md(X_t)}\hw(X_{t+1}, m') - \max_{m'\in[0,1]_N,\,m'\leq\md(X_t)} \hw(X_{t+1}', m')\Big)^+ \givenBig X_t}  \nonumber \\
        &\qquad\leq \EBig{\max_{m'\in[0,1]_N,\,m'\leq\md(X_t)} \big(\hw(X_{t+1}, m') - \hw(X_{t+1}', m') \big)^+ \givenBig X_t}  \nonumber \\
        &\qquad\leq \EBig{\max_{m'\in[0,1]_N,\,m'\leq\md(X_t)} \norm{X_{t+1}([Nm']) - X_{t+1}'([Nm'])}_\wmat \givenBig X_t} \nonumber \\
        &\qquad \leq \EBig{\max_{m'\in[0,1]_N,\,m'\leq\md(X_t)}  \Big(\normbig{X_{t+1}\big([Nm')] \backslash [\Ngood]\big)}_\wmat + \normbig{X_{t+1}'\big([Nm'] \backslash [\Ngood]\big)}_\wmat \Big)  \givenBig X_t} \nonumber \\
         \label{eq:proof-id-monotonic-intermediate-4-0}
        &\qquad\leq \frac{2 \lamw^{1/2}}{N}  \EBig{\max_{m'\in[0,1]_N,\,m'\leq\md(X_t)} \big(Nm' - \Ngood\big)^+ \givenBig X_t} \\
        \nonumber
        &\qquad\leq \frac{2 \lamw^{1/2}}{N} \EBig{\big(N\md(X_t) - \Ngood\big)^+ \givenBig X_t}\\
        \label{eq:proof-id-monotonic-intermediate-4}
        &\qquad\leq \frac{4\lamw^{1/2}}{\rhoBudget\sqrt{N}} + \frac{2\lamw^{1/2}}{N},
    \end{align}
    where \eqref{eq:proof-id-monotonic-intermediate-4-0} follows from the facts $\norm{v}_\wmat \leq \lamw^{1/2}\norm{v}_1$ for any vector $v$ and that $\normplain{X_{t+1}(D)}_1 = \normplain{X_{t+1}'(D)}_1 = m(D)$ for any $D\subseteq[N]$, and \eqref{eq:proof-id-monotonic-intermediate-4} applies the bound on $\Ebig{(N\md(X_t) - \Ngood)^+ \givenbig X_t}$ in \Cref{lem:id-pibs-consistency}. 
    
    Combining \eqref{eq:proof-id-monotonic-intermediate-2}, \eqref{eq:proof-id-monotonic-intermediate-3} and \eqref{eq:proof-id-monotonic-intermediate-4}, we get 
    \begin{equation*}
        \Ebig{(\md(X_t) - \md(X_{t+1}))^+ \givenbig X_t} \leq \frac{4\ratiocw\lamw^{1/2}(1+\rhoBudget)}{\rhoBudget^2 \sqrt{N}} + \frac{2\ratiocw\lamw^{1/2}+\rhoBudget}{\rhoBudget N}.\Halmos
    \end{equation*}
\end{proof}

\subsection{Proof of \Cref{lem:id-large-enough}}\label{sec:proof-id-policy:large-enough}

\begin{proof}{\textit{Proof of \Cref{lem:id-large-enough}}.}
    This lemma almost directly follows from the definition of the focus set $D_t \triangleq [N\md(X_t)]$, where
    \begin{equation}
        \md(X_t) \triangleq \max \{m\in [0,1]_N \colon \ratiocw \hid(X_t, m) \leq \rhoBudget(1-m)\}.\label{eq:def-md-Xt}
    \end{equation}
    The main technical step is to account for the discretization effect, i.e., $\md(X_t)$ must be a multiple of $1/N$.

    We focus on the non-trivial case where $\md(X_t) < 1$. 
    By the maximality in the definition \eqref{eq:def-md-Xt}, we have
    \begin{equation}\label{eq:pf-id-large-enough:intermediate-1}
        \ratiocw \hid\Big(X_t, \md(X_t)+\frac{1}{N}\Big) > \rhoBudget\Big(1-\md(X_t)-\frac{1}{N}\Big). 
    \end{equation}
    Next, by the Lipschitz continuity of $\hid(x, m)$ stated in \eqref{eq:hid-feature-lyapunov:lipschitz}, we have
    \begin{equation}\label{eq:pf-id-large-enough:intermediate-2}
        \ratiocw \hid\Big(X_t, \md(X_t)+ \frac{1}{N}\Big) 
        \leq \ratiocw \hid\Big(X_t, \md(X_t)\Big) + \frac{2 \ratiocw \lamw^{1/2}}{N}.
    \end{equation}
    Combining \eqref{eq:pf-id-large-enough:intermediate-1} with \eqref{eq:pf-id-large-enough:intermediate-2}, we get
    \[
        \rhoBudget(1-\md(X_t)) < 
        \ratiocw \hid(x, \md(X_t)) + \frac{2\ratiocw \lamw^{1/2}+\rhoBudget}{N}. \Halmos
    \]
\end{proof}

\section{Experiments}
\label{sec:experiments}
In our theoretical analysis, we have shown that our policies achieve asymptotic optimality assuming only the aperiodic unichain assumption, removing GAP or SA assumed in prior work. 
In this section, we compare the numerical performance of our policies with the policies in the prior work when the number of arms $N$ is finite. 
These numerical results complement our theory, showing that our policies also empirically outperform previous policies on some RB problems that violate GAP or SA (\Cref{sec:experiments:compare-non-ugap} and \ref{sec:experiments:compare-non-sa}) but still satisfy our \Cref{assump:aperiodic-unichain}. 
Moreover, such RB problems are not rare. For GAP, there exists some natural classes of randomly-generated RB instances, a decent fraction of which do not satisfy GAP under all LP-Priority policies (\Cref{sec:experiments:non-ugap-is-common}); 
for SA, we give two counterexamples for SA and discuss ways to construct more counterexamples in Appendix~\ref{app:sa-counterexample}. The code for all the experiments are available on Github \cite{Hon_24_github}.

\subsection{Comparing policies on two non-GAP examples}
\label{sec:experiments:compare-non-ugap}

In this section, we consider two examples where two prevalent versions of LP-Priority policies, the Whittle index policy \cite{Whi_88_rb} 
and the LP index policy \cite{GasGauYan_23_exponential} 
(whose variants are also studied as primal-dual heuristics, Lagrange-based policies, or the Optimal Lagrangian Index Policy \cite{BerNin_00_RB_index,Haw03_Lagrangian,HuFra_17_rb_asymptotic, BroSmi_19_rb})
are not asymptotically optimal, either because of the failure of GAP or because the policy itself is not well-defined. 
We will simulate the Whittle index policy and the LP-index policy on these two examples, where the Whittle index policy is implemented using the algorithm in \cite{GasGauKhu_23}. 
Along with these two LP-Priority policies, we also evaluate the performance of the FTVA policy \cite{HonXieCheWan_23}, the set-expansion policy (\Cref{sec:results_set-expansion-policy}), and the ID policy (\Cref{sec:results_ID-policy}). 
Note that for the set-expansion policy, we will consider two versions of implementations that perform action rectification differently: the vanilla version performs action rectification in the uniformly random way as described in \Cref{alg:set-expansion}; the ``LP index version'' of the set-expansion policy applies the LP index policy to the arms not in the focus set. See Appendix~\ref{app:exp-details:set-expansion} for implementation details of the set-expansion policy.

\begin{figure}
    \FIGURE{
    \subcaptionbox{Performance comparison on the three-state example. \label{fig:non-ugap-perf:three-state}}{\includegraphics[width=7.5cm]{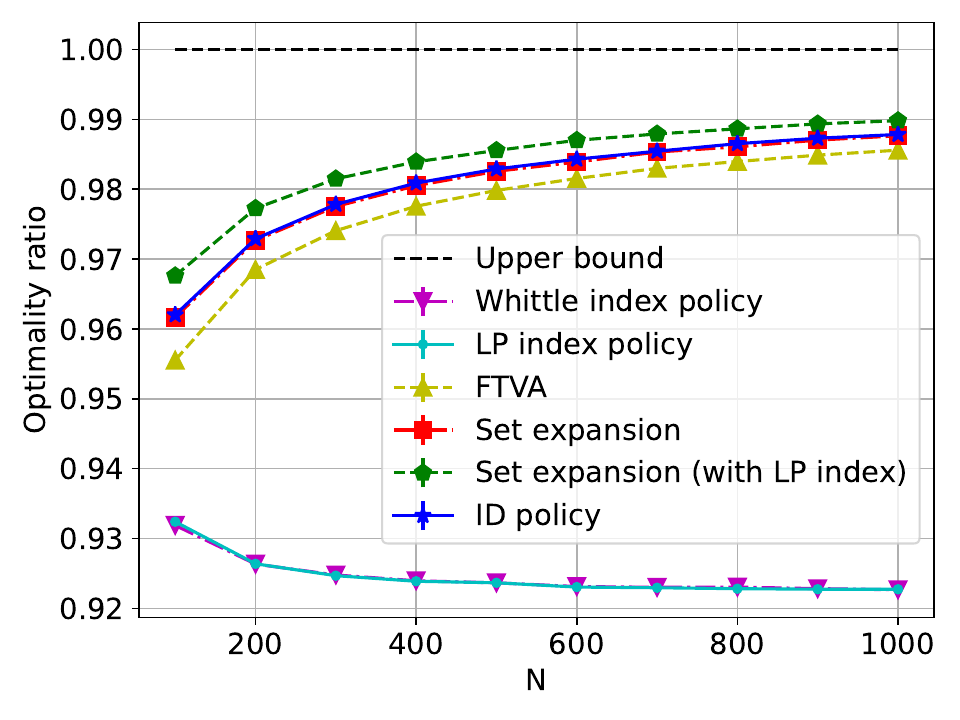}}
    \hfill
    \subcaptionbox{Performance comparison on the eight-state example.  \label{fig:non-ugap-perf:eight-state}}{\includegraphics[width=7.5cm]{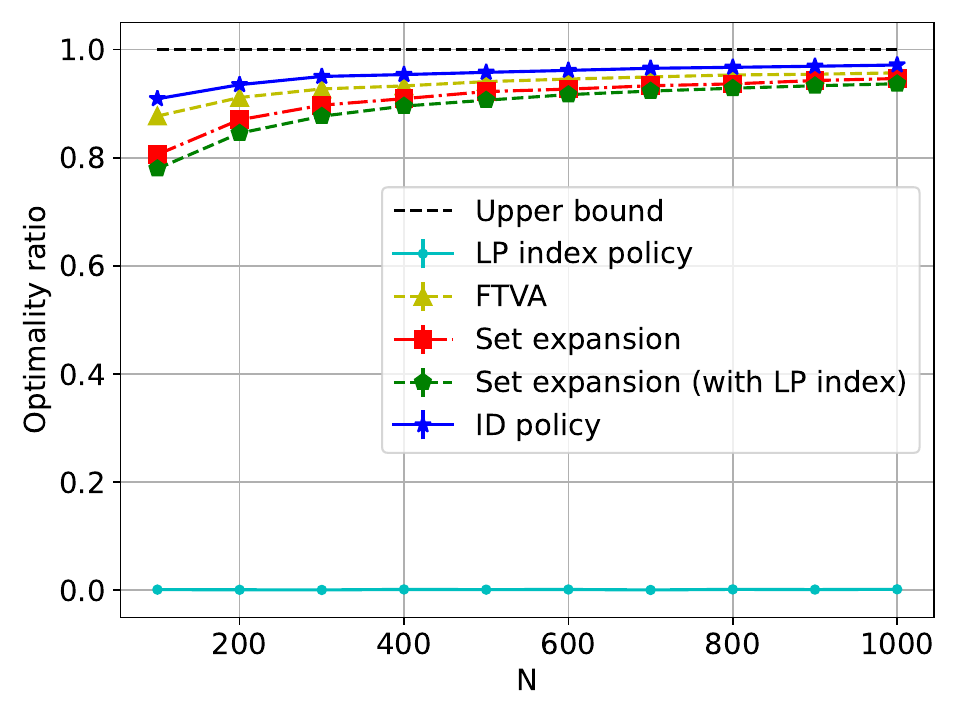}}
    }
    {Performance comparison on two examples where GAP fails to hold.  \label{fig:non-ugap-perf}}
    {}
\end{figure}

For all the simulations, we compare the \emph{optimality ratios} of the policies, which are their average rewards normalized by the optimal value of the  LP relaxation in \eqref{eq:lp-single}. Note that the optimality ratio of an asymptotically optimal policy converges to $1$ as $N\to\infty$. 
Each optimality ratio is estimated from multiple replications. For each replication, the initial state of each arm is independently sampled from the uniform distribution over the state space. Details of the simulation setting can be found in Appendix~\ref{app:exp-details:rb-examples-details}.

The first non-GAP example is an RB problem defined by a single-armed MDP with three states and was obtained in a random search by \cite{GasGauYan_20_whittles}; it was also evaluated in Figure~1 of \cite{HonXieCheWan_23}; see Appendix~\ref{app:exp-details:rb-examples-details} for its detailed definition. Note that in this three-state example, there is only one LP-Priority policy, so the Whittle index policy and the LP index policy are identical. 
Our simulation is shown in \Cref{fig:non-ugap-perf:three-state}: 
The Whittle index policy and the LP index policy are asymptotically suboptimal; 
FTVA outperforms the LP index policy and appears to be asymptotically optimal; 
the set-expansion policy and the ID policy are strictly better than FTVA; the LP-index version of set-expansion policy has the best performance among all these policies.

The second non-GAP example is defined by a single-armed MDP with eight states, and is adapted from Figure~2 of \cite{HonXieCheWan_23}; see Appendix~\ref{app:exp-details:rb-examples-details} for its detailed definition. 
A notable feature of this example is the existence of a local attractor, where the scaled state-count vector of the arms is attracted to a distribution other than the optimal stationary distribution, which is a mode of non-GAP-ness not observed in earlier literature. 
Our simulation result on this example is shown in \Cref{fig:non-ugap-perf:eight-state}: 
The Whittle index policy is not included since this example is non-indexable; 
the LP index policy has nearly zero reward; 
FTVA, the set-expansion policy, and the ID policy are asymptotically optimal, and the ID policy has the best performance among all these policies.

\subsection{Commonness of non-GAP examples}
\label{sec:experiments:non-ugap-is-common}

\begin{figure}
    \FIGURE{
        \subcaptionbox{Dirichlet($1$)\label{fig:eigs:dirichlet-1}}{ \includegraphics[width=5cm]{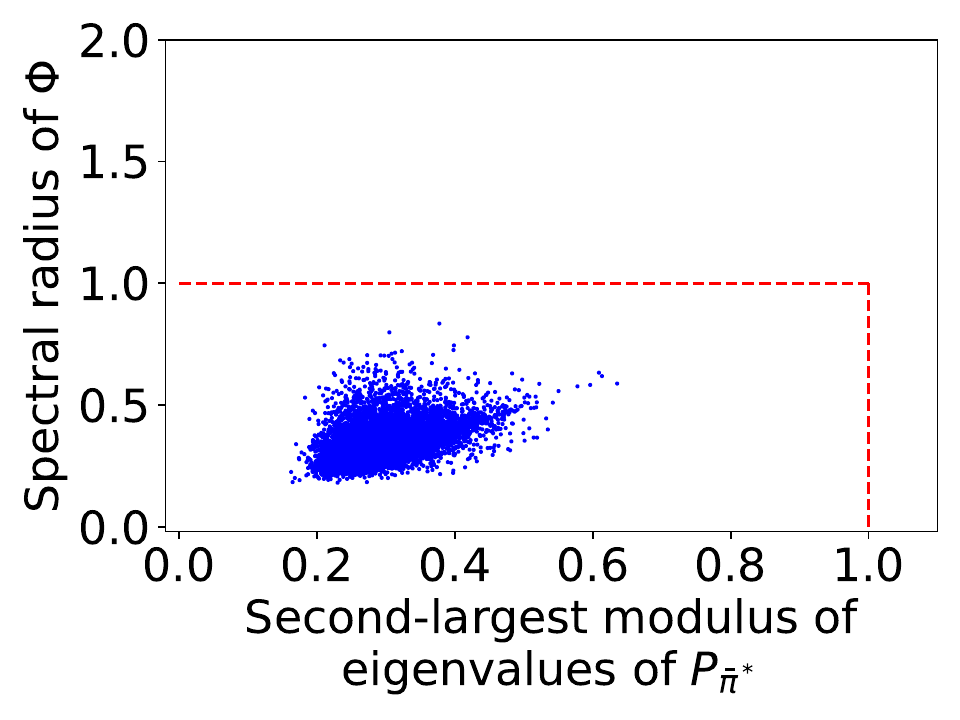}}
        \hfill
        \subcaptionbox{Dirichlet($0.2$)\label{fig:eigs:dirichlet-02}}{\includegraphics[width=5cm]{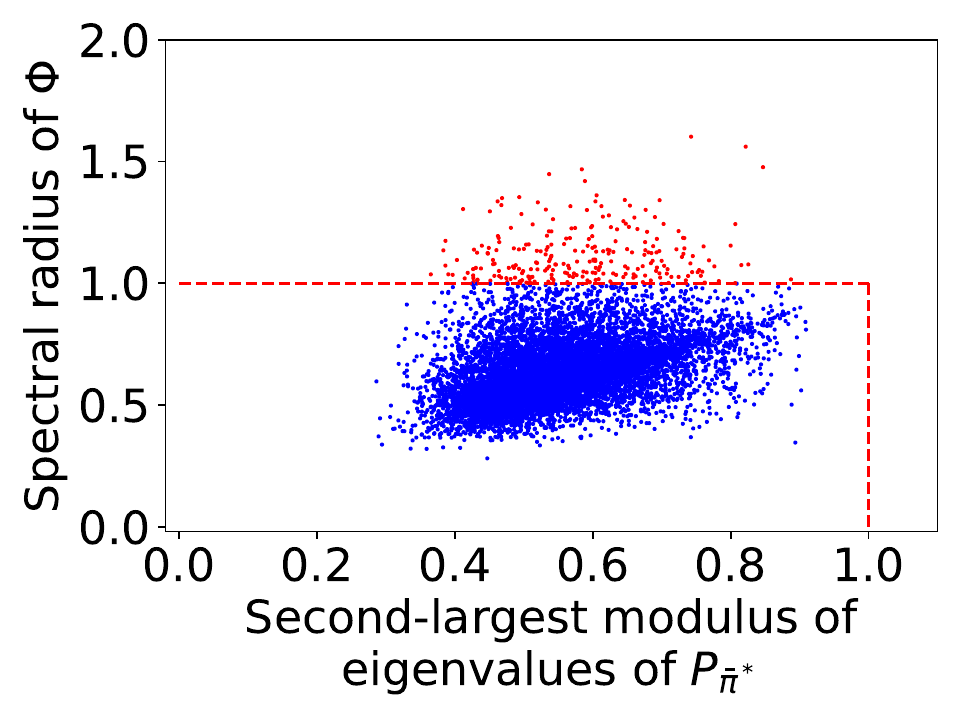}}
        \hfill
        \subcaptionbox{Dirichlet($0.05$)\label{fig:eigs:dirichlet-005}}{ \includegraphics[width=5cm]{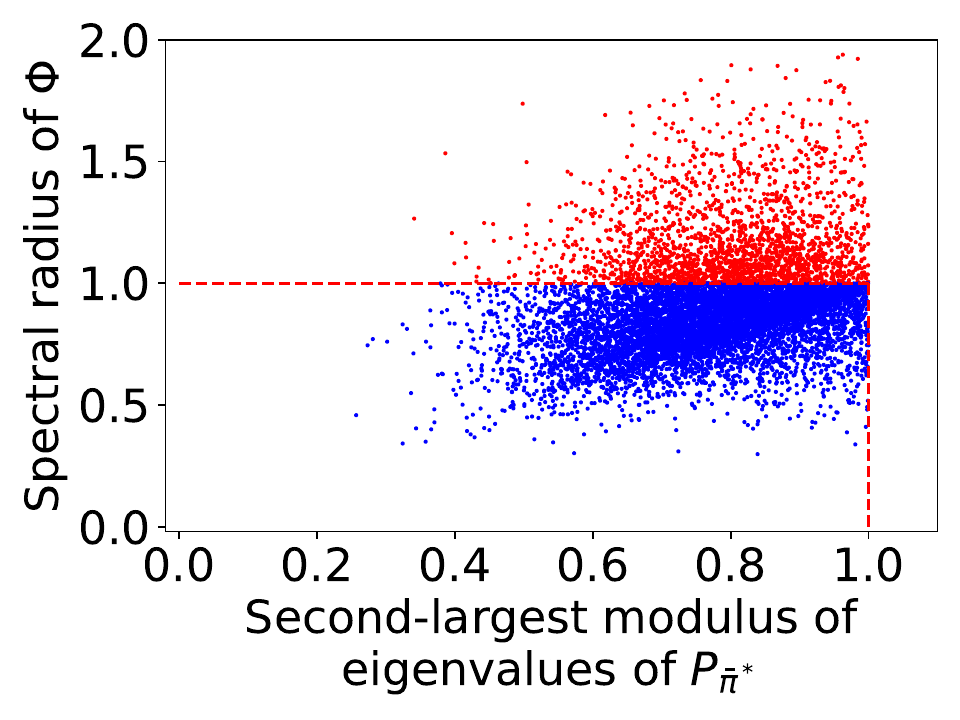}}
    }
    {Scatter plots illustrating the eigenvalues of about $10^4$ random RB problems, whose transition probabilities and reward function follow the symmetric Dirichlet distribution with different parameters in each subplot.\label{fig:eigs}}
    {Each point in a scatter plot represents an RB problem, whose $x$-coordinate (or $y$-coordinate) represents the second largest modulus of  eigenvalues of $P_\pibs$ (or spectral radius of $\Phi$). Each RB problem has $|\sspa| = 10$. The points outside the dashed box (highlighted in red) represent RB problems that violate GAP under all LP-Priority polices. 
    }
\end{figure}

Although counterexamples to GAP are well-known to exists starting from \cite{WebWei_90}, such examples are rare in previously known classes of RB problems. In particular, it has been found in \cite{GasGauYan_23_whittles} that when $|\sspa| = 3$, no more than $0.2\%$ of the uniformly random examples are non-indexable or violate GAP for Whittle index policy; when $|\sspa|$ gets large, this fraction among the uniformly random examples further decreases and becomes less than $10^{-4}\%$ when $|\sspa| = 7$.

In this section, we study some classes of RB problems that are \emph{sparser} than the RB problems following the uniform distribution, i.e., RB problems with fewer non-zero entries in the transition kernels. 
Specifically, we generate random RB instances with $|\sspa| = 10$, where the transition distributions $P(s, a, \cdot)$ for $s\in\sspa, a\in\aspa$ are independently sampled from symmetric Dirichlet distributions, a natural class of distributions for generating points on the probability simplex. 
The reward functions $r(\cdot, a)$ for $a\in \aspa$ are sampled from the same Dirichlet distribution.  
We have also experimented with uniformly distributed reward functions and observed similar results, as discussed in the appendix~\ref{app:exp-details:dense-reward}. Note that all entries less than $10^{-7}$ are rounded to zero.

To count the number of non-GAP examples, we focus on identifying the RB problems that are \emph{locally unstable}, which implies the violation of GAP under all LP-Priority policies. 
The local instability of an RB problem is easy to certify: it happens when the spectral radius of a certain matrix $\Phi$ representing the local mean-field dynamics under the LP-Priority policies is larger than $1$, when the definition of $\Phi$ is unambiguous (See Appendix~\ref{app:exp-details:commonness} for details). 
Based on this fact, we make three scatter plots in \Cref{fig:eigs}, each visualizing $10^4$ independently generated RB problems following the symmetric Dirichlet distribution with different parameters. Each point in a scatter plot represents an RB problem, whose $y$-coordinate is the spectral radius of $\Phi$, and whose $x$-coordinate represents the second-largest absolute value of $P_\pibs$'s eigenvalues.

In Figures~\ref{fig:eigs:dirichlet-1}, \ref{fig:eigs:dirichlet-02}, and \ref{fig:eigs:dirichlet-005}, the parameter of the symmetric Dirichlet distribution decreases as $1$, $0.2$, and $0.05$, indicating the increased sparsity of the single-armed MDPs. 
In particular, the average fractions of non-zero entries of the transition distributions $P(s, a, \cdot)$ are approximately $100\%$, $95.3\%$, and $58.7\%$, respectively. 
In \Cref{fig:eigs:dirichlet-1}, because Dirichlet($1$) is the uniform distribution on probability simplex, the fact that no RB examples are found to be locally unstable is consistent with the findings in \cite{GasGauYan_23_exponential}. 
In contrast, in \Cref{fig:eigs:dirichlet-005}, under Dirichlet($0.05$) distribution, a significant proportion of the problem instances are locally unstable (marked in red); moreover, even if we focus on the examples whose $x$-coordinate is less than $0.95$, that is, the examples whose $P_\pibs$ is aperiodic unichain with a decently large spectral gap, $1844$ out of $9131$ examples (about $20.2\%$) are locally unstable. 
Note that when generating the $10^4$ random examples for each of the three scatter plots, the definition of $\Phi$ is ambiguous for only a small number of examples, which we do not display in the figures; specifically, the actual number of examples displayed in Figures~\ref{fig:eigs:dirichlet-1}, \ref{fig:eigs:dirichlet-02}, and \ref{fig:eigs:dirichlet-005} are $9776$, $9918$, and $9914$, respectively. 

The experiment in \Cref{fig:eigs} shows that for the RB problems whose single-armed MDPs are sparse, a significant fraction of them could be aperiodic unichain, but violate GAP for all LP-Priority policies.

\begin{figure}
    \FIGURE{
        \subcaptionbox{LP index policy. \label{fig:asym-subopt:lpp}}{\includegraphics[width=5cm]{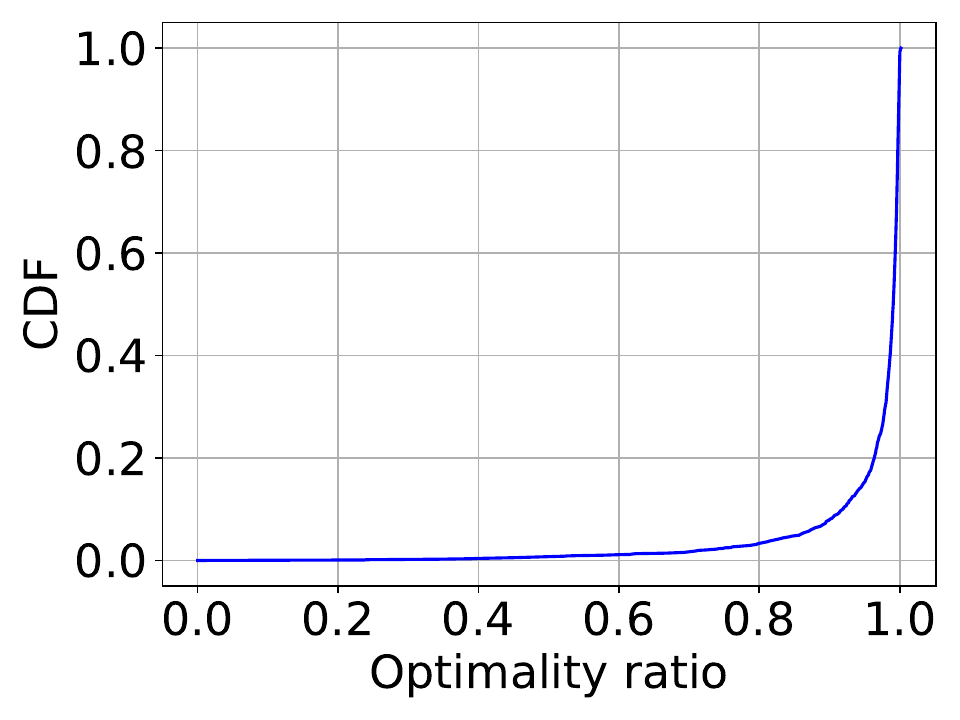}}
        \hfill
        \subcaptionbox{Whittle index policy. \label{fig:asym-subopt:whittle}}{\includegraphics[width=5cm]{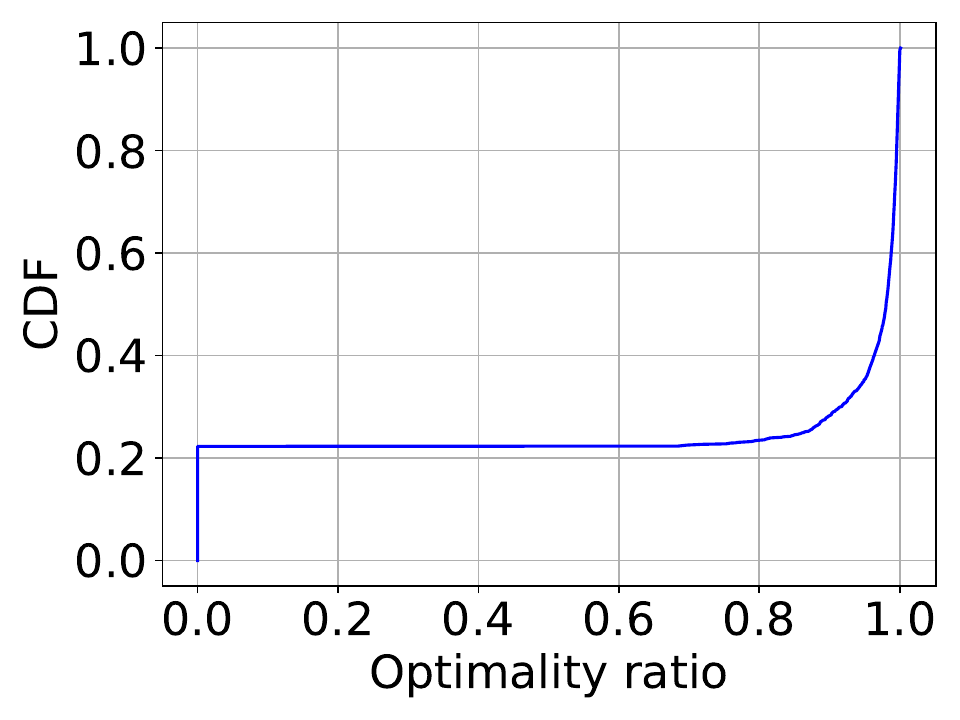}}
        \hfill
        \subcaptionbox{Maximal reward of the two index policies. \label{fig:asym-subopt:max}}{\includegraphics[width=5cm]{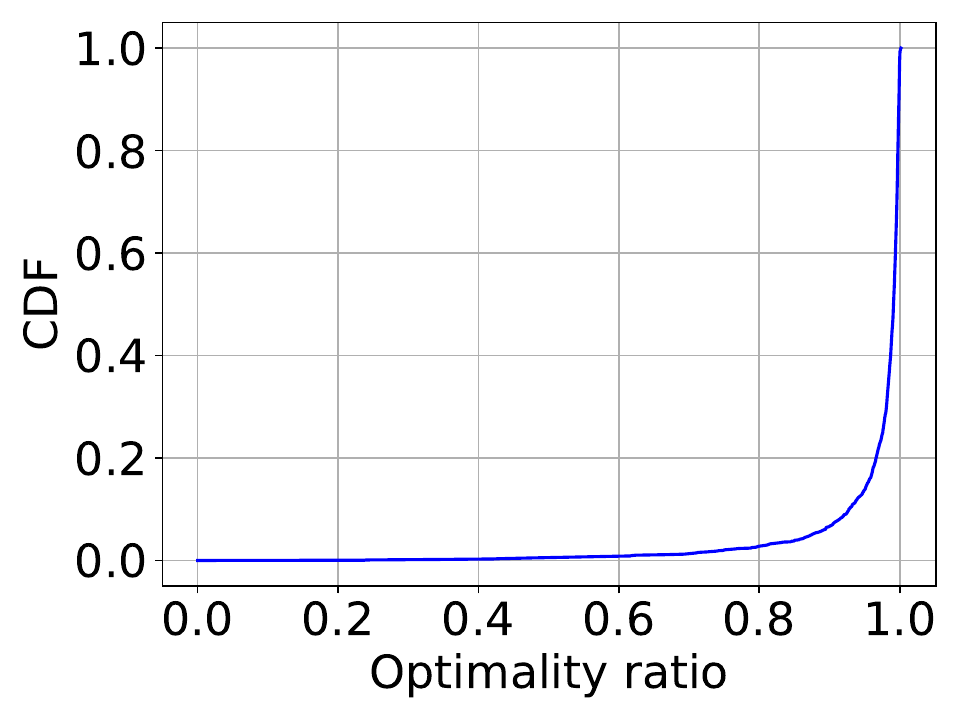}}
    }
    {CDF of the optimality ratios of LP-Priority policies when $N=500$, among $2049$ non-GAP examples generated from Dirichlet($0.05$). \label{fig:asym-subopt}}
    {We regard the average reward of Whittle index policy as $0$ if it is not well-defined.}
\end{figure}

\paragraph{\textbf{LP-Priority policies on random non-GAP examples}}
Violation of GAP invalidates the asymptotic optimality guarantee of an LP-Priority policy, but how suboptimal is an LP-Priority policy when GAP does not hold?
In \Cref{fig:asym-subopt}, we plot some cumulative distribution function (CDF) curves representing the optimality ratios of the LP index policy, the Whittle index policy, and their maximal performance when $N=500$, among $2034$ locally unstable examples generated from the Dirichlet($0.05$) distribution. 
Each policy is simulated for $2\times 10^4$ time steps on every example. 
As we can see from \Cref{fig:asym-subopt:max}, the average rewards under the LP index policy or the Whittle index policy are close to the LP upper bound in most non-GAP examples, which explains the good performance of LP-Priority policies observed in practice; 
on the other hand, there are about $6.7\%$ of the examples where the average rewards of both policies are less than $90\%$ of the LP upper bound. 
This experiment shows that it is not uncommon for LP-Priority policies to be substantially suboptimal when the single-armed MDPs of the RB problems are sparse. In these cases, relaxing the GAP condition could bring a practical benefit. 

In \Cref{fig:non-ugap-perf:dirichlet-582} and \Cref{fig:non-ugap-perf:dirichlet-355}, we pick two non-GAP examples where both the Whittle index and LP index policies have optimality ratios of less than $90\%$, and compare the performance of different policies there. 
In both examples, FTVA, the two versions of set-expansion policy, and the ID policy outperform the LP-Priority policies, with clear and discernible differences.

\begin{figure}
    \FIGURE{
        \subcaptionbox{First Dirichlet($0.05$) example. \label{fig:non-ugap-perf:dirichlet-582}}{\includegraphics[width=7.5cm]{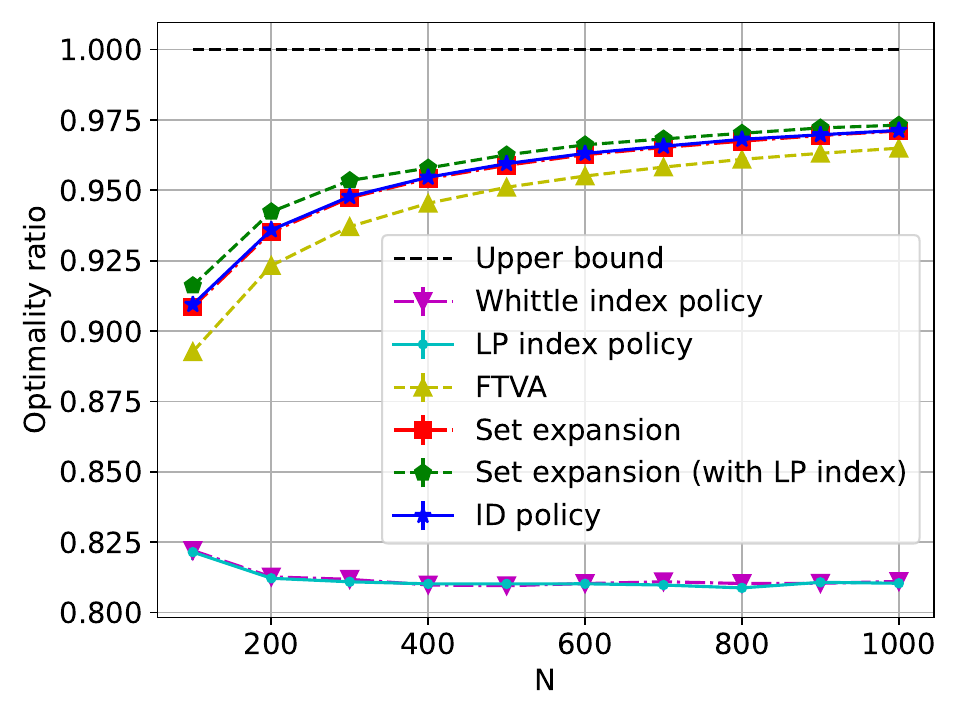}}
        \hfill
        \subcaptionbox{Second Dirichlet($0.05$) example. \label{fig:non-ugap-perf:dirichlet-355}}{\includegraphics[width=7.5cm]{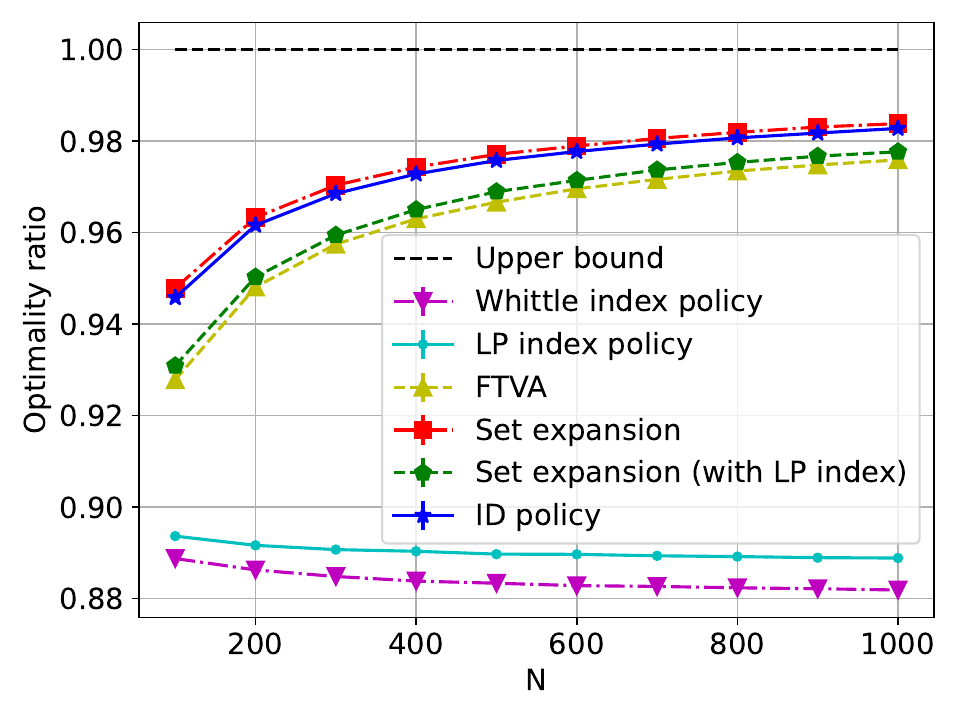}}
    }
    {Performance comparison on two Dirichlet($0.05$) examples where GAP fails to hold. \label{fig:non-ugap-perf-random}}
    {}
\end{figure}

\subsection{Comparing policies on two non-SA examples}
\label{sec:experiments:compare-non-sa}

In this section, we consider the same set of policies as in \Cref{sec:experiments:compare-non-ugap} and compare their performances on two examples that violate SA (Assumption~1 of \cite{HonXieCheWan_23}) but satisfy our aperiodic unichain assumption. 
We give one of the examples at the end of this subsection, and another one in Appendix~\ref{app:sa-counterexample}, where we also discuss ways to construct more counterexamples to SA. On a high level, in each example that we construct, an arm needs to strictly follow a particular policy to reach and remain in the states with high rewards, which is hard to achieve by following some virtual actions that are not generated based on the true state of the arm.

\begin{figure}
    \FIGURE{
        \subcaptionbox{RB problem defined in \Cref{fig:sa-counter}. \label{fig:non-sa:eight-state}}{\includegraphics[width=7.5cm]{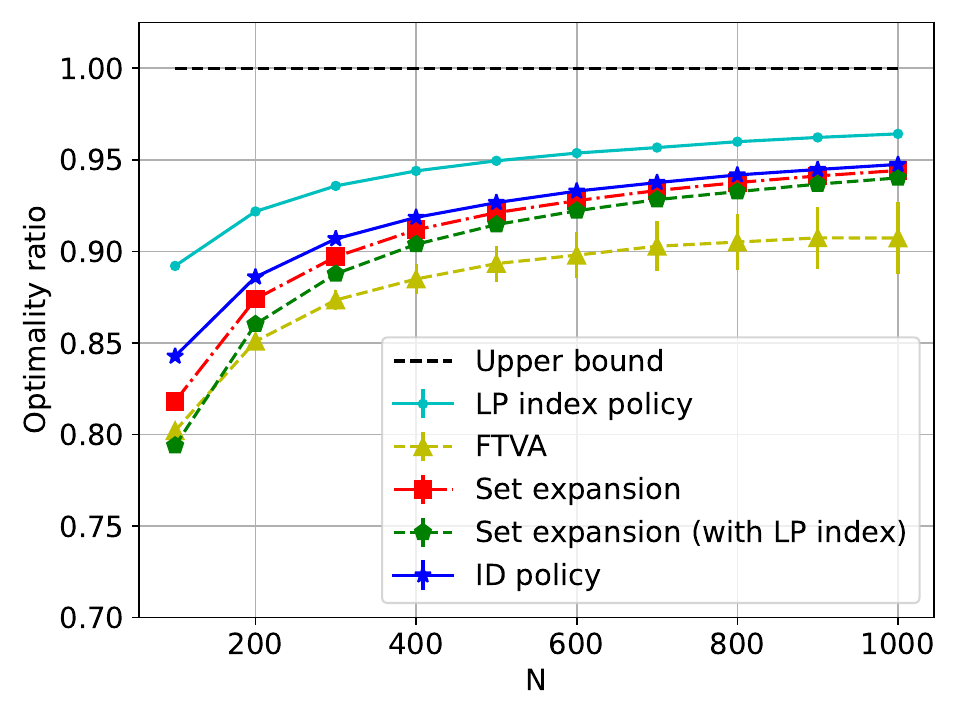}}
        \hfill
        \subcaptionbox{RB problem defined in \Cref{fig:sa-counter-big2}.  \label{fig:non-sa:eleven-state}}{\includegraphics[width=7.5cm]{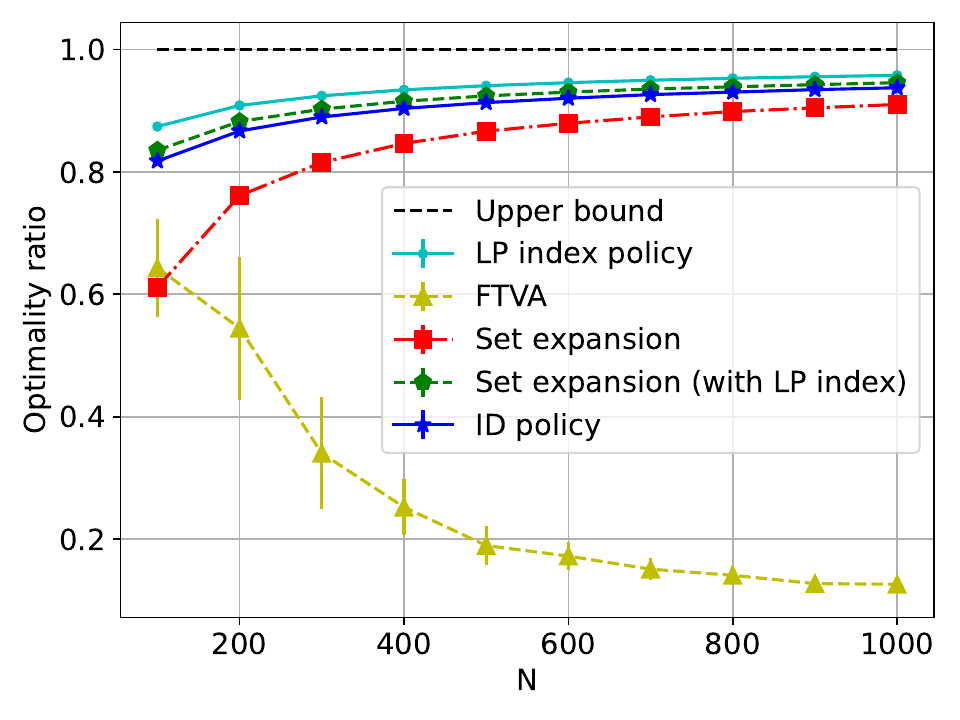}}
    }
    {Performance comparison of the policies on counterexamples to SA defined in Appendix~\ref{app:sa-counterexample}. \label{fig:non-sa}}
    {}
\end{figure}

The simulation results are shown in \Cref{fig:non-sa}. 
Note that in each replication, the initial state of each arm is independently sampled from the uniform distribution over the state space; FTVA is simulated for $1.6\times 10^5$ time steps with five replications, whereas the other policies are simulated for $2\times 10^4$ time steps with five replications. 
Despite the longer simulations, the performances of FTVA still exhibit significant variability with large confidence intervals. 
In both examples, FTVA perform worse than the other policies, especially in the RB problem considered in \Cref{fig:non-sa:eleven-state}, whose single-armed MDP has a larger state space. 
In contrast, the ID policy and the two version of the set-expansion policy demonstrate solid performances, though not quite reaching the performances of the LP index policy. 
The Whittle index policy is not well-defined on these two examples due to the multichain nature of the single-armed MDPs (see \cite{GasGauKhu_23} for details).

\paragraph{Definition of one of the non-SA examples.}

\begin{figure}[t]
    \FIGURE{\includegraphics[height=3cm]{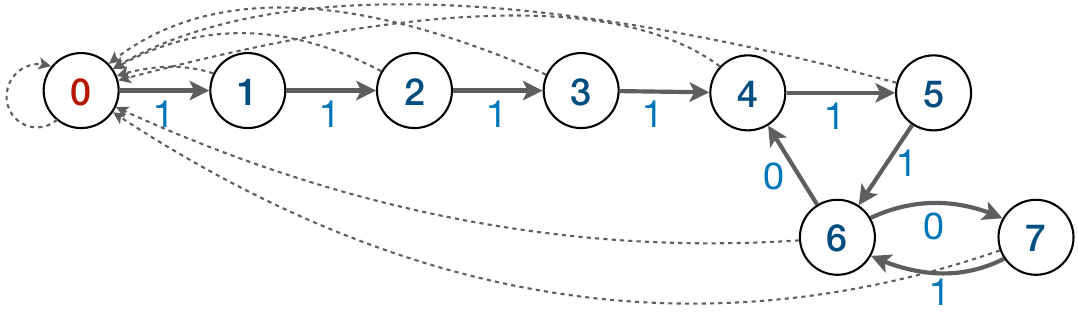}}
    {A counterexample to the Synchronization Assumption in \cite{HonXieCheWan_23}. \label{fig:sa-counter}}
    {Each circle denotes a state, indexed by $0,1,2,\dots, 7$. Each arrow denotes a possible transition. The numbers labeled on the solid-line arrows denote actions. 
    If an arm takes an action that is labeled on one of the outward solid-line arrows at its current state, it picks such an arrow labeled by the action uniformly at random and transitions to a nearby state along the arrow; otherwise, the arm jumps to state $0$. The reward is $1$ if an arm is in states $\{4,5,6,7\}$ and takes the action on an outward solid-line arrow at its current state. Otherwise, the reward is zero. The budget parameter $\alpha$ is set to be $0.6$.}    
\end{figure}

Now we define the example considered in \Cref{fig:non-sa:eight-state}. 
The single-armed MDP of this example is defined using Figure~\ref{fig:sa-counter}. 
This figure consists of a set of circles denoting states and a set of arrows in solid lines and dashed lines denoting possible transitions under different actions. The states are indexed by $0,1,2,\dots, 7$. Each solid arrow is labeled by an action, which is $0$ or $1$. 
In each time step, when the arm takes an action labeled on one of the outward solid-line arrows adjacent to its current state, the arm transitions to a random nearby state through one of such arrows. When the arm takes an action that does not exist on any of the adjacent outward solid-line arrows, it jumps to state $0$. 
For example, if the arm takes action $1$ at state $7$, it goes to state $6$ with probability $1$; if the arm takes action $0$ at state $6$, it goes to state $7$ or $4$ with equal probabilities; if the arm takes action $0$ at state $2$, it jumps to state $0$ with probability $1$. 
For the reward functions, one unit of reward is generated if the arm is in states $\{4,5,6,7\}$ and takes the action on an adjacent outward solid-line arrow; no reward is generated otherwise. 
We let $\alpha = 0.6$, that is, the arm is activated for $0.6$ fraction of the time in the long run. 

One can verify that the optimal single-armed policy $\pibs$ takes each action with probability $0.5$ in the states $\{0,1,2,3\}$, and always takes the actions labeled on the solid line arrows in the states $\{4,5,6,7\}$. 
Thus, $\pibs$ induces an aperiodic unichain with the recurrent class $\{4,5,6,7\}$. The long-run average reward of $\pibs$ is $1$. 

On the other hand, we argue that SA is violated in this example. 
To see this, recall the leader-and-follower system in the SA (see Section 4.1 in \cite{HonXieCheWan_23}), which consists of two arms, the leader arm and the follower arm, whose states are denoted as $\syshat{S}_t$ and $S_t$; the leader arm takes the action $\syshat{A}_t \sim \pibs(\cdot | \syshat{S}_t)$, and the follower arm takes the same action, $A_t = \syshat{A}_t$. 
SA requires that the stopping time $\tau \triangleq \inf\{t\colon S_t = \syshat{S}_t\}$ has a finite expectation for any possible pair of initial states. 
However, in the above example, if we initialize the pair of states as $\syshat{S}_0 = 7$ and $S_0 = 0$, $\syshat{S}_t$ will remain in states $\{4,5,6,7\}$ under $\pibs$, and the action sequences applied by both arms will not contain more than two subsequent $1$'s. 
Consequently, $S_t$ always falls back to the state $0$ before reaching state $3$, so the two arms never reach the same state, implying $\tau = \infty$.

\section{Conclusion and discussions}
\label{sec:discussion}

In this paper, we consider the infinite-horizon, average-reward restless bandit problem. We introduce a new class of policies that are asymptotically optimal with $O(1/\sqrt{N})$ optimality gaps, if the optimal single-armed policy induces an aperiodic unichain. Our paper is the first to show that asymptotic optimality can be achieved without any additional assumptions like GAP and SA. 

Our policy design and analysis highlight the use of multiple, bivariate Lyapunov functions. 
This novel approach holds promise beyond restless bandits, showing potential for a broader class of large stochastic systems consisting of many coupled components. 
In such complex systems, it can be challenging to directly design a policy that steers the whole system towards optimality or to construct a Lyapunov function that certifies such convergence. 

To complement our theory, we simulate our policies on examples where either GAP or SA fails, along with the policies from prior work. 
Our policies consistently demonstrate good performance, whereas the policies from the prior work may perform suboptimally in some examples. 
Additionally, we identify some natural classes of RB instances where GAP is violated with considerable probabilities and discuss a method for constructing more counterexamples of SA.

There are several interesting directions for future work. 
Some directions have been largely addressed by the follow-up papers reviewed in \Cref{sec:additional-related-work}. These directions include multiple actions, heterogeneous arms, and exponential optimality gaps without the global attractor property.
However, many interesting open problems remain. 
For example, an interesting open direction is to improve the dependency of the optimality gap on the problem parameters other than $N$, such as the spectral gap of the single-armed system. 
Another open direction is exploring restless bandits (or weakly-coupled MDPs) with an infinite state space, which could naturally arise in various applications. 
The third open direction is learning to optimize restless bandits with unknown parameters. 
Finally, our bivariate Lyapunov function technique (\Cref{sec:formalizing}) seems to be a flexible tool that allows us to build a complex Lyapunov function from a collection of simple Lyapunov functions. It would be intriguing to investigate its potential applications beyond the restless bandit setting.

\ACKNOWLEDGMENT{Yige Hong and Weina Wang are supported in part by U.S.\ National Science Foundation (NSF) grants ECCS-2145713, CCF-2403194, CCF-2428569, and ECCS-2432545.
Yudong Chen is supported in part by NSF grant CCF-2233152. Qiaomin Xie is supported in part by NSF grants CNS-1955997, ECCS-2339794, and ECCS-2432546.}

\bibliographystyle{informs2014}
\bibliography{refs-yige-v240809}

\newpage
\clearpage

\begin{APPENDICES}

\SingleSpacedXI

\section{Additional counterexamples for the Synchronization Assumption}\label{app:sa-counterexample}

Recall that in \Cref{sec:experiments:compare-non-sa}, we have compared the performances of our policies and the policies from prior work on two examples where the Synchronization Assumption (SA), required by the FTVA policy in \citep{HonXieCheWan_23}, does not hold. One of the examples has been defined in \Cref{fig:sa-counter}. 

In this appendix, we discuss how to generalize the graphical way of defining the example in \Cref{fig:sa-counter} to construct non-SA examples of arbitrary sizes. In particular, we use this method to define the example simulated in \Cref{fig:non-sa:eleven-state}. 

To specify the single-armed MDP for a non-SA example, we can first pick a set of transient states (like the states $\{0,1,2,3\}$ in \Cref{fig:sa-counter}), and a set of recurrent states (like the states $\{4,5,6,7\}$ in \Cref{fig:sa-counter}). Then we set a \emph{required action} for each state, such that the arm goes to a next state if it follows the required action, and jumps back to a fixed transient state (we call it state $0$) otherwise. The reward is positive if the arm follows the required action on the recurrent states, and is zero otherwise. The budget parameter $\alpha$ is chosen to be the long-run average fraction of activations if the arm always follows the required actions. 
To make the SA fail, we can specify the transition structure and the required actions such that from state $0$, the arm can reach a recurrent state only after strictly following \emph{a particular sequence of actions}, which should be different from all possible action sequences taken by the arms in the recurrent states. 
In this way, the leader arm $\syshat{S}_t$ keeps circulating among the recurrent states, so the action sequence of the leader arm cannot bring the follower arm from state $0$ to a recurrent state.

Using the above method, we construct another example, whose single-armed MDP is illustrated in \Cref{fig:sa-counter-big2}, with budget $1/2$. One can verify that if the leader arm and the follower arm are initialized as $\syshat{S}_0 = 5$ and $S_0 = 0$, the two arms never reach the same state. 
On the other hand, note that the optimal single-armed policy $\pibs$ defined by \eqref{eq:single-arm-opt-def} induces an aperiodic unichain on this example, implying the compliance of \Cref{assump:aperiodic-unichain}.

\begin{figure}
    \FIGURE{\includegraphics[height=5.5cm]{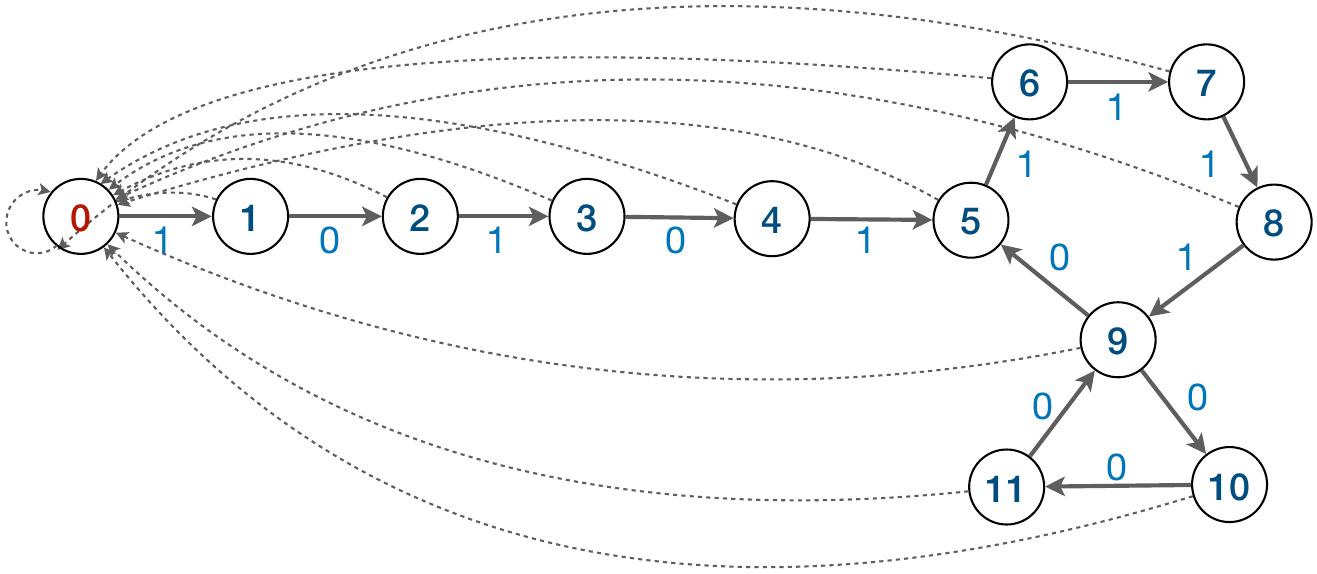}}
    {Another counterexample to the Synchronization Assumption in \citep{HonXieCheWan_23}. \label{fig:sa-counter-big2}}
    {This figure can be interpreted in the way as \Cref{fig:sa-counter}. The reward is $1$ if the arms takes the required action in states $\{5,6,7,8,9,10,11\}$. The budget parameter $\alpha$ is set to be $1/2$.}    
\end{figure}

\section{Discussion of Assumption~\ref{assump:aperiodic-unichain}}
\label{app:assump-discuss}
\subsection{Unichain conditions in prior work}
\label{app:assump-discuss-unichain}
In this appendix, we discuss our version of the unichain condition stated in Assumption~\ref{assump:aperiodic-unichain}, which assumes that the optimal single-armed policy $\pibs$ induces a unichain. 
We compare it with other unichain-like assumptions in the literature. 

The all-policy unichain condition commonly used in the average-reward MDP literature \citep[][Section~8.3]{Put_05} assumes that every stationary policy induces a unichain. Our single-policy unichain condition in \Cref{assump:aperiodic-unichain} is weaker, because we only require a particular policy $\pibs$ to induce a unichain.

Another commonly used condition in the average-reward MDP literature is the weakly-communicating condition, which assumes that the state space can be partitioned into two sets: a closed set of states where every pair of states in the set can be reached from each other under some policy, and a possibly empty set of states that are transient under every policy. 
The weakly-communicating condition, a weaker alternative to the all-policy unichain condition, ensures that an MDP has an initial-state-independent optimal average reward. 

Our single-policy unichain condition in \Cref{assump:aperiodic-unichain} and the weakly-communicating condition are not directly comparable. In particular,
\begin{itemize}
    \item Our unichain condition does not imply the weakly-communicating condition because the transient states under $\pibs$ may not be transient under every policy. 
    \item The weakly-communicating condition does not imply our unichain condition either. A counterexample is given in Example 3.1 of \citepapp{ShiTsi_05_CMDP_example}, as paraphrased below. 
    Consider the following two-state MDP with the state space $\{0,1\}$. The state of the MDP transitions to $0$ (resp., $1$) in the next time step with probability $1$ after taking action $0$ (resp., $1$), regardless of the current state; the reward function is $r(1,1) = r(0,0) = 1$ and $r(1,0) = r(0,1) = 0$. 
    This MDP is clearly communicating. 
    However, if we consider the RB problem defined by this MDP with the budget parameter $\alpha = 1/2$, then the optimal solution to the LP relaxation \eqref{eq:lp-single} is $y^*(1,1) = y^*(0,0) = 1/2$ and $y^*(1,0)=y^*(0,1) = 0$. The optimal single-armed policy is thus given by $\pibs(1|1) = \pibs(0|0) = 1$ and $\pibs(1|0)=\pibs(0|1) = 0$, which induces a Markov chain with no transitions between the two states, violating our unichain condition. 
\end{itemize}

Now we review the unichain-like conditions considered in the RB literature. 
To the best of our knowledge, all prior work on average-reward RBs assumes the all-policy unichain assumption: 
In \citep{WebWei_90,Ver_16_verloop,GasGauYan_23_whittles,GasGauYan_23_exponential}, \emph{the $N$-armed restless bandit system} is assumed to be unichain under every policy; in \citep{HonXieCheWan_23}, the single-armed MDP is assumed to be unichain under every policy. 
All these assumptions are stronger than our unichain condition in \Cref{assump:aperiodic-unichain}. In particular, assuming the $N$-armed restless bandit system to be unichain under every policy implies that the single-armed MDP is unichain under every policy, because if a policy for the single-armed MDP induces more than one recurrent classes, one can construct a policy for the $N$-armed system that also induces more than one recurrent classes. 
Nevertheless, all these unichain-like conditions in prior work are mostly for simplifying the presentation and can often be relaxed: 
For example, \citep{GasGauYan_23_exponential} mentions that their analysis still goes through if they assume the $N$-armed system to be weakly communicating; \citep{HonXieCheWan_23} discusses in their appendices that the unichain condition can be dropped as long as the Synchronization Assumption holds, albeit at the cost of a slightly more complicated formulation of the single-armed problem. 

As mentioned in \Cref{sec:additional-related-work}, two papers \citep{Yan_24_multichain,GolAvr_24_wcmdp_multichain} appearing after the arXiv version of this paper further relax the unichain condition assumed in our paper. 
Specifically, \citep{Yan_24_multichain} considers the discrete-time average-reward RB problem and shows that a policy named ``align-and-steer'' achieves $o(1)$ optimality gap. They assume a certain control problem has certain ``reachability''; this control problem captures the optimal control of the single-armed MDP's state distribution under the expected budget constraint. 
\citep{GolAvr_24_wcmdp_multichain} considers the weakly-coupled MDPs with homogeneous arms. 
Their policy achieves asymptotic optimality when there \emph{exists} a single-armed policy that induces an aperiodic unichain, such that the support of the corresponding stationary distribution contains the support of the optimal stationary distribution. 
The assumptions in both papers are weaker than the aperiodic-unichain assumption in our work. Their assumptions also hold if the single-armed MDP is aperiodic in a mild sense and weakly-communicating.

\subsection{Necessity of aperiodicity}\label{app:aperiodicity-counterexample}
In this subsection, we provide an example showing that without aperiodicity, the gap between the optimal value of the $N$-armed RB problem, $\ropt$, and the optimal value of its single-armed relaxation, $\rrel$, can be non-diminishing as $N\to\infty$. 

Consider a single-armed problem with two states, $A$ and $B$.
At each time step, the arm transitions to the other state with probability $1$, regardless of the action applied.
The reward function is given by $r(A,0)=r(B,1)=1$ and $r(A,1)=r(B,0)=0$. 
Let $\alpha$ be $\frac{1}{2}$ in the relaxed budget constraint, i.e., the arm is pulled half of the time in the long run. 
One can verify that an optimal policy $\pibs$ of this single-armed problem is given by $\pibs(0|A) = \pibs(1|B) = 1$ and $\pibs(1|A) = \pibs(0|B) = 0$, and achieves the optimal value $\rrel = 1$.  
Note that any policy in this single-armed problem induces a \emph{periodic} unichain. 

Now we consider the RB system consisting of $N$ copies of the single-armed MDP defined above, with budget constraint $\alpha N = N/2$. 
Suppose all arms of the RB system are initialized in state $A$. Then at any time $t$, either all arms are in state $A$ or all arms are in state $B$. 
In this case, all policies have the same outcome: when all arms are in state $A$, $N/2$ arms take action $0$ and generate $N/2$ unit of reward; when all arms are in state $B$, $N/2$ arms take action $1$ and generate $N/2$ unit of reward. 
Therefore, under any policy, the long-run average reward per time step and per arm is $1/2$, which has a non-diminishing gap with the upper bound $\rrel = 1$.

\section{Proof of LP relaxation upper bound}\label{app:upper-bound}
In this appendix, we prove a lemma to show that the linear program \eqref{eq:lp-single} is a relaxation of the restless bandit problem \eqref{eq:N-arm-formulation}. 
Although the lemma has been proved and is used in all prior work on average-reward restless bandit \citep[see, e.g.][Lemma 4.3]{Ver_16_verloop}, 
we prove it here for completeness. 

For ease of reference, we first restate \eqref{eq:N-arm-formulation} and \eqref{eq:lp-single}. 

\begin{align}
    \tag{\ref{eq:N-arm-formulation}} 
    \underset{\text{policy } \pi}{\text{maximize}} & \quad \rliminf \triangleq \liminf_{T\to\infty} \frac{1}{T} \sum_{t=0}^{T-1} \frac{1}{N} \sumN \E{r(S_t^\pi(i), A_t^\pi(i))} \\
    \text{subject to}  
    &\quad  \sumN A_t^\pi(i) = \alpha N,\quad \forall t\ge 0, \tag{\ref{eq:hard-budget-constraint}} \\
    \tag{\ref{eq:lp-single}}  \underset{\{y(s, a)\}_{s\in\sspa,a\in\aspa}}{\text{maximize}} \mspace{12mu}&\sum_{s\in\sspa,a\in\aspa} r(s, a) y(s, a) \\
    \text{subject to}\mspace{25mu}
    &\mspace{15mu}\sum_{s\in\sspa} y(s, 1) = \alpha, \tag{\ref{eq:expect-budget-constraint}}\\
    & \sum_{s'\in\sspa, a\in\aspa} y(s', a) P(s', a, s) = \sum_{a\in\aspa} y(s,a), \quad \forall s\in\sspa, \tag{\ref{eq:flow-balance-equation}}\\
    &\mspace{3mu}\sum_{s\in\sspa, a\in\aspa} y(s,a) = 1,  
    \quad
     y(s,a) \geq 0, \;\; \forall s\in\sspa, a\in\aspa.  \tag{\ref{eq:non-negative-constraint}}
\end{align}

Next, we show that the optimal value of \eqref{eq:lp-single} is an upper bound for that of \eqref{eq:N-arm-formulation}. 

\begin{lemma}[LP relaxation]\label{lem:lp-relaxation}
    Let $\rrel$ be the optimal value of the linear program \eqref{eq:lp-single}, and let $\ropt$ be the optimal reward of the $N$-armed restless bandit problem \eqref{eq:N-arm-formulation}. Then we have
    \begin{equation}
        \rrel \geq \ropt. 
    \end{equation}
\end{lemma}

\begin{proof}{\textit{Proof of \Cref{lem:lp-relaxation}}.}
    By standard results for MDP with finite state and action spaces, there always exists a stationary Markovian policy whose long-run average reward achieves the optimal reward $\ropt$ \citep[][Theorem~9.1.8]{Put_05}. Therefore, to show that $\ropt \leq \rrel$, it suffices to show that for each stationary Markovian policy $\pi$ and initial state vector $\veS_0$, we have $\rsysn \leq \rrel$. 
    
    Fixing an arbitrary stationary Markovian policy $\pi$ and initial state vector $\veS_0$, we define
    \[
        y^{\pi}(s,a) \triangleq \lim_{T\to\infty} \frac{1}{T} \sum_{t=0}^{T-1} \Ebigg{\frac{1}{N} \sumN \indibrac{S_t^\pi(i) = s, A_t^\pi(i) = a}} \quad \forall s\in\sspa, a\in\aspa.
    \]

    We first show that $\rsysn = \sumsa r(s,a)y^\pi(s,a)$.
    \begin{align}
        \sumsa r(s,a)y^\pi(s,a)
        \nonumber
        &= \sumsa r(s,a)\lim_{T\to\infty} \frac{1}{T} \sum_{t=0}^{T-1} \Ebigg{\frac{1}{N} \sumN \indibrac{S_t^\pi(i) = s, A_t^\pi(i) = a}} \\
        \nonumber 
        &= \lim_{T\to\infty} \frac{1}{NT}\sum_{t=0}^{T-1} \sumN \Ebigg{\sumsa r(s,a) \indibrac{S_t^\pi(i) = s, A_t^\pi(i) = a}} \\
        \nonumber
        &= \lim_{T\to\infty} \frac{1}{NT}\sum_{t=0}^{T-1} \sumN \E{r(S_t^\pi(i), A_t^\pi(i)} \\
        \nonumber
        &= \rsysn.
    \end{align}
    
    Then we show that $(y^\pi(s,a))_{s\in\sspa, a\in\aspa}$ satisfies the constraints of \eqref{eq:lp-single}.  
    We first consider the constraint \eqref{eq:expect-budget-constraint}: 
    \begin{align*}
        \sums y^\pi(s,1) 
        &= \sums \lim_{T\to\infty} \frac{1}{T} \sum_{t=0}^{T-1} \Ebigg{\frac{1}{N} \sumN \indibrac{S_t^\pi(i) = s, A_t^\pi(i) = 1}}  \\
        &= \lim_{T\to\infty} \frac{1}{NT}\sum_{t=0}^{T-1}\Ebigg{ \sumN \sums \indibrac{S_t^\pi(i) = s, A_t^\pi(i) = 1}} \\
        &= \lim_{T\to\infty} \frac{1}{NT}\sum_{t=0}^{T-1} \alpha N \\
        &= \alpha.
    \end{align*}
    Next, we look at the constraint \eqref{eq:flow-balance-equation}:
    \begin{align*}
        \sum_{s'\in\sspa, a\in\aspa} y^\pi(s',a) P(s',a,s)
        &= \lim_{T\to\infty} \frac{1}{NT}\sum_{t=0}^{T-1} \sumN \sumsa P(s',a,s) \Prob{S_t^\pi(i) = s', A_t^\pi(i) = a} \\
        &=  \lim_{T\to\infty} \frac{1}{NT}\sum_{t=0}^{T-1} \sumN \Prob{S_{t+1}^\pi(i) = s} \\
        &= \lim_{T\to\infty} \frac{1}{NT}\sum_{t=1}^{T} \sumN \Prob{S_t^\pi(i) = s} \\
        &= \sum_{a\in\aspa} y^\pi(s,a).
    \end{align*} 
    Finally, we consider the constraint \eqref{eq:non-negative-constraint}: 
    \begin{align*}
        \sum_{s'\in\sspa,a'\in\aspa} y^\pi(s',a')
        &= \lim_{T\to\infty} \frac{1}{NT}\sum_{t=0}^{T-1} \sumN \Ebigg{\sum_{s'\in\sspa,a'\in\aspa} \indibrac{S_t^\pi(i) = s', A_t^\pi(i) = a'}} = 1,
    \end{align*}
    and it is obvious that $y^\pi(s,a) \geq 0$ for any $s\in\sspa$ and $a\in\aspa$. 

    Combining the above argument, $(y^\pi(s,a))_{s\in\sspa, a\in\aspa}$ is a feasibile solution to \Cref{eq:lp-single}, so
    $\rsysn = \sumsa r(s,a)y^\pi(s,a) \leq \rrel$. Letting $\pi$ be the optimal stationary Markovian policy finishes the proof.
     \Halmos 
\end{proof}

\section{Proof of Theorem~\ref{thm:focus-set-policy} in general case}
\label{app:proof-meta-thm-general}

Recall that in \Cref{sec:meta-theorem}, we have proved \Cref{thm:focus-set-policy} assuming that the focus-set policy induces a Markov chain that converges to a unique stationary distribution. Here we provide the general proof without this simplifying assumption.

\begin{proof}{\textit{Proof of \Cref{thm:focus-set-policy} in the general case}.}
    Most steps in the general proof go through almost verbatim if we replace any steady-state expectations of the form $\E{f(S_\infty, A_\infty, X_\infty, D_\infty)}$ with the long-run averages of the form:
    \[
        \lim_{T\to\infty} \frac{1}{T} \sum_{t=0}^{T-1} \E{f(S_t, A_t, X_t, D_t)}.
    \]
    Note that the long-run averages of the above form always exist because $(S_t, A_t, X_t, D_t)$ is a finite-state Markov chain and Proposition 8.1.1 in \citep{Put_05} can be applied with a trivial generalization of its proof, although the values of the long-run averages could depend on the initial states. With the steady-state expectations replaced by the long-run averages, we get the following analogs of \eqref{eq:opt-gap-bound-1} and \eqref{eq:opt-gap-bdd-by-v}:
    \begin{align}
        \nonumber
        &\ropt - \rsysn \\
        \label{eq:opt-gap-bound-1-general}
        &\qquad\leq  \rmax \lim_{T\to\infty} \frac{1}{T} \sum_{t=0}^{T-1}\Big(\Ebig{\normbig{\statdist - X_t([N])}_1} + 2\rmax  \Ebig{1 - m(D_t)}\Big) +  \frac{2\rmax\constAction}{\sqrt{N}} \\
        \label{eq:opt-gap-bdd-by-v-general}
        &\qquad\leq \rmax \Big(\frac{1}{\constVnorm} + \frac{2}{\liph}\Big) \lim_{T\to\infty} \frac{1}{T} \sum_{t=0}^{T-1} \Ebig{V(X_t, D_{t})} + \frac{2\rmax\constAction}{\sqrt{N}}. 
    \end{align}

    The only place that needs a different treatment in the general case is in the last few steps, after deriving the drift condition for each finite $t$:
    \begin{align}\tag{\ref{eq:thm4-proof:drift-condition} restated}
        \Ebig{V(X_{t+1}, D_{t+1}) \givenbig X_t, D_t}  \leq \rhoFinal V(X_t, D_t) + \frac{K_1}{\sqrt{N}}.
    \end{align}
    We take expectation on both sides of \eqref{eq:thm4-proof:drift-condition} to get the recursive inequality on $\E{V(X_{t}, D_{t})}$: 
    \begin{equation*}
        \E{V(X_{t+1}, D_{t+1})}  \leq \rhoFinal \E{V(X_t, D_t)} + \frac{K_1}{\sqrt{N}}.
    \end{equation*}
    We expand the recursion to get
    \begin{align*}
        \E{V(X_{t}, D_{t})} 
        &\leq \rhoFinal^t \, \E{V(X_0, D_0)} + \frac{K_1}{(1-\rhoFinal)\sqrt{N}} \\
        \frac{1}{T} \sumt \E{V(X_{t}, D_{t})}
        &\leq \frac{1}{(1-\rhoFinal)T} \E{V(X_0, D_0)} + \frac{K_1}{(1-\rhoFinal)\sqrt{N}}.
    \end{align*}
    Therefore, the long-run average of $\E{V(X_{t}, D_{t})}$ can be bounded as 
    \begin{equation}
        \label{eq:thm4full:proof_meta_1}
        \lim_{T\to\infty} \frac{1}{T} \sumt \E{V(X_t, D_t)} \leq \frac{K_1}{(1-\rhoFinal)\sqrt{N}}.
    \end{equation}
    Combining \eqref{eq:thm4full:proof_meta_1} with \eqref{eq:opt-gap-bdd-by-v-general}, we finish the proof. 
\end{proof}

\begin{remark}   
    The proof of \Cref{thm:focus-set-policy} also yields finite-time performance bounds. Under the same conditions, we can bound the difference between the long-run optimal reward $\ropt$ and two finite-time performance metrics: the expected reward at time $T$, and the average expected reward over the first $T$ time steps. These bounds are given by:
    \begin{align}
        \ropt - \E{R(S_T, A_T)} &\leq \rmax \Big(\frac{1}{\constVnorm} + \frac{2}{\liph}\Big) \E{V(X_0, D_0)}\rhoFinal^T  + \frac{\constfs}{\sqrt{N}}  \\
        \ropt - \frac{1}{T}\sum_{t=0}^T \E{R(S_t, A_t)} &\leq \rmax \Big(\frac{1}{\constVnorm} + \frac{2}{\liph}\Big) \frac{\E{V(X_0, D_0)}}{(1-\rhoFinal)T}  + \frac{\constfs}{\sqrt{N}},
    \end{align}
    where $\constfs = \rmax \left(\Big(\frac{1}{\constVnorm} + \frac{2}{\liph}\Big) \frac{K_1}{1-\rhoFinal} + 2\constAction\right)$. 
    As $T$ goes to infinity, these two upper bounds converge to the upper bound in \Cref{thm:focus-set-policy} exponentially fast in $T$ or at the rate $1/T$, respectively. 
\end{remark}

\section{Supplementary lemmas and proofs}
In this appendix, we provide lemmas and proofs that serve as preliminaries for analyzing our policies. 
In Appendix~\ref{app:proof-norm-lemmas}, we show that the weight matrix $\wmat$ in \Cref{def:w-and-w-norm} is well-defined and prove \Cref{lem:pibar-one-step-contraction-W},  which claims that the state distribution of the Markov chain $P_\pibs$ converges to the steady-state distribution $\statdist$ geometrically fast under the $\wmat$-weighted $L_2$ norm. 
Then in Appendix~\ref{app:proof-feature-lyapunov-lemmas}, we show that two classes of functions, $\{\hw(x, D))\}_{D\subseteq [N]}$ and $\{\hid(\cdot, m)\}_{m\in[0,1]_N}$, are subset Lyapunov functions. 
Finally, in Appendix~\ref{app:proof-ell-1}, we prove two lemmas about the $L_1$ norm that are used when analyzing the set-expansion policy.

\subsection{Lemmas and proofs about matrix $\wmat$ and $\wmat$-weighted $L_2$ norm}\label{app:proof-norm-lemmas}

For ease of reference, we first restate \Cref{def:w-and-w-norm} below.

\wdef*

The next lemma shows that the matrix $\wmat$ is well-defined and positive definite, with eigenvalues in the range $[1, \lamw]$. 

\begin{lemma}\label{lem:W-well-defined}
    The matrix $\wmat$ given in \Cref{def:w-and-w-norm} is well-defined. Moreover, $\wmat$ is positive definite with eigenvalues lower bounded by $1$. 
\end{lemma}

\begin{proof}{\textit{Proof of \Cref{lem:W-well-defined}}.}
    Consider the sum of the spectral norm of all terms in the definition of $W$:
    \[
        \sum_{k=0}^\infty \norm{(P_\pibs - \Xi)^k (P_\pibs^{\top} - \Xi^{\top})^k}_2.
    \] 
    Note that $(P_\pibs - \Xi)^k = P_\pibs^k - \Xi$. Because $\pibs$ induces an aperiodic unichain, $P_\pibs^k \to \Xi$ as $k\to\infty$. 
    Consequently, there exist $k_0 \in \mathbb{N}^+$ and $\rhobar < 1$ such that $\norm{(P_\pibs-\Xi)^{k_0}}_2 = \rhobar$. Then we have
    \begin{align*}
        &\sum_{k=0}^\infty \norm{(P_\pibs - \Xi)^k (P_\pibs^{\top} - \Xi^{\top})^k}_2 
        = \sum_{j=0}^\infty \sum_{k=j k_0}^{(j+1)k_0-1} \norm{(P_\pibs - \Xi)^k (P_\pibs^{\top} - \Xi^{\top})^k}_2 \\
        &\qquad= \sum_{j=0}^\infty \sum_{k=0}^{k_0-1} \norm{(P_\pibs-\Xi)^{jk_0}(P_\pibs - \Xi)^k (P_\pibs^{\top} - \Xi^{\top})^k(P_\pibs^{\top}-\Xi^{\top})^{jk_0}}_2 \\
        &\qquad\leq \sum_{j=0}^\infty \sum_{k=0}^{k_0-1} \norm{(P_\pibs-\Xi)^{k_0}}_2^j\norm{(P_\pibs - \Xi)^k (P_\pibs^{\top} - \Xi^{\top})^k}_2\norm{(P_\pibs^{\top}-\Xi^{\top})^{k_0}}_2^j \\
        &\qquad= \sum_{j=0}^\infty \rhobar^{2j} \sum_{k=0}^{k_0-1} \norm{(P_\pibs - \Xi)^k (P_\pibs^{\top} - \Xi^{\top})^k}_2 \\
        &\qquad= \frac{C_0}{1-\rhobar^2} < \infty,
    \end{align*}
    where $C_0 = \sum_{k=0}^{k_0-1} \norm{(P_\pibs - \Xi)^k (P_\pibs^{\top} - \Xi^{\top})^k}_2$. 
    Therefore, the infinite sum is absolutely convergent.

    To show that $\wmat$ is positive definite, observe that each term in its definition, $(P_\pibs - \Xi)^k (P_\pibs^{\top} - \Xi^{\top})^k$, is positive semi-definite; and its first term is the identity matrix. Therefore, for any row vector $v\in \R^{\abs{\sspa}}$ such that $v\neq 0$, $v\wmat v^{\top} \geq vv^{\top}$. Therefore, $\wmat$ is positive definite and its eigenvalues are lower bounded by $1$. 
    \Halmos
\end{proof}

Next, we restate and prove \Cref{lem:pibar-one-step-contraction-W}.

\wnorm*

\begin{proof}{\textit{Proof of \Cref{lem:pibar-one-step-contraction-W}}.}
    Recall that by \Cref{lem:W-well-defined}, the eigenvalues of $\wmat$ lie in the range $[1, \lamw]$, where $\lamw$ is the maximal eigenvalue of $\wmat$.  
    It is not hard to see from the definition that $\wmat$ satisfies the equation
    \begin{equation*}
        (P_\pibs - \Xi) \wmat (P_\pibs^{\top} - \Xi^{\top}) - \wmat + I = 0.
    \end{equation*}
    It follows that
    \begin{align}
        \nonumber
        &\norm{(v - \statdist) P_\pibs}_\wmat - \norm{v-\statdist}_\wmat \\
        &\qquad\leq \frac{(v-\statdist) P_\pibs \wmat P_\pibs^{\top} (v-\statdist)^{\top} - (v-\statdist) \wmat (v-\statdist)^{\top}}{2\norm{v-\statdist}_\wmat} \nonumber \\
        &\qquad= \frac{(v-\statdist) (P_\pibs-\Xi) \wmat (P_\pibs-\Xi)^{\top} (v-\statdist)^{\top} - (v-\statdist) \wmat (v-\statdist)^{\top}}{2\norm{v-\statdist}_\wmat} \nonumber \\
        &\qquad= \frac{(v-\statdist) (\wmat-I) (v-\statdist)^{\top} - (v-\statdist) \wmat (v-\statdist)^{\top}}{2\norm{v-\statdist}_\wmat} \nonumber \\
        &\qquad= -\frac{\norm{v-\statdist}_2^2}{2\norm{v-\statdist}_\wmat}, \label{eq:prove-w-contract:intermediate-1}
    \end{align}
    where the inequality is because $x - y \leq (x^2 - y^2) / (2y)$ for any positive numbers $x,y$. 
    To change the norm in the numerator of \eqref{eq:prove-w-contract:intermediate-1} to the $\wmat$-weighted $L_2$ norm, we use the following bound: 
    \[
        \norm{v-\statdist}_\wmat^2 = (v-\statdist)\wmat(v-\statdist)^{\top} \leq \lamw \norm{v-\statdist}_2^2.
    \]
    Therefore, we have
    \begin{equation}
        \norm{(v - \statdist)P_\pibs}_\wmat - \norm{v-\statdist}_\wmat  \leq -\frac{1}{2\lamw} \norm{v-\statdist}_\wmat. \nonumber 
    \end{equation}
    This implies \eqref{eq:pibar-contraction} because $\statdist P_\pibs = \statdist$. 
    \Halmos
\end{proof}

\subsection{Lemmas and proofs about subset Lyapunov functions}\label{app:proof-feature-lyapunov-lemmas}

In this section, we consider two classes of functions, $\{\hw(x, D)\}_{D\subseteq[N]}$ and $\{\hid(x, m)\}_{m\in[0,1]_N}$, defined in \Cref{sec:pf-set-exp:subset-lyapunov} and \Cref{sec:proof-id-policy:feature-lyapunov} respectively. 
Then we prove Lemmas~\ref{lem:hw-feature-lyaupnov} and \ref{lem:hid-feature-lyaupnov}, which verify that these two classes of functions satisfy the definition of subset Lyapunov functions.

For any system state $x$ and subset $D\subseteq[N]$, recall that $\hw(x, D)$ is defined as 
\begin{equation*}
    \hw(x, D) = \norm{x(D) - m(D) \statdist}_W,
\end{equation*}
where $\wmat$ is the matrix in \Cref{def:w-and-w-norm}; $\norm{u}_\wmat = \sqrt{u W u^\top}$ for any row vector $u$.

The lemma restated below shows that $\{\hw(x, D)\}_{D\subseteq[N]}$ are subset Lyapunov functions. 

\hwfeature*

\begin{proof}{\textit{Proof of \Cref{lem:hw-feature-lyaupnov}}.}
    We first prove the drift condition \eqref{eq:hw-feature-lyapunov:drift}. 
    Let $X_1'$ be the system state after one step of transition if $A_0(i)$ is an independent sample from $\pibs(\cdot | S_0(i))$ for each $i\in D$. Then
    \begin{align}
        \nonumber
        &h_\wmat(X_1', D) - \big(1-\frac{1}{2\lamw}\big) h_\wmat(x, D)  \\
        &\qquad= \norm{X_1'(D) - m(D) \statdist}_\wmat - \big(1-\frac{1}{2\lamw}\big) \norm{x(D) - m(D)\statdist}_\wmat  \nonumber \\
        &\qquad\leq \lVert X_1'(D) - m(D) \statdist\rVert_\wmat - \norm{x(D)P_\pibs - m(D) \statdist}_\wmat \nonumber \\
        \label{eq:pf-hw-drift:noise-bdd-by-w-norm}
        &\qquad \leq \lVert X_1'(D) - x(D) P_{\pibs}\rVert_\wmat. 
    \end{align}
    where the first inequality follows from applying \Cref{lem:pibar-one-step-contraction-W} with $v = x(D) / m(D)$; the second inequality is due to the triangle inequality. 
    To bound the expression in \eqref{eq:pf-hw-drift:noise-bdd-by-w-norm}, we define the centered random vector $\zmNoise(i) = X_1'(\{i\})- x(\{i\})P_\pibs$ for each $i \in D$, and denote its $s$-th entry as $\zmNoise(i, s)$ for $s\in\sspa$. This allows us to rewrite $\normplain{X_1'(D) - x(D) P_{\pibs}}_\wmat$ as 
    \begin{equation}\label{eq:pf-hw-drift:discretize-noise-partial-sums-w}
        \normBig{X_1'(D) - x(D) P_{\pibs}}_\wmat = \normBig{\sum_{i\in D} \zmNoise(i)}_\wmat. 
    \end{equation}
    Because $X_1'$ is the system state after each arm in $D$ independently follows $\pibs$, conditioned on $x$, $\zmNoise(i)$'s are independent across $i$ and have zero means on each entry, i.e., $\E{\zmNoise(i, s) \givenplain X_0=x} = 0$ for each $i\in D$ and $s\in\sspa$. 
    We can thus bound the conditional expectation of $\normplain{\sum_{i\in D} \zmNoise(i)}_\wmat^2$ as follows:
    \begin{align}
        \nonumber
        \Ebigg{\normBig{\sum_{i\in D} \zmNoise(i)}_\wmat^2 \givenbigg X_0 = x} 
        &\leq  \lamw\Ebigg{\normBig{\sum_{i\in D} \zmNoise(i)}_2^2 \givenbigg X_0 = x}  \nonumber \\
        &=  \lamw\Ebigg{\sums\Big(\sum_{i\in D} \zmNoise(i,s)^2 + 2\sum_{0\leq i<i'\leq \Md{x}-1} \zmNoise(i,s)\zmNoise(i',s) \Big)\givenbigg X_0 = x} \nonumber \\
        &= \lamw \sums \sum_{i\in D} \E{\zmNoise(i, s)^2 \givenBig X_0 = x} \nonumber \\
        &\leq \lamw \sum_{i\in D} \Ebigg{\Big(\sums \abs{\zmNoise(i, s)}\Big)^2 \givenbigg X_0 = x}  \nonumber \\
        \label{eq:pf-hw-drift:w-norm-iid}
        &\leq \frac{4 \lamw}{N},
    \end{align}
    where the first inequality uses from the fact that $\norm{v}_\wmat \leq \lamw^{1/2} \norm{v}_2$ for any $v\in\R^{|\sspa|}$; the first equality is by the definition of $\norm{\cdot}_2$ on $\R^{\abs{\sspa}}$; the second equality is because $\zmNoise(i, s)$'s are independent across $i\in D$ and have zero means; the last inequality uses the fact that $\sums \abs{\zmNoise(i,s)} = \normplain{\zmNoise(i)}_1 \leq \normplain{X_1'(\{i\})}_1 +  \normplain{x(\{i\})P_\pibs}_1 = 2/N$. By the Cauchy-Schwarz inequality, it follows from \eqref{eq:pf-hw-drift:w-norm-iid} that 
    \begin{equation}\label{eq:pf-hw-drift:w-norm-iid-no-square}
        \Ebigg{\normBig{\sum_{i\in D} \zmNoise(i)}_\wmat \givenbigg X_0 = x}  \leq \Ebigg{\normBig{\sum_{i\in D} \zmNoise(i)}_\wmat^2 \givenbigg X_0 = x}^{1/2} \leq \frac{2\lamw^{1/2}}{\sqrt{N}}.
    \end{equation}

    Therefore, by combining the above calculations, we get
    \begin{align}
          \nonumber
          &\EBig{ h_\wmat(X_1', D) - \big(1-\frac{1}{2\lamw}\big) h_\wmat(x, D) \givenBig X_0 = x}  \\
          &\qquad\leq \EBig{\lVert X_1'(D) - x(D) P_{\pibs}\rVert_\wmat \givenBig X_0=x} \nonumber \\
          &\qquad=  \Ebigg{\normBig{\sum_{i\in D} \zmNoise(i)}_\wmat \givenbigg X_0 = x} \nonumber \\
          &\qquad\leq \frac{2\lamw^{1/2}}{\sqrt{N}}, \nonumber
    \end{align}
    which implies the drift condition \eqref{eq:hw-feature-lyapunov:drift}.

    Next, we prove the distance dominance property \eqref{eq:hw-feature-lyapunov:strength}. Because the eigenvalues of $\wmat$ are at least $1$, we have
    \[
        \hw(x, D) = \norm{x(D) - m(D) \statdist}_\wmat \geq \norm{x(D) - m(D) \statdist}_2 \geq \frac{1}{|\sspa|^{1/2}} \norm{x(D) - m(D) \statdist}_1. 
    \]
    
    Finally, we show the Lipschitz continuity \eqref{eq:hw-feature-lyapunov:lipschitz}.  Note that
    \begin{align*}
        &\absbig{\hw(x, D) - \hw(x, D')} \\
        &\qquad= \absbig{\norm{x(D) - m(D)\statdist}_\wmat - \norm{x(D') - m(D)'\statdist}_\wmat } \\
        &\qquad\leq \norm{x(D) - m(D)\statdist - x(D') + m(D')\statdist}_\wmat \\
        &\qquad= \norm{x(D'\backslash D) - m(D'\backslash D)\statdist}_\wmat \\
        &\qquad\leq \norm{x(D'\backslash D)}_\wmat + m(D'\backslash D) \norm{\statdist}_\wmat.
    \end{align*}
    Note that for any $v \in \R^{|\sspa|}$, $\norm{v}_\wmat \leq \lamw^{1/2} \norm{v}_2 \leq \lamw^{1/2}  \norm{v}_1$. Because $\norm{x(D'\backslash D)}_1 = m(D') - m(D)$, and $\norm{\statdist}_1 = 1$, we have
    \[
        \norm{x(D'\backslash D)}_\wmat + m(D'\backslash D) \norm{\statdist}_\wmat \leq 2\lamw^{1/2} (m(D') - m(D)).   
    \]
    This proves the Lipschitz continuity \eqref{eq:hw-feature-lyapunov:lipschitz}.  \Halmos
\end{proof}

Recall that for any system state $x$ and $m\in[0,1]_N$, $\hid(x, m)$ is given by
\begin{equation*}
    \hid(x, m) =  \max_{\substack{m'\in [0,1]_N\\m'\leq m}} \hw(x, m'), 
\end{equation*}
where $\hw(x, m') = \norm{x([Nm']) - m'\statdist}_W$. 
Next, we restate and prove \Cref{lem:hid-feature-lyaupnov}, which 
implies that $\{\hid(x, m)\}_{m\in [0,1]_N}$ are also subset Lyapunov functions satisfying \Cref{def:feature-lyapunov}.

\hidfeature*

\begin{proof}{\textit{Proof}.}
    In this proof, the variable $n$ is a non-negative integer unless otherwise stated. 
    We first prove the drift condition \eqref{eq:hid-feature-lyapunov:drift}. 
    Let $X_1'$ be the system state after one step of transition if $A_0(i)$ is an independent sample from $\pibs(\cdot | S_0(i))$ for each $i\in [Nm]$. Then
    \begin{align}
        \nonumber
        &\hid(X_1', m) - \big(1-\frac{1}{2\lamw}\big) \hid(x, m) \\
        &\qquad= \max_{\substack{m'\in [0,1]_N, m'\leq m}} \hw(X_1', m') - \big(1-\frac{1}{2\lamw}\big) \max_{\substack{m'\in [0,1]_N, m'\leq m}} \hw(x, m')   \nonumber \\
        &\qquad\leq \max_{\substack{m'\in [0,1]_N, m'\leq m}} \Big(\hw(X_1', m') -\big(1-\frac{1}{2\lamw}\big)  \hw(x, m')\Big)  \nonumber \\
        \label{eq:pf-hid-drift:intermediate-0}
        &\qquad\leq \max_{\substack{m'\in [0,1]_N, m'\leq m}}  \lVert X_1'([Nm']) - x([Nm']) P_{\pibs}\rVert_\wmat,
    \end{align}
    where the last inequality can be justified using the same argument as \eqref{eq:pf-hw-drift:noise-bdd-by-w-norm}. 
    Therefore,
    \begin{equation}\label{eq:pf-hid-drift:intermediate-1}
        \Big(\hid(X_1', m) - \big(1-\frac{1}{2\lamw}\big) \hid(x, m)\Big)^+ \leq \max_{\substack{m'\in [0,1]_N, m'\leq m}}  \lVert X_1'([Nm']) - x([Nm']) P_{\pibs}\rVert_\wmat.
    \end{equation}
    For any $i \in [Nm]$, we define the centered random vector $\zmNoise(i) \in \R^{\abs{\sspa}}$ as 
    \[
        \zmNoise(i) = X_1'(\{i\})- x(\{i\})P_\pibs,
    \]
    and denote its $s$-th entry as $\zmNoise(i, s)$ for $s\in\sspa$. 
    This allows us to rewrite the right-hand side of \eqref{eq:pf-hid-drift:intermediate-1} as
    \begin{equation}\label{eq:pf-hid-drift:discretize-noise-partial-sums-id}
        \max_{\substack{m'\in [0,1]_N, m'\leq m}} \normBig{X_1'(\setnm{m'}) - x(\setnm{m'}) P_{\pibs}}_\wmat = \max_{n\leq Nm} \normBig{\sum_{i\in[n]} \zmNoise(i)}_\wmat.
    \end{equation}
    Therefore, to prove the bound in \eqref{eq:hid-feature-lyapunov:drift}, it suffices to show that 
    \begin{equation}
        \Ebigg{ \max_{n\leq Nm} \normBig{\sum_{i\in[n]} \zmNoise(i)}_\wmat \givenbigg X_0 = x}  
         \leq  \frac{4\lamw^{1/2}}{\sqrt{N}}.
    \end{equation}
    
    Conditioned on $X_0 = x$, we argue that $\norm{\sum_{i\in[n]} \zmNoise(i)}_\wmat$ is a submartingale in $n$ so that we can invoke Doob's $L_2$ maximal inequality to bound the right-hand side of \eqref{eq:pf-hid-drift:discretize-noise-partial-sums-id} (see, e.g., Theorem 4.4.4 of \citep{Dur_19_prob_book}). Observe that 
    \begin{itemize}
        \item $\zmNoise(i)$'s are independent across $i\in [Nm]$;
        \item For each $i\in [Nm]$ and $s\in\sspa$, $\E{\zmNoise(i, s) \givenplain X_0=x} = 0$.
    \end{itemize}
    Therefore, $\sum_{i\in[n]} \zmNoise(i)$ is a martingale in $n$. Because $\norm{\cdot}_\wmat$ is a convex function, $\norm{\sum_{i\in[n]} \zmNoise(i)}_\wmat$ is a submartingale in $n$.  We apply Doob's $L_2$ maximal inequality to $\norm{\sum_{i\in[n]} \zmNoise(i)}_\wmat$ to get
    \begin{equation}\label{eq:pf-hid-drift:apply-doob-maximal}
        \Ebigg{\Big(\max_{n\leq Nm} \normBig{\sum_{i\in[n]} \zmNoise(i)}_\wmat\Big)^2 \givenbigg X_0 = x} \leq 4 \Ebigg{\normBig{\sum_{i\in[Nm]} \zmNoise(i)}_\wmat^2 \givenbigg X_0 = x}. 
    \end{equation}
    Applying Cauchy-Schwarz inequality to the LHS of \eqref{eq:pf-hid-drift:apply-doob-maximal}, we get 
    \begin{align}
        \Ebigg{\max_{n\leq Nm} \normBig{\sum_{i\in[n]} \zmNoise(i)}_\wmat \givenbigg X_0 = x} 
        \label{eq:pf-hid-drift:cauchy-schwartz}
        &\leq \Ebigg{\Big(\max_{n\leq Nm} \normBig{\sum_{i\in[n]} \zmNoise(i)}_\wmat\Big)^2 \givenbigg X_0 = x}^{1/2} 
    \end{align}
    Using the same argument as \eqref{eq:pf-hw-drift:w-norm-iid} with $D = [Nm]$, we bound the right-hand side of \eqref{eq:pf-hid-drift:apply-doob-maximal} as
    \begin{equation}\label{eq:pf-hid-drift:w-norm-iid}
        4\Ebigg{\normBig{\sum_{i\in[Nm]} \zmNoise(i)}_\wmat^2 \givenbigg X_0 = x} \leq \frac{16\lamw}{N}.
    \end{equation}
    Plugging \eqref{eq:pf-hid-drift:cauchy-schwartz} and \eqref{eq:pf-hid-drift:w-norm-iid} into two sides of \eqref{eq:pf-hid-drift:apply-doob-maximal}, we get 
    \begin{equation}
    \label{eq:pf-hid-drift:intermediate-2}
        \Ebigg{\max_{n\leq Nm} \normBig{\sum_{i\in[n]} \zmNoise(i)}_\wmat \givenbigg X_0 = x} \leq \frac{4\lamw^{1/2}}{\sqrt{N}}.
    \end{equation}
    This proves the drift condition \eqref{eq:hid-feature-lyapunov:drift}.

    Next, we show the distance dominance property \eqref{eq:hid-feature-lyapunov:strength}. By the definition of $\hid(x, m)$ and the fact that the eigenvalues of $\wmat$ are at least $1$, we have
    \[
        \hid(x, m) \geq \norm{x([Nm]) - m \statdist}_\wmat \geq \norm{x([Nm]) - m \statdist}_2 \geq \frac{1}{|\sspa|^{1/2}} \norm{x([Nm]) - m \statdist}_1.
    \]

    Finally, we show the Lipschitz continuity \eqref{eq:hid-feature-lyapunov:lipschitz}. For simplicity, we omit $m\in[0,1]_N$ in the subscripts. 
    Consider any $m, m' \in [0,1]_N$. Without loss of generality, we assume that $m\leq m'$. 
    We expand the definitions of $\hid(x, m)$ and $\hid(x, m')$ to rewrite them as follows: 
    \begin{align*}
        \hid(x, m) &= \max\big\{\hid(x, m), \hw(x, m)\big\} \\
        \hid(x, m') &= \max\Big\{\hid(x, m), \max_{m''\in [m,m']_N}\hw(x,m'')\Big\},
    \end{align*}
    where $[m,m']_N$ is a shorthand for $[m,m']\cap [0,1]_N$. 
    Observe that for any $a,b,c\in\R$, we have $\absplain{\max\{a,b\} - \max\{a,c\}} \leq \absplain{b-c}$. Letting $a = \hid(x, m)$, $b = \hw(x, m)$, and $c = \max_{m''\in [m,m']_N}\hw(x,m'')$, we get 
    \begin{equation}\label{eq:pf-hid-lipschitz:intermediate-1}
        \absBig{\hid(x, m) - \hid(x,m')} \leq \absBig{\max_{m''\in [m,m']_N}\hw(x,m'')- \hw(x, m)}.
    \end{equation}
    We further bound the right-hand side of \eqref{eq:pf-hid-lipschitz:intermediate-1} as 
    \begin{align}
        \nonumber
        &\absBig{\max_{m''\in [m,m']_N}\hw(x,m'')- \hw(x, m)} \\
        &\qquad\leq \max_{m''\in[m,m']_N} \absplain{\hw(x,m'') - \hw(x,m)} \nonumber \\
        &\qquad\leq  \max_{m''\in[m,m']_N} 2\lamw^{1/2} \absplain{m''-m} \nonumber\\
        \label{eq:pf-hid-lipschitz:intermediate-2}
        &\qquad=  2\lamw^{1/2} \absplain{m'-m}, 
    \end{align}
    where in the second inequality we used \eqref{eq:hw-feature-lyapunov:lipschitz}, the Lipschitz continuity of $\hw(x, D)$ in $D$ that we have proved in \Cref{lem:hw-feature-lyaupnov}. 
    Combining \eqref{eq:pf-hid-lipschitz:intermediate-1} and \eqref{eq:pf-hid-lipschitz:intermediate-2}, we prove the Lipschitz continuity \eqref{eq:hid-feature-lyapunov:lipschitz}. 
    \Halmos
\end{proof}

\subsection{Lemmas and proofs about $L_1$ norm}\label{app:proof-ell-1}
In this subsection, we prove two lemmas about the $L_1$ norm that are useful for the analysis of the set-expansion policy, considering that the set-expansion policy chooses focus sets based on the slack $\slk(x, D)$ whose definition involves $L_1$ norm. 

We first show that if the optimal single-armed policy $\pibs$ induces an aperiodic unichain, right-multiplying $P_\pibs$ is non-expansive under the $L_1$ norm.

\begin{lemma}[Non-expansiveness of $P_\pibs$ under $L_1$ norm]\label{lem:ell-1-non-expansive}
    For any distribution $v\in \Delta(\sspa)$, 
    \begin{equation}
        \norm{(v  - \statdist)P_\pibs}_1 \leq \norm{v - \statdist}_1. 
    \end{equation}
\end{lemma}

\begin{proof}{\textit{Proof}.}
    For any $v \in \Delta(\sspa)$, 
    \begin{align*}
        \norm{(v-\statdist) P_\pibs}_1  &= \sum_{s'\in\sspa} \abs{\sum_{s\in\sspa} (v(s)-\statdist(s)) P_\pibs(s,s')} \\
        &\leq \sum_{s'\in\sspa} \sum_{s\in\sspa} \abs{v(s)-\statdist(s)} P_\pibs(s,s') \\
        &=  \sum_{s\in\sspa} \abs{v(s)-\statdist(s)} \sum_{s'\in\sspa}  P_\pibs(s,s') \\
        &=  \sum_{s\in\sspa} \abs{v(s)-\statdist(s)} \\
        &= \norm{v-\statdist}_1.    \Halmos
    \end{align*}
\end{proof}

Next, we show that if all arms in a subset $D$ follow $\pibs$, the $L_1$ distance between the scaled state-count vector $X_t(D)$ and the scaled optimal steady-state distribution $m(D)\statdist$ does not significantly increase.

\begin{lemma}\label{lem:ell-1-drift}
    For any system state $x$ and any subset $D\subseteq[N]$, 
    \begin{equation}
        \label{eq:ell-1:drift}
        \Ebig{ \big(\normplain{X_1(D) - m(D) \statdist}_1 - \normplain{x(D) - m(D) \statdist}_1\big)^+ \;\givenbig\; X_0 = x, A_0(i)\sim \pibs(\cdot| S_0(i)) \forall i\in D } 
        \leq  \frac{2|\sspa|^{1/2}}{\sqrt{N}}. 
    \end{equation}
\end{lemma}
    
\begin{proof}{\textit{Proof}.}
     Let $X_1'$ be the system state after one step of transition if $A_0(i) \sim \pibs(\cdot | S_0(i))$ for any $i\in D$. Then
    \begin{align}
        &\norm{X_1'(D) - m(D) \statdist}_1 - \norm{x(D) - m(D)\statdist}_1  \nonumber \\
        &\qquad \leq \lVert X_1'(D) - m(D) \statdist\rVert_1 - \norm{x(D)P_\pibs - m(D) \statdist}_1 \nonumber \\
        &\qquad \leq \lVert X_1'(D) - x(D) P_{\pibs}\rVert_1,
    \end{align}
    where the first inequality follows from applying \Cref{lem:ell-1-non-expansive} with $v = x(D) / m(D)$; the second inequality is due to the triangular inequality. Therefore, 
    \begin{equation}
    \label{eq:pf-ell-1-drift:noise-bdd-by-w-norm}
        \big(\norm{X_1'(D) - m(D) \statdist}_1 - \norm{x(D) - m(D)\statdist}_1 \big)^+ \leq \lVert X_1'(D) - x(D) P_{\pibs}\rVert_1. 
    \end{equation}

    For any $i \in [\Md{x}]$, we define the centered random vector $\zmNoise(i) \in \R^{\abs{\sspa}}$ as 
    \[
        \zmNoise(i) = X_1'(\{i\})- x(\{i\})P_\pibs,
    \]
    and denote its $s$-th entry as $\zmNoise(i, s)$ for $s\in\sspa$. This allows us to rewrite $\normplain{X_1'(D) - x(D) P_{\pibs}}_1$ as 
    \begin{equation}\label{eq:pf-ell-1-drift:discretize-noise-partial-sums-ell-1}
        \normBig{X_1'(D) - x(D) P_{\pibs}}_1 = \normBig{\sum_{i\in D} \zmNoise(i)}_1. 
    \end{equation}
    Observe that conditioned on $X_0 = x$, we have the following facts about these centered random vectors:
    \begin{itemize}
        \item $\zmNoise(i)$'s are independent across $i\in D$;
        \item For each $i\in D$ and $s\in\sspa$, $\E{\zmNoise(i, s)\givenplain X_0=x} = 0$.
    \end{itemize}
    Therefore, we can bound the conditional expectation of $\normplain{\sum_{i\in D} \zmNoise(i)}_1^2$ as follows:
    \begin{align}
        \nonumber
        \Ebigg{\normBig{\sum_{i\in D} \zmNoise(i)}_1^2 \givenbigg X_0 = x} 
        &\leq  |\sspa| \, \Ebigg{\normBig{\sum_{i\in D} \zmNoise(i)}_2^2 \givenbigg X_0 = x}  \nonumber \\
        &=   |\sspa| \, \Ebigg{\sums\Big(\sum_{i\in D} \zmNoise(i,s)^2 + 2\sum_{0\leq i<i'\leq \Md{x}-1} \zmNoise(i,s)\zmNoise(i',s) \Big)\givenbigg X_0 = x} \nonumber \\
        &=  |\sspa| \, \sums \sum_{i\in D} \E{\zmNoise(i, s)^2 \givenBig X_0 = x} \nonumber \\
        &\leq  |\sspa| \, \sum_{i\in D} \Ebigg{\Big(\sums \abs{\zmNoise(i, s)}\Big)^2 \givenbigg X_0 = x}  \nonumber \\
        \label{eq:pf-ell-1-drift:w-norm-iid}
        &\leq \frac{4 |\sspa|}{N},
    \end{align}
    where the first inequality uses from the fact that $\norm{v}_1 \leq  |\sspa|^{1/2} \norm{v}_2$ for any $v\in\R^{|\sspa|}$; the first equality is by the definition of $\norm{\cdot}_2$ on $\R^{\abs{\sspa}}$; the second equality is because $\zmNoise(i, s)$'s are independent across $i\in D$ and have zero means; the last inequality uses the fact that $\sums \abs{\zmNoise(i,s)} = \normplain{\zmNoise(i)}_1 \leq \normplain{X_1'(\{i\})}_1 +  \normplain{x(\{i\})P_\pibs}_1 = 2/N$. By the Cauchy-Schwarz inequality, it follows from \eqref{eq:pf-ell-1-drift:w-norm-iid} that 
    \begin{equation}\label{eq:pf-ell-1-drift:ell-1-norm-iid-no-square}
        \Ebigg{\normBig{\sum_{i\in D} \zmNoise(i)}_1 \givenbigg X_0 = x}  \leq \EBig{\normBig{\sum_{i\in D} \zmNoise(i)}_1^2 \givenbigg X_0 = x}^{1/2} \leq \frac{2|\sspa|^{1/2}}{\sqrt{N}}.
    \end{equation}

    Combining the above calculations, we get 
    \begin{align*}
        &\Ebig{ \big(\norm{X_1'(D) - m(D) \statdist}_1 - \norm{x(D) - m(D)\statdist}_1 \big)^+ \givenbig X_0 = x}  \\
        &\qquad\leq \Ebig{\lVert X_1'(D) - x(D) P_{\pibs}\rVert_1 \givenbig X_0 = x } \\
        &\qquad= \Ebigg{\normBig{\sum_{i\in D} \zmNoise(i)}_1 \givenbigg X_0 = x}  \\
        &\qquad\leq \frac{2|\sspa|^{1/2}}{\sqrt{N}},
    \end{align*}
    which implies \eqref{eq:ell-1:drift}. 
    \Halmos
\end{proof}

\section{Set-optimization policy}
\label{app:set-optimization-policy}
This appendix introduces the \emph{set-optimization policy}, another focus-set policy that also achieves an $O(1/\sqrt{N})$ optimality gap.  
The key distinguishing feature of the set-optimization policy is that it is Markovian with respect to the scaled state-count vector $X_t([N])$, meaning that $X_t([N])$ forms a Markov chain under this policy --- a property not shared by the set-expansion policy (\Cref{alg:set-expansion}) or the ID policy (\Cref{alg:id}). Although the advantages of the set-optimization policy over the other two focus-set policies are not evident as it does not perform better in simulations, we include it here for its potential theoretical interests.

\subsection{Definition and asymptotic optimality of set-optimization policy}

The set-optimization policy is similar to the set-expansion policy in that they both choose a focus set $D_t$ in each time step and give priority to arms in $D_t$ to follow their ideal actions. 
However, they differ in how $D_t$ is chosen. In the set-optimization policy, $D_t$ is updated by solving an optimization problem \eqref{eq:set-opt-obj}--\eqref{eq:set-opt-constraint}, where $\hw(x, D) = \norm{X_t(D) - m(D) \statdist}_\wmat$ is the subset Lyapunov function that we have used in \Cref{sec:proof-set-expansion-policy}, $\lipw = 2\lamw^{1/2}$, and the slack $\slk(x, D)=\beta(1-m(D)) - 0.5\norm{x(D)-m(D)\statdist}_1$ is the same notion as in \Cref{sec:results_set-expansion-policy}. 
Importantly, $D_t$ is chosen to be a \emph{maximal} optimal solution in the sense that there is no other optimal solution $D'$ that contains $D_t$. When there are multiple maximal optimal solutions, $D_t$ is picked uniformly at random.
The formal definition of the set-optimization policy is given in \Cref{alg:set-optimization}.

We note that under the set-optimization policy, the number of arms corresponding to each state-action pair at time step $t$ is determined by the state-count vector $X_t([N])$ rather than the full system state $X_t$, implying that $X_t([N])$ is a Markov chain. Although the subproblem in \eqref{eq:set-opt-obj}--\eqref{eq:set-opt-constraint} requires evaluating $X_t(D)$ for a specific subset $D$, which seems to require full knowledge of the system state $X_t$, we observe that $\hw(X_t, D)$, $m(D)$, and $\slk(X_t, D)$ are determined solely by $X_t(D)$.
Thus, solving the subproblem boils down to deciding $X_t(D_t)$, which only requires the knowledge of $X_t([N])$ rather than $X_t$.

\begin{algorithm}[t] 
\caption{Set-optimization policy}
\label{alg:set-optimization}
\begin{flushleft}
\hspace*{\algorithmicindent} \textbf{Input}: number of arms $N$, budget $\alpha N$, the optimal single-armed policy $\pibs$, \\
\hspace*{\algorithmicindent} \hspace{0.37in} initial system state $X_0$, initial state vector $\veS_0$ 
\end{flushleft}
\begin{algorithmic}[1]
    \For{$t=0,1,2,\dots$}
        \State  Let $\dynsetr_t$ be a maximal optimal solution to the problem below:   \Comment{\ul{\emph{Set update}}}
        \begin{minipage}{15.5cm}
        \begin{subequations}
        \begin{align}
            \label{eq:set-opt-obj}
            \dynsetr_t \gets \argmin_{\dynset\subseteq [N]}& \quad \hw(X_t, D) + \lipw\big(1-m(\dynset)\big)  \\
            \label{eq:set-opt-constraint}
            \text{subject to} &\quad  \slk(X_t, D) \geq 0 
        \end{align}
        \end{subequations}
        \end{minipage}
        \State Run the same action sampling and action rectification steps as in Lines~\ref{alg:set-expansion-sampling}--\ref{alg:set-expansion-action-rect-ends} of \Cref{alg:set-expansion}
        \State  Apply $A_t(i)$ and observe $S_{t+1}(i)$ for each arm $i\in[N]$
    \EndFor
\end{algorithmic}
\end{algorithm}

Our next theorem, \Cref{thm:set-optimization-policy}, shows that the set-optimization policy achieves an $O(1/\sqrt{N})$ optimality gap, which matches the order of optimality gaps of the set-expansion policy and the ID policy.

\begin{theorem}[Optimality gap of set-optimization policy]\label{thm:set-optimization-policy}
    Consider an $N$-armed restless bandit problem with the single-armed MDP $(\sspa, \aspa, P, r)$ and budget $\alpha N$ for $0< \alpha < 1$.  
    Assume that the optimal single-armed policy induces an aperiodic unichain (\Cref{assump:aperiodic-unichain}). Let $\pi$ be the set-optimization policy (\Cref{alg:set-optimization}). Then for all $N$ and initial state vector $\veS_0$, the optimality gap of $\pi$ is bounded as
    \begin{equation}
        \label{eq:set-optimization-policy-bound}
        \ropt - \rsysn \leq \frac{\constso}{\sqrt{N}},
    \end{equation}
    where $\constso$ is a constant depending on $\rmax$, $|\sspa|$, $\rhoBudget \triangleq \min\{\alpha, 1-\alpha\}$, and $P_{\pibs}$; the explicit expression of $\constso$ is given in the proof. 
\end{theorem}

\subsection{Proof of Theorem~\ref{thm:set-optimization-policy} (Optimality gap of set-optimization policy)}
\label{app:proof-set-optimization-policy}

We will spend the remainder of this appendix proving \Cref{thm:set-optimization-policy}. 
Unlike the ID policy and the set-expansion policy, the set-optimization policy does not satisfy \Cref{def:focus-set:monotonic}, so \Cref{thm:set-optimization-policy} can not be proved as a direct corollary of \Cref{thm:focus-set-policy}. However, the proof of \Cref{thm:set-optimization-policy}  follows a similar structure as the framework established in \Cref{sec:formalizing}. 
Specifically, in Appendix~\ref{app:pf-set-opt:lemmas-and-pfs}, we state and prove three lemmas; each lemma either verifies a condition or states a fact that modifies one of the three conditions in \Cref{sec:formalizing}; 
in Appendix~\ref{app:pf-set-opt:main-thm-pf}, we prove \Cref{thm:set-optimization-policy} using similar ideas as \Cref{thm:focus-set-policy}, with the subset Lyapunov functions being $\{\hw(x, D)\}_{D\subseteq[N]}$.

\subsubsection{Lemmas and proofs}\label{app:pf-set-opt:lemmas-and-pfs}
We first show that the set-optimization policy satisfies \Cref{def:focus-set:pibs-consistency}. 

\begin{lemma}[Set-optimization policy satisfies \Cref{def:focus-set:pibs-consistency}]\label{lem:set-optimization-pibs-consistency}
    Consider the set-optimization policy defined in \Cref{alg:set-optimization}. For any time step $t\geq0$, there exists a subset $D_t' \subseteq D_t$ such that for all $i\in D_t'$, the policy chooses $A_t(i) = \syshat{A}_t(i)$, and 
    \begin{equation}
        \Ebig{m(D_t \backslash D_t') \givenbig X_t, D_t} \leq \frac{1}{\sqrt{N}} + \frac{1}{N} \quad a.s.
    \end{equation}
\end{lemma}

\begin{proof}{\textit{Proof of \Cref{lem:set-optimization-pibs-consistency}}.}
    The whole proof is verbatim to the proof of \Cref{lem:set-expansion-pibs-consistency}, based on the fact that for both the set-optimization policy and the set-expansion policy, $D_t$ satisfies $\slk(X_t, D_t) \geq 0$, and $\veA_t$ is chosen such that the number of arms $i\in D_t$ with $A_t(i) = \syshat{A}_t(i)$ is maximized. \Halmos
\end{proof}

Although the set-optimization policy does not satisfy the almost non-shrinking condition (\Cref{def:focus-set:monotonic}), we show that for each time step $t$, there is another subset $D_{t+1}^\SE$ such that the pair of subsets $(D_t, D_{t+1}^\SE)$ satisfies \Cref{def:focus-set:monotonic}, and $D_{t+1}^\SE$ is feasible to the optimization subproblem \eqref{eq:set-opt-obj}--\eqref{eq:set-opt-constraint} in the $(t+1)$-th time step. 

\begin{lemma}\label{lem:set-optimization-superset-monotonic}
    Consider the set-optimization policy defined in \Cref{alg:set-optimization}. For any $t\geq0$, there exists a subset $D_{t+1}^\SE\subseteq[N]$ that depends on $X_{t+1}$ such that 
    \begin{enumerate}
        \item $\slk(X_{t+1}, D_{t+1}^\SE) \geq 0$; 
        \item either $D_{t+1}^\SE\supseteq D_t$ or $D_{t+1}^\SE \subseteq D_t$;
        \item 
        \begin{minipage}{15.5cm}
        \begin{fleqn}
        \begin{equation}\label{eq:set-optimization-monotonic}
            \Ebig{(m(D_t) - m(D_{t+1}^\SE))^+ \givenbig X_t, D_t} \leq \frac{|\sspa|^{1/2} + 1}{\rhoBudget\sqrt{N}} + \frac{1 + (\rhoBudget + 1)|\sspa|}{\rhoBudget N} \quad a.s.
        \end{equation}
        \end{fleqn}
        \end{minipage}
    \end{enumerate}
\end{lemma}

\begin{proof}{\textit{Proof of \Cref{lem:set-optimization-superset-monotonic}}.}
    We let the set $D_{t+1}^\SE$ be the next-time-step focus set if we apply the update rule of the set-expansion policy to $(X_t, D_t)$. By the definition of the set-expansion policy, we automatically get $\slk(X_{t+1}, D_{t+1}^\SE) \geq 0$, and we also have $D_{t+1}^\SE\supseteq D_t$ or $D_{t+1}^\SE \subseteq D_t$. 

    The proof of \eqref{eq:set-optimization-monotonic} closely follows that of \eqref{eq:set-expansion-monotonic} in \Cref{lem:set-expansion-monotonic}, with the only change being that $D_{t+1}$ is replaced by $D_{t+1}^\SE$. The logic of both proofs relies on two key facts that hold in this context:
    \begin{itemize}
        \item The definition of the set-expansion policy implies that when $D_{t+1}^\SE\subseteq D_t$, $D_{t+1}^\SE$ is chosen to be the subset with the largest number of arms among all sets $\overline{D}$ such that $\slk(X_{t+1}, \overline{D}) \geq 0$. 
        \item By \Cref{lem:set-optimization-pibs-consistency}, there exists a subset $D_t' \subseteq D_t$ such that for all $i\in D_t'$, the policy chooses $A_t(i) = \syshat{A}_t(i)$, and $\Ebig{m(D_t \backslash D_t') \givenbig X_t, D_t} = O(1/\sqrt{N})$. 
    \end{itemize}
    With these two properties established, the rest of the argument proceeds identically to the proof of \Cref{lem:set-expansion-monotonic} with $D_{t+1}$ replaced by $D_{t+1}^\SE$.
    \Halmos
\end{proof}

Finally, we show that the set-optimization policy satisfies \Cref{def:focus-set:large-enough}.

\begin{lemma}[Set-optimization policy satisfies \Cref{def:focus-set:large-enough}]\label{lem:set-optimization-large-enough}
    Consider the set-optimization policy defined in \Cref{alg:set-optimization}. For any $t\geq0$, 
    \begin{equation}
        1 - m(D_t) \leq \frac{|\sspa|^{1/2}}{\rhoBudget} \hw(X_t, D_t) + \frac{2}{\rhoBudget N} \quad a.s.,
    \end{equation}
\end{lemma}

\begin{proof}{\textit{Proof of \Cref{lem:set-optimization-large-enough}}.}
    Recall that $D_t$ is chosen to be maximal among the optimal solutions of
    \begin{align}
        \tag{\ref{eq:set-opt-obj}}
        \min_{\dynset\subseteq [N]}&\quad \hw(X_t, D) + \lipw\big(1-m(\dynset)\big)  \\
        \tag{\ref{eq:set-opt-constraint}}
        \text{subject to}&\quad  \slk(X_t, D) \geq 0.
    \end{align}
    Because $\hw(X_t, D)$ is $\lipw$-Lipschitz continuous in $D$ according to \Cref{lem:hw-feature-lyaupnov}, the objective $\hw(X_t, D) + \lipw (1-m(D))$ is non-increasing as $D$ expands. Consequently, there is no subset $D'$ strictly containing $D_t$ that satisfies $\slk(X_t, D') \geq 0$, because otherwise $D'$ would be an optimal solution to \eqref{eq:set-opt-obj}-\eqref{eq:set-opt-constraint} that strictly contains $D_t$, violating the maximality of $D_t$. 
    Then we must have
    \[
        \rhoBudget(1-m(D_t)) - \frac{1}{2}\norm{X_t(D_t) - m(D_t) \statdist}_1 \leq \frac{2}{N},
    \]
    because otherwise, $m(D_t) < 1$, and we can pick any $i\notin D_t$ and show that $\slk(X_t, D_t\cup\{i\}) \geq 0$. 
    Therefore, 
    \begin{align}
        1 - m(D_t) 
        &\leq \frac{1}{\rhoBudget} \norm{X_t(D_t) - m(D_t)\statdist}_1 + \frac{2}{\rhoBudget N}. \nonumber \\
         &\leq \frac{|\sspa|^{1/2}}{\rhoBudget} \hw(X_t, D_t) + \frac{2}{\rhoBudget N} \label{eq:pf-set-opt:cov:dist-dom}
    \end{align}
    where \eqref{eq:pf-set-opt:cov:dist-dom} is by the distance domination property of $\hw(x, D)$ established in \Cref{lem:hw-feature-lyaupnov}. 
    \Halmos
\end{proof}

\subsubsection{Proof of Theorem~\ref{thm:set-optimization-policy}}\label{app:pf-set-opt:main-thm-pf} 

Next, we prove \Cref{thm:set-optimization-policy}. For simplicity, we focus on the case where the policy's induced Markov chain converges to a unique stationary distribution. 
The extension to the general case requires the same minor modifications as in the proof of \Cref{thm:focus-set-policy} and is omitted to avoid repetition.

\begin{proof}{\textit{Proof of \Cref{thm:set-optimization-policy}}.}
    Following the same steps that derive \eqref{eq:opt-gap-bdd-by-v} in \Cref{thm:focus-set-policy}, we get
    \begin{equation}\label{eq:set-opt:opt-gap-bdd-by-v}
        \ropt - \rsysn \le \rmax \left(\frac{1}{\constVnorm} + \frac{2}{\lipw}\right)\Ebig{V(X_\infty, D_{\infty})} + \frac{2\rmax\constAction}{\sqrt{N}}, 
    \end{equation}
    where 
    \[
        V(x, D) = \hw(x, D) + \lipw(1-m(D)).
    \]
    Therefore, it suffices to bound $\Ebig{V(X_\infty, D_{\infty})}$. 
    
    We fix any $t\geq0$. 
    Recall that $D_{t+1}$ is chosen to be the minimizer of $\Vhat(X_{t+1}, D)$ among sets $D$ with $\slk(X_{t+1}, D)\geq 0$. Because $D_{t+1}^\SE$ defined in \Cref{lem:set-optimization-superset-monotonic} satisfies $\slk(X_{t+1}, D_{t+1}^\SE)\geq0$, we must have 
    \begin{equation}\label{eq:proof-set-optimization-use-argmin}
        \Vhat(X_{t+1}, D_{t+1}) \leq \Vhat(X_{t+1}, D_{t+1}^\SE).
    \end{equation}
    Therefore, 
    \begin{align}
        \Vhat(X_{t+1}, D_{t+1})  
        &\leq \Vhat(X_{t+1}, D_{t+1}^\SE)   \nonumber \\
        & = \hw(X_{t+1}, D_{t+1}^\SE)  + \lipwtilde(1-m(D_{t+1}^\SE)) \nonumber \\
        & \leq \hw(X_{t+1}, D_t)  +  \lipwtilde\big\lvert m(D_{t+1}^\SE) - m(D_t)\big\rvert \\
        &\mspace{23mu} +  \lipwtilde(1-m(D_t)) + \lipwtilde(m(D_t) - m(D_{t+1}^\SE))  \nonumber \\
        \label{eq:proof-set-opt-drift-intermediate-1}
        & = \hw(X_{t+1}, D_t)  +  \lipwtilde(1-m(D_t)) + 2 \lipwtilde \big(m(D_t) -  m(D_{t+1}^\SE)\big)^+, 
    \end{align}
    where the second inequality utilizes the fact that $D_{t+1}^\SE \supseteq D_t$ or $D_{t+1}^\SE\subseteq D_t$ established in \Cref{lem:set-optimization-superset-monotonic}, and the Lipschitz continuity of $\hw(x, D)$ with respect to $D$ established in \Cref{lem:hw-feature-lyaupnov}. 

    Therefore, subtracting $\Vhat(X_t, D_t)$ and taking the expectation in \eqref{eq:proof-set-opt-drift-intermediate-1} conditioned on $X_t$, we have
    \begin{align}
        \Ebig{\Vhat(X_{t+1}, D_{t+1}) - \Vhat(X_t, D_t) \givenbig X_t} 
        \label{eq:proof-set-opt-drift-intermediate-2}
        &\leq \Ebig{\hw(X_{t+1}, \dynset_t) - \hw(X_t, \dynset_t) \givenbig X_t}  \\
        \label{eq:proof-set-opt-drift-intermediate-3}
        &\mspace{20mu} + 2 \lipwtilde \Ebig{\big(m(\dynset_t) - m(D_{t+1}^\SE)\big)^+ \givenbig X_t}.
    \end{align}
    We will bound each of the terms in \eqref{eq:proof-set-opt-drift-intermediate-2} and \eqref{eq:proof-set-opt-drift-intermediate-3} separately. 

    To bound the term in \eqref{eq:proof-set-opt-drift-intermediate-2}, notice that by \Cref{lem:set-optimization-pibs-consistency}, there exists $D_t'\subseteq D_t$ such that for any $i\in D_t'$, the policy chooses $A_t(i) = \syshat{A}_t(i)$, and $\Eplain{m(D_t\backslash D_t') \givenplain X_t, D_t} = O(1/\sqrt{N})$. Let $X_{t+1}'$ be the random element denoting the system state at time $t+1$ if $A_t(i) = \syshat{A}_t(i)$ for all $i\in [N]$. We couple $X_{t+1}$ with $X_{t+1}'$ such that $X_{t+1}(D_t') = X'_{t+1}(D_t')$, then we have $\hw(X_{t+1}, D_t') = \hw(X_{t+1}', D_t')$. Consequently, 
    \begin{align}
        \nonumber
        &\Ebig{\hw(X_{t+1}, D_t) \givenbig X_t} \\
        \nonumber
        &\qquad= \Ebig{\hw(X_{t+1}', D_t) \givenbig X_t} +  \Ebig{\hw(X_{t+1}, D_t) - \hw(X_{t+1}', D_t) \givenbig X_t}  \\
        \label{eq:pf-set-opt:main:temp-1}
        &\qquad\leq \rhoContr \Ebig{\hw(X_t, D_t)\givenbig X_t} + \frac{\constHnoise}{\sqrt{N}} + \Ebig{\hw(X_{t+1}, D_t) - \hw(X_{t+1}', D_t) \givenbig X_t}\\
        \label{eq:pf-set-opt:main:temp-2}
        &\qquad\leq \rhoContr \Ebig{\hw(X_t, D_t)\givenbig X_t} + \frac{\constHnoise}{\sqrt{N}}  + 2\lipwtilde \Ebig{ m(D_t\backslash D_t') \givenbig X_t} \\
        \label{eq:pf-set-opt:main:temp-3} 
        &\qquad\leq \rhoContr \Ebig{\hw(X_t, D_t)\givenbig X_t} + \frac{\constHnoise + 2\lipwtilde \constAction}{\sqrt{N}}, 
    \end{align}
    where $\rhoContr = 1 - 1/(2\lamw)$, $\constHnoise = 2\lamw^{1/2}$, $\constAction = 2$; the inequality in \eqref{eq:pf-set-opt:main:temp-1} follows from the drift condition of $\hw(x, D)$ established in \Cref{lem:hw-feature-lyaupnov}; to get the inequality in \eqref{eq:pf-set-opt:main:temp-2}, we use the argument that
    \begin{align*}
        \hw(X_{t+1}, D_t) - \hw(X_{t+1}', D_t) &= \hw(X_{t+1}, D_t) - \hw(X_{t+1}, D_t') +\hw(X_{t+1}', D_t') - \hw(X_{t+1}', D_t)  \\
        &\leq 2\lipwtilde m(D_t\backslash D_t');
    \end{align*}
    the inequality in \eqref{eq:pf-set-opt:main:temp-3} follows from the majority conformity of the set-optimization policy proved in \Cref{lem:set-optimization-pibs-consistency}. Therefore, 
    \[
        \Ebig{\hw(X_{t+1}, D_t) \givenbig X_t} - \hw(X_t, D) \leq - (1-\rhoContr) \Ebig{\hw(x, D_t) \givenbig X_t} + \frac{\constHnoise + 2\lipwtilde \constAction}{\sqrt{N}}. 
    \] 
    To bound the term $2 \lipwtilde \Ebig{\big(m(\dynset_t) - m(D_{t+1}^\SE)\big)^+ \givenbig X_t=x}$ in \eqref{eq:proof-set-opt-drift-intermediate-3}, we apply \Cref{lem:set-optimization-superset-monotonic} to get 
    \[
        2 \lipwtilde \Ebig{\big(m(\dynset_t) - m(D_{t+1}^\SE)\big)^+\givenbig X_t} \leq \frac{2\lipw \constMono}{\sqrt{N}},
    \] 
    where $\constMono = \big(2 + (\rhoBudget + 2)|\sspa|\big) / \rhoBudget$. 
    Plugging the above bounds into \eqref{eq:proof-set-opt-drift-intermediate-2} and \eqref{eq:proof-set-opt-drift-intermediate-3}, we get
    \begin{align}
        \nonumber
        &\Ebig{\Vhat(X_{t+1}, \dynset_{t+1}) - \Vhat(X_t, D_t) \givenbig X_t} \\
        \label{eq:thm2-proof:drift-intermediate-3}
        &\qquad\leq  - (1-\rhoContr) \Ebig{\hw(X_t, D_t) \givenbig X_t} + \frac{\constHnoise +  2\lipwtilde(\constAction + \constMono)}{\sqrt{N}}. 
    \end{align}
    
    Note that by \Cref{lem:set-optimization-large-enough},     
    \begin{equation}
        \label{eq:thm2-proof:drift-intermediate-3.5}
       \Vhat(X_t, D_t) \leq \big(1+ \lipwtilde\lipexpand \big) \hw(X_t, D_t)  + \frac{\lipwtilde \constEpnoise}{\sqrt{N}},
    \end{equation}
    where $\lipexpand = |\sspa|^{1/2} / \rhoBudget$, $\constEpnoise = 3 / \rhoBudget$. 
    Combining \eqref{eq:thm2-proof:drift-intermediate-3.5} with \eqref{eq:thm2-proof:drift-intermediate-3}, we have proved that for any $t\geq0$, 
    \begin{equation}\label{eq:thm2-proof:drift-condition-expression}
        \Ebig{\Vhat(X_{t+1}, D_{t+1}) \givenbig X_t}  \leq \rhoFinal \Ebig{\Vhat(X_t, D_t) \givenbig X_t} + \frac{K_1}{\sqrt{N}},
    \end{equation}
    where $\rhoFinal = 1 - (1-\rhoContr)/(1+ \lipwtilde \lipexpand)$ and $K_1 = \constHnoise + 2\lipwtilde (\constAction + \constMono) + (1-\rhoContr)\lipwtilde\constEpnoise / (1+ \lipwtilde \lipexpand)$.

    Now, with \eqref{eq:thm2-proof:drift-condition-expression}, $\E{V(X_\infty, D_\infty)}$ can be bounded as follows:  
    We take the expectation on both sides of \eqref{eq:thm2-proof:drift-condition-expression} conditioned on $X_t$, and let $t\to\infty$ to get
    \[
        \E{V(X_\infty, D_\infty)} \leq \rhoFinal \E{V(X_\infty, D_\infty)} + \frac{K_1}{\sqrt{N}},
    \]
    which implies that 
    \begin{equation}\label{eq:set-opt:v-bd}
        \E{V(X_\infty, D_\infty)} \leq \frac{K_1}{(1-\rhoFinal)\sqrt{N}}.
    \end{equation}
    We combine \eqref{eq:set-opt:v-bd} with the bound of $\ropt - \rsysn$ in terms of $V(X_\infty, D_\infty)$ in \eqref{eq:set-opt:opt-gap-bdd-by-v}, and substitute $\rhoBudget$, $\lipwtilde$, $\constHnoise$, $\constAction$, $\constMono$, $\lipexpand$, and $\constEpnoise$ with their values. This leads to the final bound
    \begin{equation}\label{eq:set-opt-bound-in-proof}
         \ropt - \rsysn \leq  \frac{256\rmax \lamw^2 |\sspa|^2}{\rhoBudget^2\sqrt{N}}. 
    \end{equation}
    We omit the detailed calculations to obtain \eqref{eq:set-opt-bound-in-proof}. 
    \Halmos
\end{proof}

\section{Extending results to restless bandits with inequality budget constraint}
\label{app:inequality-constraint}
In this appendix, we consider the restless bandit problem with an inequality budget constraint, i.e., 
\begin{subequations}
\begin{align}
    \tag{\ref{eq:objective-inequality}}
    \underset{\text{policy } \pi}{\text{maximize}} & \quad \rliminf  \\
    \tag{\ref{eq:hard-budget-constraint-inequality}}
    \text{subject to}  
    &\quad  \sumN A_t^\pi(i) \leq \alpha N,\quad \forall t\ge 0. 
\end{align} 
\end{subequations}
As stated in \Cref{sec:inequality}, our policies can be adapted to this setting and have the same $O(1/\sqrt{N})$ optimality gap bounds. In the following, we describe the necessary modifications to the policies and the proofs.

\subsection{Policy modifications}
We modify the set-expansion policy and the ID policy for the inequality-constraint setting. As in the equality-constraint setting, we first solve the LP relaxation of problem \eqref{eq:N-arm-formulation-inequality}, as given in \eqref{eq:lp-single-inequality} below. 
This LP relaxation is almost identical to the LP relaxation in the equality-constraint setting \eqref{eq:lp-single}, except that the budget constraint \eqref{eq:expect-budget-constraint-inequality} is now an inequality. 
\begin{subequations}
\label{eq:lp-single-inequality}
\begin{align}
    \underset{\{y(s, a)\}_{s\in\sspa,a\in\aspa}}{\text{maximize}} \mspace{12mu}&\sum_{s\in\sspa,a\in\aspa} r(s, a) y(s, a) \\
    \text{subject to}\mspace{21mu}
    &\mspace{15mu}\sum_{s\in\sspa} \quad y(s, 1) \leq \alpha, \label{eq:expect-budget-constraint-inequality}\\
    & \sum_{s'\in\sspa, a\in\aspa} y(s', a) P(s', a, s) = \sum_{a\in\aspa} y(s,a), \quad \forall s\in\sspa, \\
    &\mspace{3mu}\sum_{s'\in\sspa, a'\in\aspa} y(s',a') = 1,   
    \quad
     y(s,a) \geq 0, \;\; \forall s\in\sspa, a\in\aspa. 
\end{align}
\end{subequations}
Given the optimal solution $(y^*(s,a))_{s\in\sspa,a\in\aspa}$ of \eqref{eq:lp-single-inequality}, the optimal single-armed policy $\pibs$ is defined using \eqref{eq:single-arm-opt-def}. We assume that $\pibs$ induces an aperiodic unichain in the single-armed MDP.

The set-expansion policy for the inequality-constraint setting is defined in \Cref{alg:set-expansion-inequality}, where recall that $\slk(x, D)$ is a function of system state $x$ and subset $D \subseteq [N]$, given by  
\begin{equation*}
    \slk(x, \dynset) = \rhoBudget (1-m(\dynset)) -  \frac{1}{2}\norm{x(D) - m(D) \statdist}_1,
\end{equation*}
Compared with the equality-constraint setting, the only change of the set-expansion policy here is in the action rectification step: As indicated in Lines~\ref{alg:set-expansion-action-rect-starts-inequality}--\ref{alg:set-expansion-action-rect-ends-inequality} of \Cref{alg:set-expansion-inequality}, here we can simply take $A_t(i) = \syshat{A}_t(i)$ for all $i\in D_t$ as long as $\sum_{i\in D_t} \syshat{A}_t(i) \leq \alpha N$.

\begin{algorithm}
\caption{Set-expansion policy (inequality constraint)}
\label{alg:set-expansion-inequality}
\begin{flushleft}
\hspace*{\algorithmicindent} \textbf{Input}: number of arms $N$, budget $\alpha N$, the optimal single-armed policy $\pibs$, \\
\hspace*{\algorithmicindent} \hspace{0.45in} initial system state $X_0$, initial state vector $\veS_0$, initial focus set $D_{-1}=\emptyset$ 
\end{flushleft}
\begin{algorithmic}[1]
    \For{$t=0,1,2,\dots$} 
        \If{$\slk(X_t, \dynsetr_{t-1}) > 0$}
        \Comment{\ul{\emph{Set update}}}
            \State Let $D_t$ be any set with the largest $m(D_t)$ such that $D_t \supseteq D_{t-1}$ and  $\slk(X_t, D_t) \geq 0$
        \Else
            \State Let $D_t$ be any set with the largest $m(D_t)$ such that $D_t \subseteq D_{t-1}$ and $\slk(X_t, D_t) \geq 0$
        \EndIf
        
        \noindent\algrule
        \State Independently sample $\syshat{A}_t(i) \sim \pibs(\cdot|S_t(i))$ for $i\in[N]$ 
        \Comment{\ul{\emph{Action sampling}}}
        \If{$\sum_{i\in \dynsetr_t} \syshat{A}_t(i) \geq \alpha N$} \Comment{\ul{\emph{Action rectification}}}
        \label{alg:set-expansion-action-rect-starts-inequality}
            \State Select $\alpha N$ arms in $\dynsetr_t$ with $\syshat{A}_t(i)=1$ uniformly at random, and set $A_t(i) = 1$
            \State For the rest of $i\in [N]$, set $A_t(i)=0$
        \Else
            \State Set $A_t(i) = \syshat{A}_t(i)$ for $i\in D_t$
            \State Set $A_t(i) = \syshat{A}_t(i)$ for as many $i\notin D_t$ as possible; break ties uniformly at random
            \label{alg:set-expansion-action-rect-ends-inequality}
        \EndIf
        \State Apply $A_t(i)$ and observe $S_{t+1}(i)$ for each arm $i\in[N]$ 
    \EndFor
\end{algorithmic}
\end{algorithm}

The ID policy for the inequality-constraint setting is defined in \Cref{alg:id-inequality}, where the only change is again in the action rectification step. As indicated in Lines~\ref{alg:id-inequality:difference-start}--\ref{alg:id-inequality:difference-end} of \Cref{alg:id-inequality}, here we let $\Ngood = N$ and $A_t(i) = \syshat{A}_t(i)$ for all $i\in[N]$ if $\sumN \syshat{A}_t(i) \leq \alpha N$.

\begin{algorithm}
\caption{ID policy (inequality constraint)}
\label{alg:id-inequality}
\begin{flushleft}
\hspace*{\algorithmicindent} \textbf{Input}: number of arms $N$, budget $\alpha N$, the optimal single-armed policy $\pibs$,\\
\hspace*{\algorithmicindent} \hspace{0.37in}
initial system state $X_0$, initial state vector $\veS_0$
\end{flushleft}
\begin{algorithmic}[1]
    \For{$t=0,1,2,\dots$}
        \State Independently sample $\syshat{A}_t(i)\sim \pibs(\cdot | S_t(i))$ for $i\in[N]$
        \Comment{\ul{\emph{Action sampling}}}
        \If{$\sumN \syshat{A}_t(i) \geq \alpha N$}
        \Comment{\ul{\emph{Action rectification}}}
            \State $\Ngood \gets \max \{n \leq N \colon \sum_{i=1}^{n} \syshat{A}_t(i) \leq \alpha N\}$
            \State $A_t(i) \gets \syshat{A}_t(i)$ for $i \leq \Ngood$, $A_t(i) \gets 0$ for $i > \Ngood$
        \Else \label{alg:id-inequality:difference-start}
            \State $\Ngood \gets N$
            \State $A_t(i) \gets \syshat{A}_t(i)$ for $i \in[N]$ \label{alg:id-inequality:difference-end}
        \EndIf
        \State  Apply $A_t(i)$ and observe $S_{t+1}(i)$ for each arm $i\in[N]$
    \EndFor
\end{algorithmic}
\end{algorithm}

\subsection{Proof modifications}
Next, we restate and prove \Cref{thm:set-expansion-inequality} and \Cref{thm:id-inequality}, which show that the optimality gaps for the set-expansion policy and the ID policy are $O(1/\sqrt{N})$. 
Since most components of the proofs are identical to those for the equality-constraint setting, below we outline the proof steps and focus on the necessary modifications, omitting the parts of the proofs that remain unchanged.

\seineq*

\idineq*

Now we describe the proofs of \Cref{thm:set-expansion-inequality} and \Cref{thm:id-inequality}. First, it is not hard to check that the meta-theorem \Cref{thm:focus-set-policy} and its proof still hold in the inequality-constraint setting. 
Intuitively, the proof of \Cref{thm:focus-set-policy} does not directly utilize the specific form of the budget constraint. Instead, it is based on the properties of the subset Lyapunov functions and Conditions~\ref{def:focus-set:pibs-consistency}--\ref{def:focus-set:large-enough} that characterize the focus set.

Now we apply \Cref{thm:focus-set-policy} to prove \Cref{thm:set-expansion-inequality} for the set-expansion policy. We consider the same subset Lyapunov function as in the equality-constraint setting, recalled below: 
\begin{equation*}
    \hw(x, D) = \norm{x(D) - m(D)\statdist}_W. 
\end{equation*}
The fact that $\{\hw(x, D)\}_{D\subseteq[N]}$ satisfies the definition of subset Lyapunov functions does not rely on the budget constraint of the RB problem, as one can see by inspecting the proof of \Cref{lem:hw-feature-lyaupnov} in Appendix~\ref{app:proof-feature-lyapunov-lemmas}.

Next, it remains to verify the Conditions~\ref{def:focus-set:pibs-consistency}--\ref{def:focus-set:large-enough} for the set-expansion policy. We restate these conditions below for ease of reference.

\begin{customcond}{1}[Majority conformity]
    Let $\constAction>0$ be a constant.
    For any $t\geq0$, there exists $D_t'\subseteq D_t$ such that for any $i\in D_t'$, we have $A_t(i) = \syshat{A}_t(i)$ and 
    \begin{equation}
        \Ebig{m(D_t \backslash D_t') \givenbig X_t, D_t} \leq \frac{\constAction}{\sqrt{N}} \quad a.s.
    \end{equation}
\end{customcond}

\begin{customcond}{2}[Almost non-shrinking]
    For any $t\geq0$, either $D_{t+1} \supseteq D_t$ or $D_{t+1} \subseteq D_t$. 
    Moreover, there exists a constant $\constMono > 0$ such that for any $t\geq0$, 
    \begin{equation}
        \Ebig{\big(m(D_t) - m(D_{t+1})\big)^+ \givenbig X_t, D_t} \leq  \frac{\constMono}{\sqrt{N}} \quad a.s.
    \end{equation}
\end{customcond}

\begin{customcond}{3}[Sufficient coverage] 
There exists a class of subset Lyapunov functions  $\{h(\cdot,\dynset) \colon \dynset \in \mathcal{D} \}$ 
and constants $\lipexpand > 0, \constEpnoise>0$ such that for any $t\geq0$, 
        \begin{equation}
            1 - m(D_t) \leq \lipexpand h(X_t, D_t) + \frac{\constEpnoise}{\sqrt{N}} \quad a.s.
        \end{equation}
\end{customcond}

Inspecting the proofs in \Cref{sec:proof-set-expansion-policy}, we can see that only the proof of \Cref{def:focus-set:pibs-consistency} could be affected by the form of the budget constraint. 
Intuitively, only \Cref{def:focus-set:pibs-consistency} is directly associated with the feasibility of actions, whereas the other two conditions are more closely related to the dynamics of the focus set. 

In \Cref{lem:set-expansion-pibs-consistency-ineq} stated below, we show \Cref{def:focus-set:pibs-consistency} for the set-expansion policy with the inequality constraint. The proof of this lemma is modified from the proof of \Cref{lem:set-expansion-pibs-consistency}.

\begin{lemma}[Set-expansion policy with inequality constraint satisfies \Cref{def:focus-set:pibs-consistency}]\label{lem:set-expansion-pibs-consistency-ineq}
    Consider the set-expansion policy for restless bandits with the inequality constraint (\Cref{alg:set-expansion-inequality}). For any $t\geq0$, there exists a subset $D_t' \subseteq D_t$ such that for any $i\in D_t'$, 
    $A_t(i) = \syshat{A}_t(i)$, and 
    \begin{equation}
        \Ebig{m(D_t \backslash D_t') \givenbig X_t, D_t} \leq \frac{1}{\sqrt{N}} + \frac{1}{N} \quad a.s.
    \end{equation}
\end{lemma}

\begin{proof}{\textit{Proof of \Cref{lem:set-expansion-pibs-consistency}.}}
    Recall that under the set-expansion policy, the actions in the focus set $(A_t(i))_{i\in D_t}$ are chosen according to the ideal actions $(\syshat{A}_t(i))_{i\in D_t}$ based on the following rules: 
    \begin{itemize}
        \item When $\sum_{i\in D_t} \syshat{A}_t(i) > \alpha N$, the set-expansion policy chooses $\alpha N$ arms in $D_t$ with $\syshat{A}_t(i) = 1$ and sets $A_t(i) = \syshat{A}_t(i)$. 
        \item Otherwise, the set-expansion policy sets $A_t(i) = \syshat{A}_t(i)$ for all $i\in D_t$. 
    \end{itemize}
    Let $D_t' = \{i \in D_t \colon A_t(i) = \syshat{A}_t(i)\}$. Then we have
    \begin{equation}
        |D_t\setminus D_t'| = \Big(\sum_{i\in D_t} \syshat{A}_t(i) - \alpha N\Big)^+. 
    \end{equation}
    Therefore, if we can show that for each $t\geq0$, 
    \begin{equation}\label{eq:setdiff-ineq}
        \EBig{\Big(\sum_{i\in D_t} \syshat{A}_t(i) - \alpha N\Big)^+ \givenBig X_t, D_t} \leq \sqrt{N}, 
    \end{equation}
    we will have $\E{m(D_t\setminus D_t') \givenplain X_t, D_t} \leq 1/\sqrt{N}$, which will complete the proof. 
    The remainder of the proof is dedicated to proving \eqref{eq:setdiff-ineq}.

    Observe that given $X_t$ and $D_t$, the $\syshat{A}_t(i)$'s are independent for $i\in D_t$. Consider the \emph{scaled expected budget requirement} for arms in a set $D$, defined as
    \begin{equation}
        C_\pibs(x, D) \triangleq \frac{1}{N}\EBig{\sum_{i\in D}\syshat{A}_t(i) \givenBig X_t=x} = \sum_{s\in\sspa} x(D, s) \pibs(1|s),
    \end{equation}
    Then $\Ebig{\sum_{i\in D_t} \syshat{A}_t(i) \givenbig X_t, D_t} = N C_\pibs(X_t, D_t)$. By the Cauchy-Schwarz inequality, 
    \begin{align}
        \EBig{\absBig{\sum_{i\in D_t} \syshat{A}_t(i) - N C_\pibs(X_t, D_t) } \givenBig X_t, D_t} 
        &\leq \EBig{\Big(\sum_{i\in D_t} \syshat{A}_t(i) - N C_\pibs(X_t, D_t)\Big)^2 \givenBig X_t, D_t}^{\frac{1}{2}} \nonumber\\
        &= \Big(\sum_{i\in D_t} \Varbig{\syshat{A}_t(i) \givenbig X_t, D_t}\Big)^{\frac{1}{2}} \nonumber\\
        &\leq \sqrt{N}. \label{eq:budget-dev-ineq}
    \end{align}

    We next prove \eqref{eq:setdiff-ineq} utilizing \eqref{eq:budget-dev-ineq}.
    Recall that $D_t$ satisfies $\slk(X_t, D_t) \geq 0$, i.e., $\norm{X_t(D_t) - m(D_t) \statdist}_1 / 2 \leq \rhoBudget(1-m(D_t))$.
    Also, note that 
    \begin{equation}
        \sum_{s\in\sspa} \statdist(s) \pibs(1|s) = \sum_{s\in\sspa} y^*(s,1) \leq \alpha.
    \end{equation}
    Then $\big(C_\pibs(X_t, D_t) - \alpha m(D_t)\big)^+$ can be bounded as 
    \begin{align}
         \Big(C_\pibs(X_t, D_t) - \alpha m(D_t)\Big)^+
         \label{eq:pf-set-exp-ineq:pibs-consist:c-deviation-1}
         &\leq  \Big( \sum_{s\in\sspa} \big(X_t(D_t, s) - m(D_t) \statdist(s)\big) \pibs(1|s) \Big)^+ \\ 
         \label{eq:pf-set-exp-ineq:pibs-consist:c-deviation-1-5}
         &\leq  \Big( \sum_{s\in\sspa}\Big( X_t(D_t, s) - m(D_t)\statdist(s)\Big) \Big(\pibs(1|s) - \frac{1}{2}\Big) \Big)^+\\
         \label{eq:pf-set-exp-ineq:pibs-consist:c-deviation-1-6}
         &\leq \frac{1}{2} \norm{X_t(D_t) - m(D_t) \statdist}_1 \\
         \label{eq:pf-set-exp-ineq:pibs-consist:c-deviation-2}         
         &\leq \rhoBudget(1-m(D_t)),
    \end{align}
    where \eqref{eq:pf-set-exp-ineq:pibs-consist:c-deviation-1-5} is because $\sum_{s\in\sspa} \big(X_t(D_t, s) - m(D_t)\statdist(s)\big) = 0$, and \eqref{eq:pf-set-exp-ineq:pibs-consist:c-deviation-1-6} is because $\abs{\pibs(1|s) - 1/2} \leq 1/2$ for all $s\in\sspa$.
    Thus,
    \begin{equation*}
        C_\pibs(X_t, D_t) \le \alpha m(D_t) + \rhoBudget\big(1-m(D_t)\big)\le \alpha m(D_t) + \alpha\big(1-m(D_t)\big) = \alpha. 
    \end{equation*}
    Therefore, we have
    \begin{align}
        \nonumber
        &\EBig{\Big(\sum_{i\in D_t} \syshat{A}_t(i) - \alpha N\Big)^+ \givenBig X_t, D_t}\\
        \nonumber
        &\qquad= \EBig{\Big(\sum_{i\in D_t} \syshat{A}_t(i) - \alpha N\Big)^+  \givenBig X_t, D_t}\\
        \label{eq:pf-set-exp-ineq:pibs-consist:c-deviation-3}
        &\qquad\leq  \EBig{\absBig{\sum_{i\in D_t} \syshat{A}_t(i) - N C_\pibs(X_t, D_t) } \,\givenBig\, X_t, D_t}  \\
        \nonumber
        &\qquad\leq \sqrt{N}. 
    \end{align}
    This proves \eqref{eq:setdiff-ineq}, concluding the proof of Lemma~\ref{lem:set-expansion-pibs-consistency-ineq}. \Halmos
\end{proof}

Now we bound the optimality gap for the ID policy in the inequality-constraint setting. As in the case of set-expansion policy, we only need to establish \Cref{def:focus-set:pibs-consistency}.
The proof of  \Cref{def:focus-set:pibs-consistency} follows almost identical argument as \Cref{lem:id-pibs-consistency}.

\begin{lemma}[ID policy with the inequality constraint satisfies Condition~\ref{def:focus-set:pibs-consistency}]\label{lem:id-pibs-consistency-ineq}
    Consider the ID policy for restless bandits with inequality constraint (\Cref{alg:id-inequality}). 
    For any $t\ge 0$. Let $D_t' = [\min(\Ngood, N\md(X_t))]$. 
    Then
    \begin{equation}
        \Ebig{m(D_t \backslash D_t') \givenbig X_t, D_t} = \frac{1}{N} \Ebig{(N\md(X_t) - \Ngood)^+ \givenbig X_t} \leq  \frac{2}{\rhoBudget\sqrt{N}} + \frac{1}{N} \quad a.s.
    \end{equation} 
\end{lemma}

\begin{proof}{\textit{Proof of \Cref{lem:id-pibs-consistency-ineq}}.}
    In this proof, the variables $n$ and $n'$ are by default non-negative integers. 
    We fix a time step $t\ge 0$, and condition on $X_t = x$ for a fixed system state realization $x$. 
    We first prove a lower bound for $\Ngood$ in terms of the \emph{budget requirements} of the arms in $[n]$, $\sum_{i\in[n]} \syshat{A}_t(i)$, for each $n \leq N$. 
    Observe that the arms in $[n]$ can follow their ideal actions if and only if 
    \begin{equation}
        \label{eq:action-1-conformity-ineq}
        \sum_{i\in[n]} \syshat{A}_t(i) \le \alpha N. 
    \end{equation}
    Consequently, a sufficient condition for the arms in $[n]$ to follow their ideal actions is
    \begin{equation}
    \Big(\sum_{i\in[n]} \syshat{A}_t(i) - \alpha n\Big)^+ \leq \rhoBudget\big(N - n\big),\label{eq:suff-1-ineq}
    \end{equation}
    where recall that $\rhoBudget=\min\{\alpha, 1-\alpha\}$.
    A further sufficient condition for \eqref{eq:suff-1-ineq} is
    \begin{equation}
        \max_{n'\le n}\Big(\sum_{i\in[n']} \syshat{A}_t(i) - \alpha n'\Big)^+ \leq \rhoBudget\big(N - n\big).\label{eq:suff-2-ineq}
    \end{equation}
    Therefore, by the definition of $\Ngood$,
    \begin{align}
        \Ngood &\geq \max \bigg\{n \le N \colon \max_{n'\le n}\Big(\sum_{i\in[n']} \syshat{A}_t(i) - \alpha n'\Big)^+ \leq \rhoBudget\big(N - n\big) \bigg\}\nonumber\\
        &= N \max \bigg\{m\in[0,1]_N \colon \max_{\substack{m'\in[0,1]_N\\m'\le m}} \Big(\frac{1}{N}\sum_{i\in[Nm']} \syshat{A}_t(i) - \alpha m'\Big)^+ \leq \rhoBudget\big(1 - m\big) \bigg\}.\label{eq:prove-ngood-bound-intermediate-0-ineq}
    \end{align}

    Our next step is to bound the quantity $\max_{m'\in[0,1]_N,m'\le m} \big(\frac{1}{N}\sum_{i\in[Nm']} \syshat{A}_t(i) - \alpha m'\big)^+$. 
    To do this, we relate it to the subset Lyapunov function $\hid(x,m)$ and the scaled expected budget requirement, $C_\pibs(x, D)$, which is defined as: 
    \begin{equation}
        C_\pibs(x, D) \triangleq \frac{1}{N}\EBig{\sum_{i\in D}\syshat{A}_t(i)\Big|X_t=x}  = x(D) \costvec^\top,
    \end{equation}
    where $\costvec$ is the row vector  $(\pibs(1|s))_{s\in\sspa}$.
    We first decompose the target quantity using the triangle inequality: for any $m\in[0,1]_N$, we have
    \begin{align}
        &\max_{\substack{m'\in[0,1]_N\\m'\le m}} \Big(\frac{1}{N}\sum_{i\in[Nm']} \syshat{A}_t(i) - \alpha m' \Big)^+ \nonumber\\
        &\qquad\leq \max_{\substack{m'\in[0,1]_N\\m'\le m}} \biggl(\Big(C_\pibs(x, [Nm']) - \alpha m' \Big)^+ + \absBig{\frac{1}{N}\sum_{i\in[Nm']} \syshat{A}_t(i) - C_\pibs(x, [Nm'])} \biggr)\nonumber\\
        &\qquad\leq \max_{\substack{m'\in[0,1]_N\\m'\le m}}\Big(C_\pibs(x, [Nm']) - \alpha m'\Big)^+ + \max_{m'\in[0,1]_N} \absBig{\frac{1}{N}\sum_{i\in[Nm']} \syshat{A}_t(i) - C_\pibs(x, [Nm'])},\label{eq:curve-decomp-ineq}
    \end{align}
    The second term on the right-hand side is a noise term that we will bound later. We now focus on bounding the first term, which captures the deviation of the instantaneous expected budget requirement from its steady-state expectation. Using the fact that $\alpha \geq  \statdist \costvec^\top$, we have:
    \begin{align}
        \big(C_\pibs(x, \setbelow{\dynset}{m'}) - \alpha m' \big)^+
        &\leq \big( (x(\setbelow{\dynset}{m'}) - m' \statdist) \costvec^\top \big)^+ \nonumber \\
        &\leq \abs{(x(\setbelow{\dynset}{m'}) - m' \statdist) \costvec^\top} \nonumber \\
        &= \abs{ (x(\setbelow{\dynset}{m'}) - m'\statdist) \wmat^{1/2} {\wmat}^{-1/2} \costvec^\top } \nonumber \\
        &\leq \norm{x(\setbelow{\dynset}{m'}) - m'\statdist}_\wmat \norm{\costvec}_{\wmat^{-1}}\nonumber\\
        &=\ratiocw\hw(x,m').
    \end{align}
    Consequently, we have
    \begin{equation}
        \max_{\substack{m'\in[0,1]_N\\m'\le m}}\Big(C_\pibs(x, [Nm]) - \alpha m\Big)^+
        \le \ratiocw\max_{\substack{m'\in[0,1]_N\\m'\le m}}\hw(x,m')
        =\ratiocw\hid(x,m).\nonumber
    \end{equation}
    Using the monotonicity of $\hid(x,m)$ in $m$ and the definition of $\md(x)$, we have that for any $m \le \md(x)$, 
    \begin{equation}
        \ratiocw\hid(x,m) \leq \ratiocw\hid(x,\md(x)) \leq \rhoBudget(1-\md(x)).
    \end{equation}
    Substituting these bounds back into the decomposition in \eqref{eq:curve-decomp-ineq} implies the following bound for any $m\in\gridset$ such that $m\leq \md(x)$: 
    \begin{equation}
        \label{eq:budget-diff-upper-ineq}
        \max_{\substack{m'\in[0,1]_N\\m'\le m}} \Big(\frac{1}{N}\sum_{i\in[Nm']} \syshat{A}_t(i) - \alpha m'\Big)^+
        \le \rhoBudget(1-\md(x))+\max_{m'\in[0,1]_N} \absBig{\frac{1}{N}\sum_{i\in[Nm']} \syshat{A}_t(i) - C_\pibs(x, [Nm'])}. 
    \end{equation}

    To lower bound $\Ngood$ utilizing \eqref{eq:prove-ngood-bound-intermediate-0-ineq} and \eqref{eq:budget-diff-upper-ineq}, we want to find $\overline{m}\in\gridset$ such that $\overline{m}\leq \md(x)$ and the right-hand side of \eqref{eq:budget-diff-upper-ineq} is at most $\rhoBudget(1-\overline{m})$. For any such $\overline{m}$, \eqref{eq:prove-ngood-bound-intermediate-0-ineq} implies that 
    \[
        \overline{m} \in \bigg\{m\in[0,1]_N \colon \max_{\substack{m'\in[0,1]_N\\m'\le m}} \absBig{\frac{1}{N}\sum_{i\in[Nm']} \syshat{A}_t(i) - \alpha m'} \leq \rhoBudget\big(1 - m\big) \bigg\},
    \]
    and then \eqref{eq:budget-diff-upper-ineq} implies that $\Ngood \geq N\, \overline{m}$. 
    Taking the largest such $\overline{m}$ yields
    \begin{align}
        \Ngood &\geq \min\bigg\{N\md(x), \floorBig{N - \frac{N}{\rhoBudget}\Big(\rhoBudget(1-\md(x)) +\max_{m'\in[0,1]_N} \absBig{\frac{1}{N}\sum_{i\in[Nm']} \syshat{A}_t(i) - C_\pibs(x, [Nm'])}\Big)}\bigg\} \nonumber \\
        &\geq  \min\bigg\{N\md(x), N - \frac{N}{\rhoBudget}\Big(\rhoBudget(1-\md(x)) +\max_{m'\in[0,1]_N} \absBig{\frac{1}{N}\sum_{i\in[Nm']} \syshat{A}_t(i) - C_\pibs(x, [Nm'])}\Big)-1\bigg\}  \nonumber \\ 
        &=  \min\bigg\{N\md(x), N\md(x) - 1 - \frac{1}{\rhoBudget} \max_{m'\in[0,1]_N} \absBig{\sum_{i\in[Nm']} \syshat{A}_t(i) - N C_\pibs(x, [Nm'])}\bigg\} \nonumber \\ 
        &= N\md(x) -1 - \frac{1}{\rhoBudget} \max_{n'\leq N} \absBig{\sum_{i\in[n']} \syshat{A}_t(i) - N C_\pibs(x, [n'])}. \nonumber 
    \end{align}
    Rearranging the terms and taking the expectation, we get 
    \begin{equation}\label{eq:prove-ngood-bound-intermediate-3-inequality}
        \EBig{\big(N\md(x) - \Ngood\big)^+ \givenBig X_t = x} \leq 1 + \frac{1}{\rhoBudget} \EBig{\max_{n'\leq N} \absBig{\sum_{i\in[n']} \syshat{A}_t(i) - N C_\pibs(x, [n'])} \givenBig X_t=x}. 
    \end{equation}

    With \eqref{eq:prove-ngood-bound-intermediate-3-inequality}, it remains to prove
    \begin{equation}\label{eq:budget-req-maximal-deviation-inequality}
        \EBig{\max_{n \leq N} \absBig{\sum_{i\in[n]} \syshat{A}_t(i) - N C_\pibs(x, [n])} \givenBig X_t = x} \leq 2\sqrt{N}.
    \end{equation}
    Our proof uses Doob's maximum inequality. First, we simplify the expression by defining the centered random variables $\zmNoise(i) = \syshat{A}_t(i) - \Ebig{\syshat{A}_t(i) \givenbig X_t = x}$. This allows us to rewrite the left-hand side of \eqref{eq:budget-req-maximal-deviation-inequality} as:
    \begin{equation}\label{eq:budget-req-deviation:discretize-inequality}
        \EBig{\max_{n \leq N} \absBig{\sum_{i\in[n]} \syshat{A}_t(i) - N C_\pibs(x, [n])} \givenBig X_t = x} = \EBig{\max_{n \leq N} \absBig{\sum_{i\in[n]} \zmNoise(i)} \givenBig X_t = x}. 
    \end{equation}
    The sequence of partial sums $(\sum_{i\in[n]} \zmNoise(i))_{n\in[N]}$ forms a martingale when conditioned on $X_t=x$. This holds because the ideal actions $\syshat{A}_t(i)$ are sampled independently, making the $\zmNoise(i)$ terms independent random variables, each with a conditional expectation of zero. 
    Applying Cauchy-Schwarz and Doob's $L_2$ maximum inequality \citep[see, e.g.,][Theorem 4.4.4]{Dur_19_prob_book}, we bound this term as:
    \begin{align*}
        \EBig{\max_{n \leq N} \absBig{\sum_{i\in[n]} \zmNoise(i)} \givenBig X_t = x}
        &\leq \EBig{\max_{n \leq N} \Big|\sum_{i\in[n]} \zmNoise(i)\Big|^2 \givenBig X_t = x}^{1/2}\\
        &\le \biggl(4\EBig{\Big|\sum_{i\in[N]} \zmNoise(i)\Big|^2  \givenBig X_t = x}\biggr)^{1/2}\\
        &=\biggl(4\sum_{i\in[N]} \EBig{\zmNoise(i)^2 \givenBig X_t = x}\biggr)^{1/2}\\
        &\le 2\sqrt{N}.
    \end{align*}
    Here, the first inequality follows from Cauchy-Schwarz, the second applies Doob's maximal $L_2$ inequality, and the equality follows from the independence of the $\zmNoise(i)$ terms. The final inequality holds because $\absbig{\zmNoise(i)} = \absbig{ \syshat{A}_t(i) - \Ebig{\syshat{A}_t(i) \givenbig X_t = x}} \leq 1$. This completes the proof.  \Halmos
\end{proof}

\section{Experimental details}
\label{app:exp-details}
In this appendix, we provide details for the experiments in \Cref{sec:experiments}. 
In Appendix~\ref{app:exp-details:set-expansion}, we discuss the implementation details of the set-expansion policy. Then in Appendix~\ref{app:exp-details:rb-examples-details}, we provide additional details of some performance comparison experiments in \Cref{sec:experiments}, including the simulation settings and the definitions of the MDPs. Next, in Appendix~\ref{app:exp-details:commonness}, we comment on the details of the experiments in \Cref{sec:experiments:non-ugap-is-common} which investigate the probability that a random RB instance violates GAP. 
In Appendix~\ref{app:exp-details:dense-reward}, we investigate the effect of dense reward functions on the commonness of non-GAP instances and the performance of the index policies on the non-GAP instances. 
Finally, in Appendix~\ref{app:experiments:se-vs-id}, we conduct an additional experiment to investigate the differences between the set-expansion policy and the ID policy.

\subsection{Implementation details of set-expansion policy}
\label{app:exp-details:set-expansion}
In this subsection, we discuss the implementation details of the set-expansion policy (\Cref{alg:set-expansion}) that we use in our experiments. We first discuss the set-update step on Lines~\ref{alg:set-expansion:set-update-step-begins}--\ref{alg:set-expansion:set-update-step-ends}, where the focus-set $D_t$ is updated. 
Then we discuss the action rectification step on Lines~\ref{alg:set-expansion-action-rect-starts}--\ref{alg:set-expansion-action-rect-ends} which determines the actions based on the focus set and the ideal actions.

\paragraph{Set update step.}
Recall that $\slk(x, D)$ is given by $\slk(x, D) = \beta(1-m(D)) - 0.5\norm{x(D) - m(D)\statdist}_1$, and the set-expansion policy chooses $D_t$ with the maximal cardinality such that 
\begin{itemize}
    \item $\slk(X_t, \dynsetr_t) \geq 0$, and 
    \item $D_t \supseteq D_{t-1}$ if $\slk(X_t, D_{t-1}) > 0$, or $D_t \subseteq D_{t-1}$ if $\slk(X_t, D_{t-1}) \leq 0$. 
\end{itemize} 
Due to the complexity of directly optimizing over the discrete variable $D_t$, we will first decide $X_t(D_t)$, and then find $D_t$ based on $X_t(D_t)$. 
Specifically, when $\slk(X_t, \dynsetr_{t-1}) \geq 0$, consider the following optimization problem, whose optimal solution $(\ve{z}^*, m^*)$ gives $(X_t(D_t), m(D_t))$ if we omit the integer effect: 
\begin{subequations}
\begin{align}
    \underset{\ve{z}\in \R^{|\sspa|}}{\text{maximize}}  \mspace{12mu}& \sums z(s) \\
    \text{subject to}\mspace{25mu}
    \label{eq:se:expansion-constraint-1} 
    & X_t(D_{t-1}, s) \leq z(s) \leq X_t([N], s) \quad \forall s\in\sspa  \\
    \label{eq:se:z-and-m}
    & \sums z(s) = m \\
    \label{eq:se:slack-constraint} 
    & \frac{1}{2}\sums\abs{z(s) - \statdist(s)m} \leq  \beta(1-m).
    \vspace*{-\baselineskip}
\end{align}
\end{subequations}
Here, the constraints \eqref{eq:se:expansion-constraint-1} and \eqref{eq:se:z-and-m} ensure that each feasible solution $(\ve{z}, m)$ corresponds to a $(X_t(D), m(D))$ for some $D\supseteq D_{t-1}$, modulo the integer effect; the constraint \eqref{eq:se:slack-constraint} ensures that $\slk(X_t, D) \geq 0$. 
To solve the above optimization problem, we can equivalently convert it to the following LP: 
\begin{subequations}
\begin{align}
    \underset{\ve{z}, \, \ve{d}\in \R^{|\sspa|}, \, m\in \R}{\text{maximize}}  \mspace{12mu}& m \\
    \text{subject to}\mspace{25mu}
    & X_t(D_{t-1}, s) \leq z(s) \leq X_t([N], s) \quad \forall s\in\sspa \\
    & \sums z(s) = m \\
    & z(s) - \statdist(s) m \leq d(s) \quad \forall s\in\sspa \\
    & -z(s) + \statdist(s) m \leq d(s) \quad \forall s\in\sspa \\
    & \frac{1}{2}\sums d(s) \leq \beta(1-m).
    \vspace*{-\baselineskip}
\end{align}
\end{subequations}
Similarly, when $\slk(X_t, \dynsetr_{t-1}) \leq 0$, we consider the optimization problem 
\begin{subequations}
\begin{align}
    \underset{\ve{z}\in \R^{|\sspa|}}{\text{maximize}}  \mspace{12mu}& \sums z(s) \\
    \text{subject to}\mspace{25mu}
    \label{eq:se:shrink-constraint} 
    & 0 \leq z(s) \leq  X_t(D_{t-1}, s) \quad \forall s\in\sspa \\
    & \sums z(s) = m \\
    & \frac{1}{2} \sums \abs{z(s) - \statdist(s)m} \leq  \beta(1-m),
\end{align}
\end{subequations}
where the only change is \eqref{eq:se:shrink-constraint}, which ensures that each feasible solution $(\ve{z}, m)$ corresponds to a $(X_t(D), m(D))$ for some $D\subseteq D_{t-1}$. This optimization problem is equivalent to the LP given by
\begin{subequations}
\begin{align}
\vspace*{-\baselineskip}
    \underset{\ve{z},\, \ve{d} \in \R^{|\sspa|}, \, m\in \R}{\text{maximize}}  \mspace{12mu}& m \\
    \text{subject to}\mspace{25mu}
    &  0 \leq z(s) \leq X_t(D_{t-1}, s)  \quad \forall s\in\sspa \\
    & \sums z(s) = m \\
    & z(s) - \statdist(s) m \leq d(s) \quad \forall s\in\sspa \\
    & -z(s) + \statdist(s) m \leq d(s) \quad \forall s\in\sspa \\
    & \frac{1}{2}\sums d(s) \leq \beta(1-m).
\end{align}
\end{subequations}
Note that each of the two LPs has $2|\sspa|+1$ variables and $4|\sspa| + 2$ constraints, so the complexities of solving them are polynomials in $|\sspa|$ and independent of $N$. 

After obtaining the optimal solution $(\ve{z}^*, m^*)$ of either of the LPs above, we pick $\lfloor N\ve{z}^*(s) \rfloor$ arms in state $s$ for each $s\in\sspa$ to form the subset $D_t$. 

Note that because we first solve the LP and then perform the rounding, the resulting $D_t$ may not be the exact optimal solution that maximizes the cardinality as required by Lines~\ref{alg:set-expansion:set-update-step-begins}--\ref{alg:set-expansion:set-update-step-ends} of \Cref{alg:set-expansion}. 
A more rigorous alternative is to include the integer constraints $N\ve{z}^*(s) \in \mathbb{Z}$ for $s\in\sspa$ in the LPs. 
Nevertheless, our current implementation is more efficient and has demonstrated good performance in the simulations. 
One can also verify that an $O(1/\sqrt{N})$ optimality gap can still be achieved using this approximate implementation, since intuitively, the maximality of $m(D_t)$ is still approximately preserved up to a negligible $O(1/N)$ error.

\paragraph{Action rectification step.}
For the vanilla version of the set-expansion policy, the action rectification step has been completely specified in \Cref{alg:set-expansion}, where we prioritize the arms in $D_t$ over those in $D_t^c$ to follow the ideal actions, breaking ties uniformly at random. 

For the version of the set-expansion policy that utilizes the LP index policy for tie-breaking, its pseudocode is given in  \Cref{alg:set-expansion-lp-index}. 
To summarize the differences, if not all arm in $D_t$ can follow the ideal actions, this version of the set-expansion policy breaks ties using LP indices instead of uniformly at random; if all arms in $D_t$ can follow the ideal actions, this version of the set-expansion policy invokes the LP index policy to allocate the remaining budget to the arms not in $D_t$.

\begin{algorithm}[ht] 
\caption{Set-expansion policy (with LP index)}
\label{alg:set-expansion-lp-index}
\begin{flushleft}
\hspace*{\algorithmicindent} \textbf{Input}: number of arms $N$, budget $\alpha N$, the optimal single-armed policy $\pibs$, \\
\hspace*{\algorithmicindent} \hspace{0.45in} initial system state $X_0$, initial state vector $\veS_0$, initial focus set $D_{-1}=\emptyset$ 
\end{flushleft}
\begin{algorithmic}[1]
    \For{$t=0,1,2,\dots$} 
        \If{$\slk(X_t, \dynsetr_{t-1}) > 0$}
        \Comment{\ul{\emph{Set update}}}
            \State Let $D_t$ be any set with the largest $m(D_t)$ such that $D_t \supseteq D_{t-1}$ and  $\slk(X_t, D_t) \geq 0$
        \Else
            \State Let $D_t$ be any set with the largest $m(D_t)$ such that $D_t \subseteq D_{t-1}$ and $\slk(X_t, D_t) \geq 0$
        \EndIf
        
        \noindent\algrule
        \Comment{\emph{Lines below implement Lines~\ref{alg:focus-set-action-sampling}--\ref{alg:focus-set-apply-action} of \Cref{alg:focus-set}}}
        \State Independently sample $\syshat{A}_t(i) \sim \pibs(\cdot|S_t(i))$ for $i\in[N]$ 
        \Comment{\ul{\emph{Action sampling}}}
        \If{$\sum_{i\in \dynsetr_t} \syshat{A}_t(i) \geq \alpha N$} \Comment{\ul{\emph{Action rectification}}}
        \label{alg:set-expansion-lp-index:act-rect-starts}
            \State Select $\alpha N$ arms in $\dynsetr_t$ with $\syshat{A}_t(i)=1$ to set $A_t(i) = 1$; break ties favoring larger LP indices
            \State For the rest of $i\in [N]$, set $A_t(i)=0$
        \ElsIf{$\sum_{i\in \dynsetr_t} \syshat{A}_t(i) \leq \alpha N - (N - |D_t|)$}
            \State Select $(1-\alpha)N$ arms in $\dynsetr_t$ with $\syshat{A}_t(i)=0$ to set $A_t(i)=0$; break ties favoring smaller LP indices 
            \State For the rest of $i\in [N]$, set $A_t(i) = 1$
        \Else
            \State Set $A_t(i) = \syshat{A}_t(i)$ for $i\in D_t$ 
            \State Apply the LP index policy to the arms in $D_t^c$ with $\big(\alpha N - \sum_{i\in \dynsetr_t} \syshat{A}_t(i)\big)$ units of budget 
            \label{alg:set-expansion-lp-index:act-rect-ends}
        \EndIf
        \State Apply $A_t(i)$ and observe $S_{t+1}(i)$ for each arm $i\in[N]$
    \EndFor
\end{algorithmic}
\end{algorithm}

\subsection{Additional details of the performance comparison experiments in \Cref{sec:experiments}}
\label{app:exp-details:rb-examples-details}

Next, we provide some additional details of the performance comparison experiments in \Cref{sec:experiments}. We first talk about the simulation settings and the computation of the confidence interval. Then we include the definitions of the two RB instances in \Cref{sec:experiments:compare-non-ugap}. 

\paragraph{Simulation setting and output analysis.}
When simulating most of the RB problems and the policies, we run $5$ independent replications for each $N$. The initial state of each arms in each replication is independently sampled from the uniform distribution over the state space. Each replication runs for $2\times 10^4$ time steps. 
The exceptions are the simulations of the FTVA policy on the two non-SA examples in \Cref{fig:non-sa}, where we run the simulations for $1.6 \times 10^5$ time steps in each replication. 

We compute the confidence interval of the average reward using the batch means method, a common method for computing the confidence intervals in steady-state simulations. Specifically, for each RB problem, policy, and number of arms $N$, we divide the sample path of each replication into $4$ intervals of equal lengths and compute the sample mean of the rewards within each interval. As a result, we get $20$ sample means from the $5$ replications. Then we further average the $20$ sample means and use it as the estimation of the long-run average reward. 
The confidence interval for the estimate is calculated using the variance of the $20$ sample means, under the assumption that each interval is long enough for the system to mix to the steady state, resulting in nearly independent sample means. We refer the reader to the textbook \citepapp[][]{AsmGly_07} for more backgrounds on the batch means method.

\paragraph{Definition of the example in \Cref{fig:non-ugap-perf:three-state}.}
The RB instance in \Cref{fig:non-ugap-perf:three-state} has been given in Example 2 in Appendix E.2 of \citep{GasGauYan_20_whittles} and Appendix G.1 of \citep{HonXieCheWan_23}. Nevertheless, we include it here for completeness. The single-armed MDP of this RB instance has three states, whose transition kernel is given by  
\begin{align*}
P( \cdot ,0,\cdot ) \ &=\ \begin{bmatrix}
0.02232142 & 0.10229283 & 0.87538575\\
0.03426605 & 0.17175704 & 0.79397691\\
0.52324756 & 0.45523298 & 0.02151947
\end{bmatrix}, \\
P( \cdot ,1,\cdot ) \ &=\ \begin{bmatrix}
0.14874601 & 0.30435809 & 0.54689589\\
0.56845754 & 0.41117331 & 0.02036915\\
0.25265570 & 0.27310439 & 0.4742399
\end{bmatrix},
\end{align*}
where the $s$-th row and $s'$-th column in each matrix above represents $P(s, 0, s')$ or $P(s, 1, s')$ for each pair of $s,s'\in\sspa$. The reward function of the single-armed MDP is given by 
\begin{align*}
    r( \cdot , 0) \ &=
    \ \begin{bmatrix}
        0 & 0 & 0
        \end{bmatrix}, \\
    r( \cdot , 1) \ &=
    \ \begin{bmatrix}
        0.37401552 & 0.11740814 & 0.07866135
        \end{bmatrix}.
\end{align*}
The budget parameter $\alpha$ is $0.4$ for this RB instance, that is, $0.4 N$ arms are activated in each time step.

\paragraph{Definition of the example in \Cref{fig:non-ugap-perf:eight-state}.}
The RB instance in \Cref{fig:non-ugap-perf:eight-state} is a modification of the example provided in Section~3.3 of \citep{HonXieCheWan_23}. 
Specifically, let the single-armed MDP have the state space $\sspa=\{0,1,\ldots,7\}$. 
Each state is associated with a \textit{preferred action}, which is action $1$ for states in $\{0, 1, 2, 3\}$, and action $0$ for the other states. 
If an arm is in state~$s$ and takes the preferred action, it moves to state $(s+1) \bmod 8$ with probability $p_{s,\rightsub}$, and stays in state $s$ otherwise; if it does not take the preferred action, it moves to state $(s-1)^+$ with probability $p_{s,\leftsub}$. 
Here, the subscript $\leftsub$ (resp., $\rightsub$) represents ``left'' (resp., ``right''), and we are imagining the states of the single-armed MDP being lined up in a row from state $0$ to state $7$; by taking the preferred actions, the arm moves to the right and loops back to the state $0$ after passing the state $7$. 
The probabilities $p_{\cdot, \rightsub}$ and $p_{\cdot, \leftsub}$ are given by
\begin{align*}
    p_{\cdot , \rightsub} &= \begin{bmatrix}
        0.1 & 0.1 & 0.1 & 0.1 & 0.1 & 0.1 & 0.1 & 0.1
        \end{bmatrix}, \\
    p_{\cdot , \leftsub} &= \begin{bmatrix}
            1.0 & 1.0 & 0.48 & 0.47 & 0.46 & 0.45 & 0.44 & 0.43
        \end{bmatrix}.
\end{align*}
The reward function is given by $r(7, 0) = 0.1$, $r(0,1)=1/300$, and $r(s,a)=0$ for all other $s\in\sspa,a\in\aspa$.  
The budget parameter $\alpha= 1/2$, so $N/2$ arms are activated in each time step. 

Compared to the RB instance in Section~3.3. of \citep{HonXieCheWan_23}, the only modification in this RB instance is changing $r(0,1)$ from $0$ to $1/300$. This small modification makes this instance non-indexable for the Whittle index policy, while keeping the performances of the other policies almost unchanged.

\subsection{Generating and identifying non-GAP examples in \Cref{sec:experiments:non-ugap-is-common}}
\label{app:exp-details:commonness}
In this subsection, we first provide some details on generating random RB instances following the symmetric Dirichlet distribution in \Cref{sec:experiments:non-ugap-is-common}. Then we comment on how we identify the non-GAP instances through the notion of local instability.

\paragraph{Generating Dirichlet random examples.}
Each of the random RB examples in \Cref{fig:eigs} is generated as follows.
First, for each state-action pair $(s,a)$, we sample the vector $(P(s,a,s'))_{s'\in\sspa}$ independently from a symmetric Dirichlet distribution. 
Similarly, for each action $a$, we sample the vector $(r(s, a))_{s\in\sspa}$ independently from the same Dirichlet distribution. 
To highlight the sparsity, we then set to zero any entries of $(P(s,a,s'))_{s'\in\sspa}$ and $(r(s, a))_{s\in\sspa}$ that are less than $10^{-7}$. 
Finally, we sample the budget parameter $\alpha$ from the uniform distribution over the interval $[0.1, 0.9]$ and round it down to the nearest multiple of $0.01$.

\paragraph{Identifying local instability.}
To identify a non-GAP example, we use the notation of \emph{local instability}. 
To avoid ambiguity, we only consider local instability for the RB problems that satisfy the three conditions:
\begin{enumerate}
    \item The optimal solution to the LP relaxation \eqref{eq:lp-single}, $y^*$, is unique; 
    \item There are no transient states for $y^*$, that is, $y^*(s,1)+y^*(s,0) > 0$ for all $s\in\sspa$;
    \item The RB problem is non-degenerate, that is, there exists $\sneu\in\sspa$ with $y^*(\sneu,1) > 0$ and $y^*(\sneu,0) > 0$. 
\end{enumerate}
Note that the state $\sneu$ in the third condition must be unique, due to properties of this LP's basic feasible solutions \citep[][Proposition~2]{GasGauYan_23_exponential}. 
Under these assumptions, the mean-field dynamics around $\statdist$ is locally linear and is the same for all LP-Priority policies. Specifically, we have
\begin{equation}
    \E{X_{t+1}([N]) - \statdist \givenplain X_t([N])} = (X_t([N]) - \statdist)\Phi,
\end{equation}
given that $X_t([N])$ is sufficiently close to $\statdist$, and an LP-Priority policy is used; the matrix $\Phi$ is defined as
\begin{equation}\label{eq:phi-def}
    \Phi \triangleq P_\pibs - \vone^\top \statdist - (\costvec-\alpha\vone)^\top  (P_1(\sneu) - P_0(\sneu)), 
\end{equation}
where $\costvec \triangleq (\pibs(1|s))_{s\in\sspa}$ and $P_a(\sneu) \triangleq (P(\sneu,a, s))_{s\in\sspa}$ are both row vectors; $\vone$ is the all-one row vector. We refer the readers to Appendix B of \citep{GasGauYan_23_whittles} for a detailed derivation of the dynamics under LP-Priority policies. 

In the context of our experiments, we say an RB instance is \emph{locally unstable} if the matrix $\Phi$ defined in \eqref{eq:phi-def} is unstable, i.e., its spectral radius is strictly larger than $1$. 
This local instability of an RB instance implies that under any LP-Priority policy, the system's mean-field dynamics will drift away from the optimal state distribution $\statdist$ whenever the system gets close to $\statdist$. This causes the GAP assumption to be violated under any LP-Priority policy.

\subsection{Commonness of non-GAP examples when the reward functions are not sparse}
\label{app:exp-details:dense-reward}
In \Cref{sec:experiments:non-ugap-is-common}, we have shown that a certain type of non-GAP RB instances, i.e., the locally unstable instances, are common when the transition kernel and the reward functions are generated from some sparse distribution. We have also investigated the performance of the LP-index policy and the Whittle index policy on the locally unstable instances. 
Here, we include further experimental results to show that the sparsity of the reward function may not be essential to the phenomenon observed in \Cref{sec:experiments:non-ugap-is-common}.

Specifically, we consider random RB instances with $\abs{\sspa} = 10$, where the rows of the transition kernel $P(s,a,\cdot)$ for $s\in\sspa, a\in\aspa$ are independently sampled from the symmetric Dirichlet distribution. Different from \Cref{sec:experiments:non-ugap-is-common}, we independently sample each entry of the reward functions from the uniform distribution over $[0,1]$, so that the non-zero entries in the reward functions are more dense. 
We make two scatter plots in \Cref{fig:eigs-dense}, each of which visualizes $10^4$ such random examples with a different parameter for the Dirichlet distribution. 
Each point in the scatter plots represents an RB instance, whose $y$-coordinate is the spectral radius of the matrix $\Phi$ (defined in \eqref{eq:phi-def}) and whose $x$-coordinate is the second-largest absolute value of $P_\pibs$'s eigenvalues. 

\begin{figure}
    \FIGURE{
        \subcaptionbox{Dirichlet($0.2$) \label{fig:eigs:dirichlet-02-dense}}{\includegraphics[width=5cm]{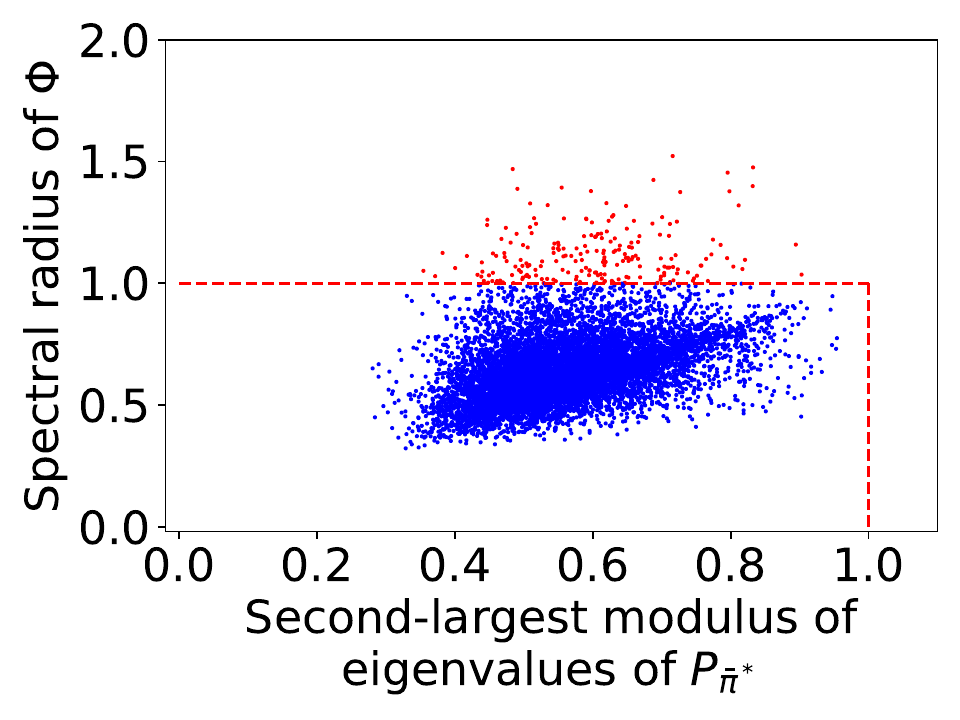}}
        \subcaptionbox{Dirichlet($0.05$)\label{fig:eigs:dirichlet-005-dense}}{ \includegraphics[width=5cm]{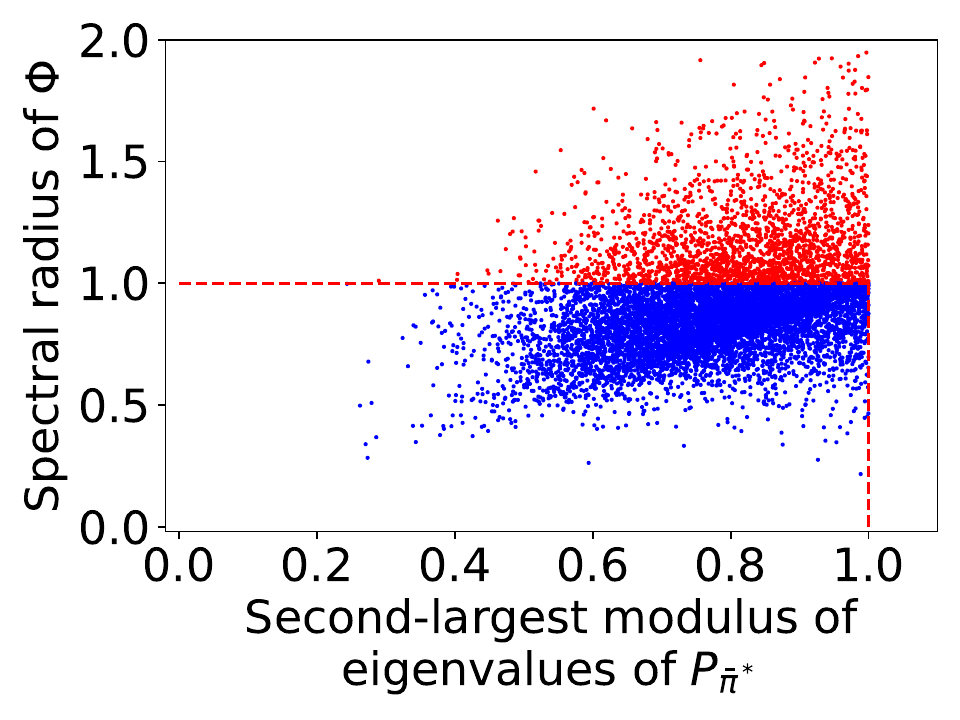}}
    }
    {Scatter plots illustrating the eigenvalues of about $10^4$ random RB problems. The transition probabilities of the instances follow the symmetric Dirichlet distribution with different parameters in each subplot. The entries of the reward functions are sampled from \emph{the uniform distribution over $[0,1]$}. \label{fig:eigs-dense} }
    {Each point in a scatter plot represents an RB problem, whose $x$-coordinate (or $y$-coordinate) represents the second largest modulus of  eigenvalues of $P_\pibs$ (or spectral radius of $\Phi$). Each RB problem has $|\sspa| = 10$. The points outside the dashed box (highlighted in red) represent RB problems that violate GAP under all LP-Priority polices.}
\end{figure}

As we can see from \Cref{fig:eigs-dense}, the scatter plots look almost the same as those in \Cref{fig:eigs}.
Moreover, in \Cref{fig:eigs:dirichlet-005-dense}, among $8965$ instances that are aperiodic unichain whose $x$-coordinates are less than $0.95$ (i.e., have a decently large spectral gap), $1828$ instances are also locally unstable.
So the ratio of locally unstable instances is approximately $20.4\%$, very close to the ratio $20.2\%$ in \Cref{fig:eigs:dirichlet-005}. 
These phenomena suggest that changing the distribution for generating the reward function may not have a significant effect on the fraction of locally unstable instances. 

\begin{figure}
    \FIGURE{
        \subcaptionbox{LP index policy. \label{fig:asym-subopt:lpp-dense}}{\includegraphics[width=5cm]{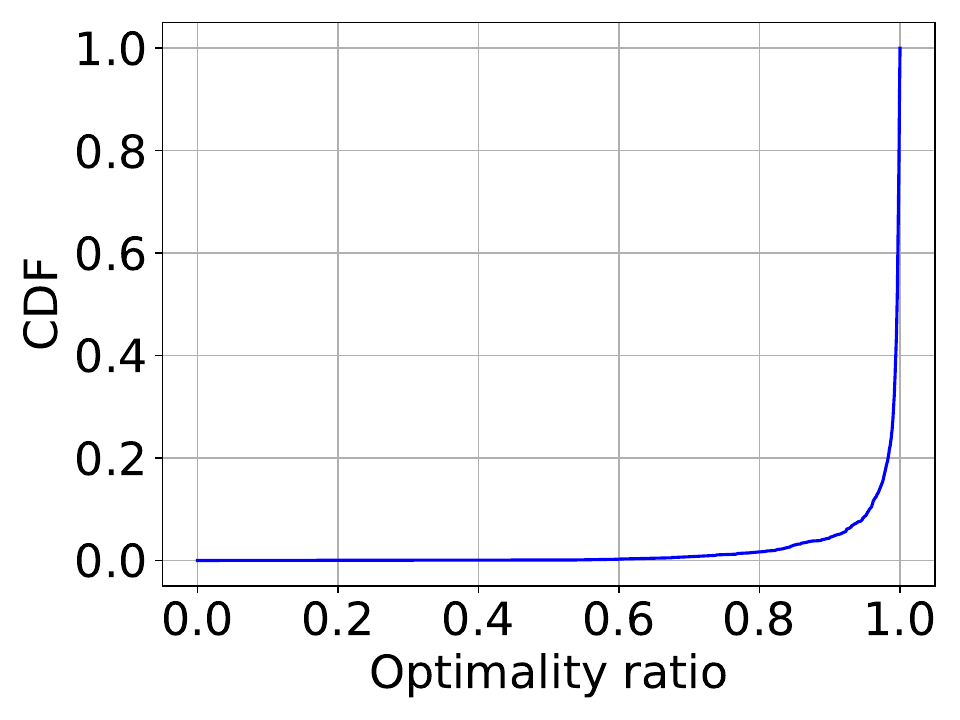}}
        \hfill
        \subcaptionbox{Whittle index policy. \label{fig:asym-subopt:whittle-dense}}{\includegraphics[width=5cm]{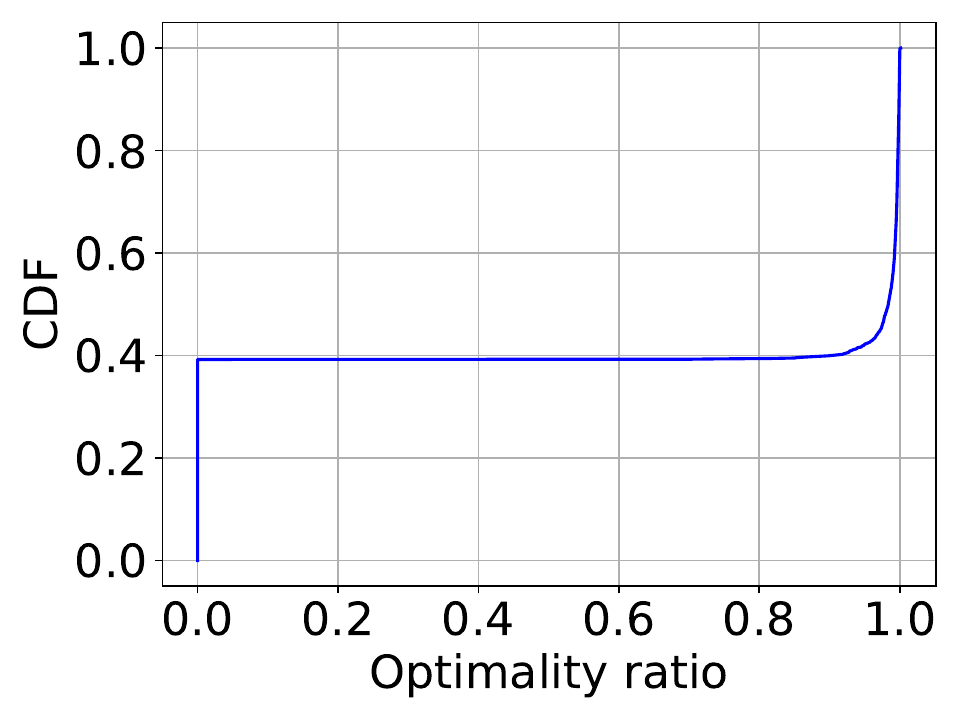}}
        \hfill
        \subcaptionbox{Maximal reward of the two index policies. \label{fig:asym-subopt:max-dense}}{\includegraphics[width=5cm]{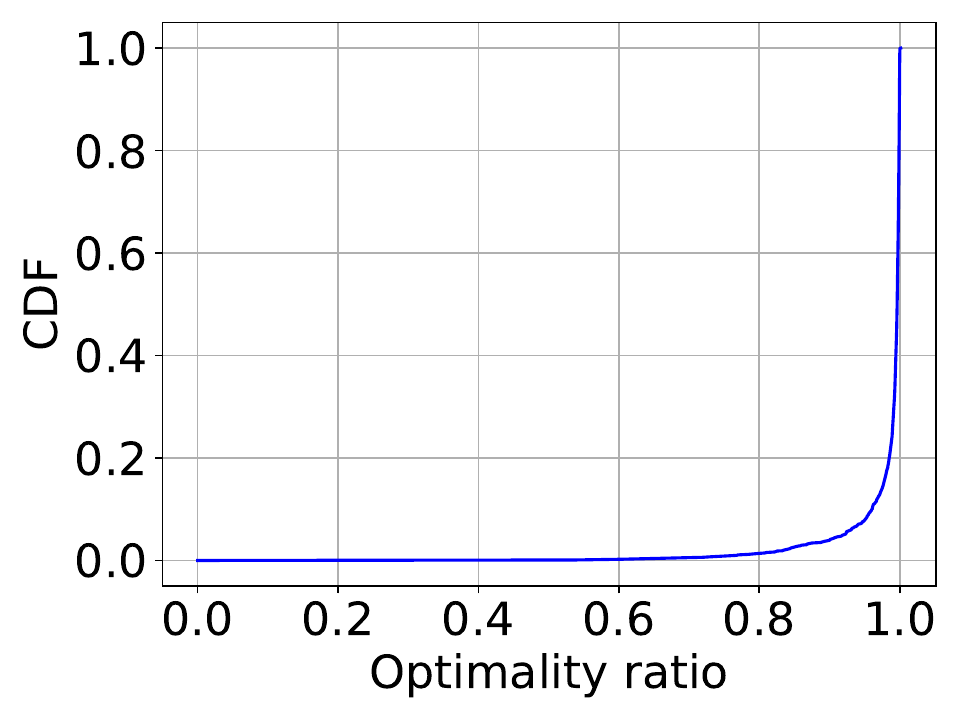}}
    }
    {CDF of the optimality ratios of LP-Priority policies when $N=500$, among $2113$ non-GAP instances with transition kernels following Dirichlet($0.05$) and reward functinos following the uniform distribution.  \label{fig:asym-subopt-dense}}
    {We regard the average reward of Whittle index policy as $0$ if it is not well-defined. } 
\end{figure}

In \Cref{fig:asym-subopt-dense}, we plot some CDF curves representing the optimality ratios of the LP index policy, the Whittle index policy, and their maximal performances when $N=500$, among $2113$ locally unstable instances with transition kernels sampled from Dirichlet($0.05$) and reward functions sampled from the uniform distribution. 
Each policy is simulated for $2\times 10^4$ time steps for each RB instance. 

We compare the CDF curves in \Cref{fig:asym-subopt-dense} with those in \Cref{fig:asym-subopt} to see the effect of the dense reward functions on the performances of the index policies. 
The shapes of the CDF curves look similar, but the optimality gap ratios are slightly better with dense reward functions. 
In particular, there are about $3.9\%$ locally unstable instances where the average rewards of both policies are less than $90\%$ of the LP upper bound, which is less than the $6.7\%$ in \Cref{fig:asym-subopt}. 
Intuitively, the consequences of suboptimal decisions are less severe when the reward function is dense than when the reward function is sparse.

\subsection{Comparing set-expansion policy with ID policy}
\label{app:experiments:se-vs-id}
In this subsection, we make an additional observation about different behaviors of the set-expansion policy and the ID policy: under the ID policy, a larger subset of arms can persistently follow the optimal single-armed policy $\pibs$ for a long period of time than under the set-expansion policy. 
Although this phenomenon does not necessarily imply that the ID policy performs better than the set-expansion policy, we include it here for its potential theoretical interests. 

The observation comes from the following experiment. 
We run the ID policy and the set-expansion policy on each of the six RB problems simulated in \Cref{sec:experiments} with $N=500$. For each of the two policies, we first let them run $5000$ time steps to mix to the steady state. Then for each of the next $1000$ time step $t$, we plot the fraction of arms that whose actions agree with ideal actions in the time interval $[t, t+199]$; these are the arms that will persistently follow $\pibs$ for $200$ time steps starting from time $t$. 
We choose the look-ahead window to be $200$ because it is large enough for an arm following $\pibs$ to converge to the stationary distribution: notice that by \Cref{lem:pibar-one-step-contraction-W}, the $W$-weighted $L_2$ distance between any distribution on $\sspa$ and the optimal stationary distribution $\statdist$ reduces by a ratio of at most $1-1/(2\lamw)$ every time step under $\pibs$, and $2\lamw$ ranges from $2.82$ to $84.29$ in each of the examples simulated in \Cref{sec:experiments}. 

The simulation results are shown in Figures~\ref{fig:persistency-1}, \ref{fig:persistency-2}, and \ref{fig:persistency-3}. 
As we can see, the numbers of arms that can persistently follow $\pibs$ for $200$ time steps under the ID policy are clearly larger than those under the set-expansion policy in all the examples.

\begin{figure}[t]
    \FIGURE{
    \subcaptionbox{The example in \Cref{fig:non-ugap-perf:three-state}.}{\includegraphics[width=7.5cm]{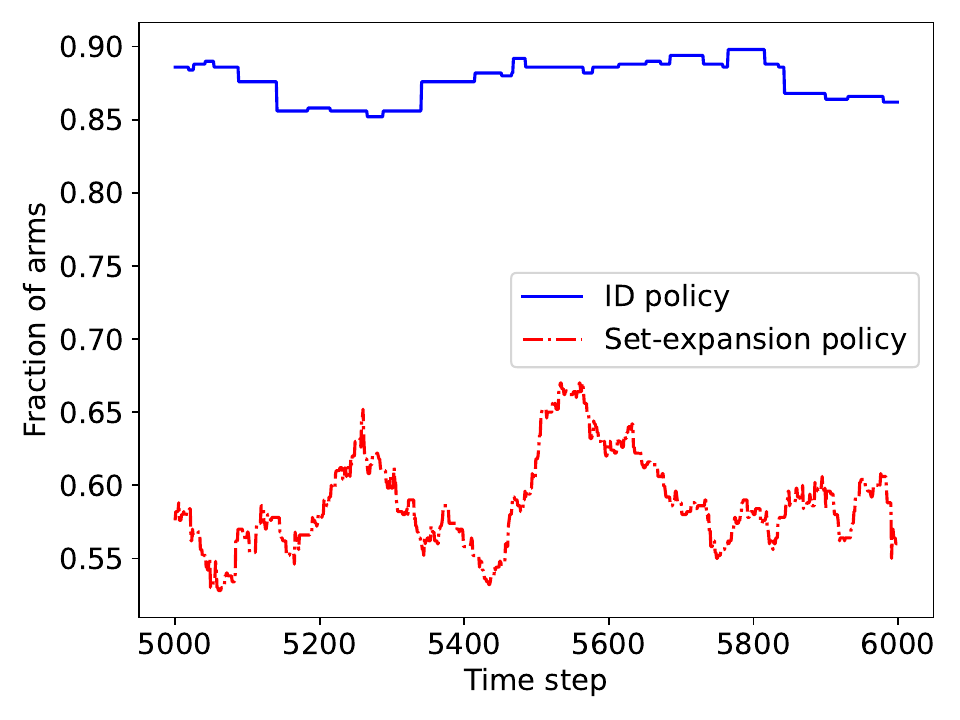}}
    \hfill
    \subcaptionbox{The example in \Cref{fig:non-ugap-perf:eight-state}.}{\includegraphics[width=7.5cm]{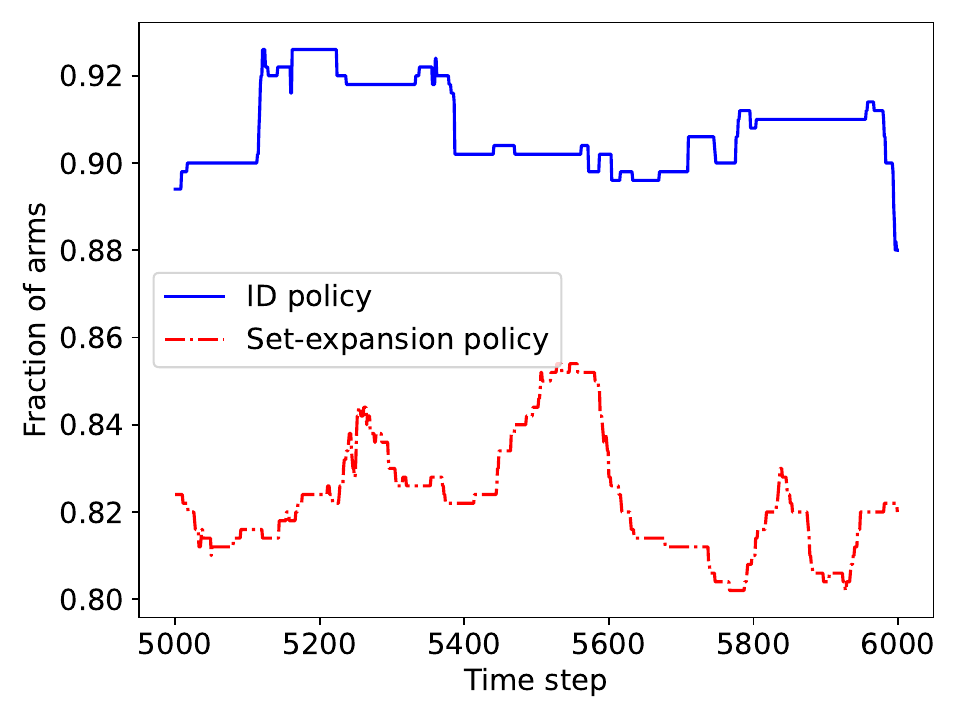}}
    }
    {Illustration of the fractions of arms that persistently follow $\pibs$ for the two RB problems simulated in \Cref{sec:experiments:compare-non-ugap}. \label{fig:persistency-1}}
    {}
\end{figure}

\begin{figure}[ht]
    \FIGURE{
    \subcaptionbox{The example in \Cref{fig:non-ugap-perf:dirichlet-582}.}{\includegraphics[width=7.5cm]{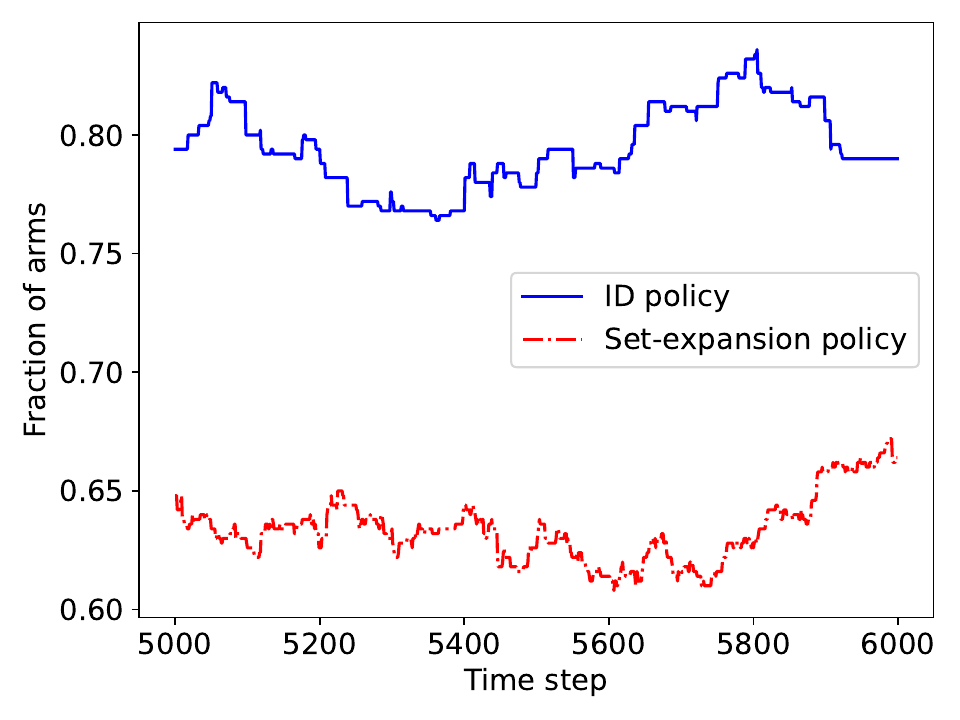}}
    \hfill
    \subcaptionbox{The example in \Cref{fig:non-ugap-perf:dirichlet-355}.}{\includegraphics[width=7.5cm]{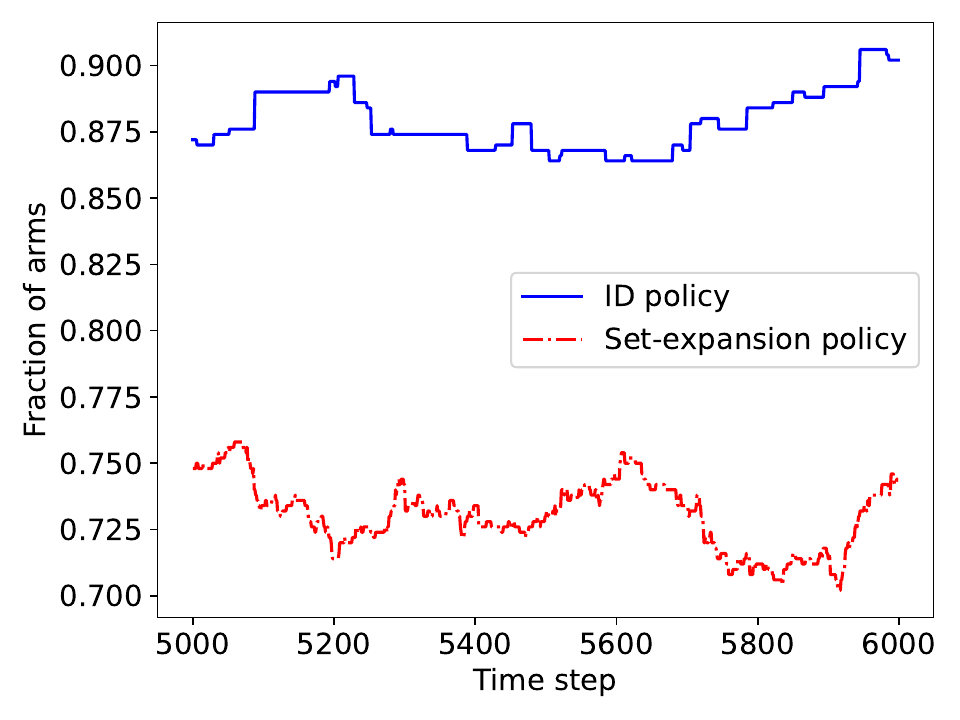}}
    }
    {Illustration of the fractions of arms that persistently follow $\pibs$ for the two RB problems simulated in \Cref{sec:experiments:non-ugap-is-common}. \label{fig:persistency-2}}
    {}
\end{figure}

\begin{figure}[ht]
    \FIGURE{
    \subcaptionbox{The example in \Cref{fig:non-sa:eight-state}.}{\includegraphics[width=7.5cm]{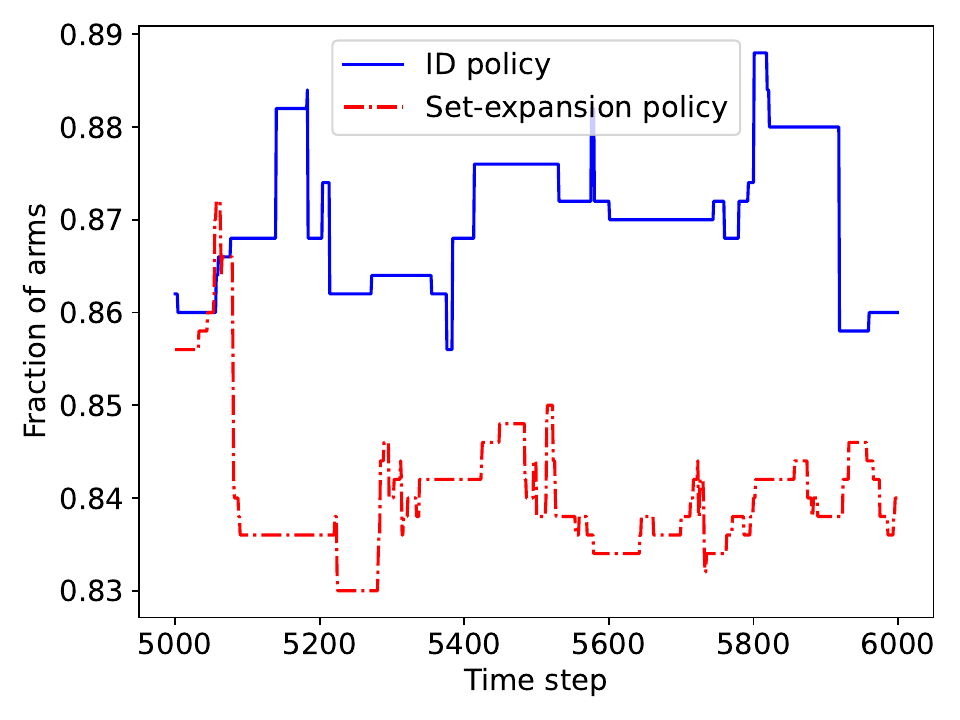}}
    \hfill
    \subcaptionbox{The example in \Cref{fig:non-sa:eleven-state}.}{\includegraphics[width=7.5cm]{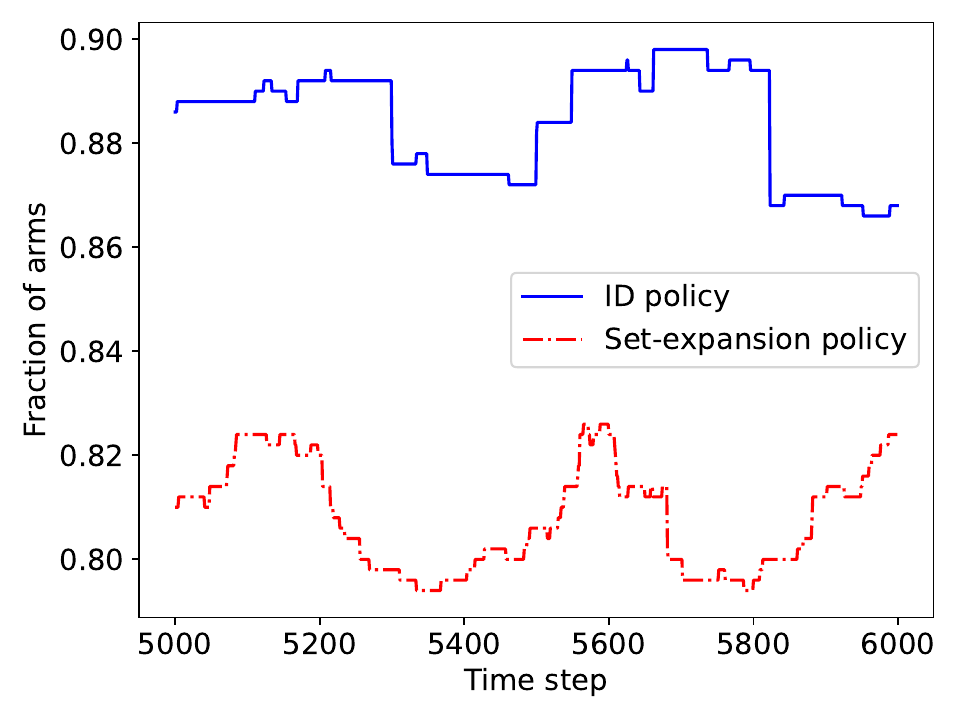}}
    }
    {Illustration of the fractions of arms that persistently follow $\pibs$ for the two RB problems simulated in \Cref{sec:experiments:compare-non-sa}. \label{fig:persistency-3}}
    {}
\end{figure}

Here is a plausible explanation for this phenomenon. Because the set-expansion policy randomly selects arms outside the focus set to follow $\pibs$, only those arms in the focus set could persistently follow $\pibs$. Moreover, the focus set of the set-expansion policy is determined according to the $L_1$ norm constraint, which is somewhat conservatively designed to facilitate the analysis. 
In contrast, the ID policy relies on a fixed set of IDs rather than the explicitly computed focus sets to decide the subset of arms that follow $\pibs$. 
As a result, some arms outside the focus set $[N\md(X_t)]$ that are considered in the analysis may also follow $\pibs$ for long periods of times under the ID policy. Consequently, more arms could follow $\pibs$ persistently under the ID policy than under the set-expansion policy.

\bibliographystyleapp{informs2014}
\bibliographyapp{refs-yige-v240809}

\end{APPENDICES}

\end{document}